\documentclass[a4paper,11pt]{llncs}

\usepackage{times}
\usepackage{url}
\usepackage{latexsym}

\usepackage{amsmath}
\usepackage{amssymb}  
\usepackage{latexsym} 
\usepackage{multirow}

\usepackage{tikz}
\usepackage{pgfplots}
\usetikzlibrary{trees}
\usetikzlibrary{topaths}
\usetikzlibrary{decorations.pathmorphing}
\usetikzlibrary{decorations.markings}
\usetikzlibrary{matrix,backgrounds,folding}
\usetikzlibrary{chains,scopes,positioning,fit}
\usetikzlibrary{arrows,shadows}
\usetikzlibrary{calc} 
\usetikzlibrary{chains}
\usetikzlibrary{shapes,shapes.geometric,shapes.misc}

\pgfdeclarelayer{edgelayer}
\pgfdeclarelayer{nodelayer}
\pgfsetlayers{background,edgelayer,nodelayer,main}

\tikzstyle{dot}=[circle,fill=black,draw=black]
\tikzstyle{every picture}=[baseline=(current bounding box).east,scale=0.5,node distance=5mm]
\tikzstyle{none}=[inner sep=0pt]
\tikzstyle{every loop}=[]

\tikzstyle{(null)}=[]
\tikzstyle{plain}=[]

\newcommand{\semantics}[1]{[\![ #1 ]\!]} 
\newcommand{\ov}{\overrightarrow}
\newcommand{\ovl}{\overline}

\newcommand{\Rel}{\mathrm{Rel}}
\newcommand{\FdVect}{\mathrm{FdVect}}

\newcommand{\cat}[1]{\mathcal{#1}}
\newcommand{\cC}{\cat{C}}

\newcommand{\ket}[1]{|{#1}\rangle}

\newcommand\relto{\mathrel{\ooalign{$\longrightarrow$\cr\hidewidth$|$\hidewidth\cr}}}

\title{A Generalised Quantifier Theory of Natural Language  \\
in Categorical Compositional Distributional Semantics with Bialgebras
}
\author{Jules Hedges\inst{1} \and Mehrnoosh Sadrzadeh\inst{2}}
\institute{Department of Computer Science, University of Oxford \\ \email{julian.hedges@cs.ox.ac.uk}
\and School of Electronic Engineering and Computer Science, Queen Mary University of London \\ \email{m.sadrzadeh@qmul.ac.uk}}

\begin{document}

\maketitle

\begin{center}
\end{center}

\begin{abstract}
Categorical compositional distributional semantics is a model of natural language; it combines the statistical vector space models of  words  with the compositional  models of  grammar.  We formalise in this model the generalised quantifier theory of natural language, due to Barwise and Cooper.  The underlying setting is a compact closed category  with bialgebras. We start from a generative grammar formalisation and  develop an abstract categorical compositional semantics for it, then   instantiate the abstract setting to  sets and relations and to finite dimensional vector spaces and linear maps. We  prove the equivalence of the relational instantiation to  the truth theoretic semantics of generalised quantifiers. The vector space instantiation formalises the  statistical usages of words and enables us to, for the first time, reason about quantified phrases and sentences compositionally in distributional semantics. 
\end{abstract}

\section{Introduction}


Distributional semantics is a statistical model of natural language; it is  based on hypothesis that  words that have similar meanings often occur in the same contexts and  meanings of words can be deduced from the contexts in which they often occur. Intuitively speaking and in a nutshell,  words like `cat' and `dog' often occur in the contexts `pet',  `furry', and `cute', hence  have a similar meaning, one which is different from `baby', since the latter despite being `cute' has not  so often  occurred in the context  `furry' or `pet'.   This  hypothesis has often been traced back to the philosophy of language discussed  by Firth  \cite{Firth} and the mathematical linguistic theory  developed by Harris \cite{Harris}.  Distributional semantics has been used to reason about different aspects of word meaning, e.g.   similarity \cite{Rubenstein,Turney},   retrieval and clustering \cite{Lin,Landauer},   and disambiguation  \cite{Schutze}.  
A criticism to these models has been that natural language is not only about words but also about sentences, but these models  do not naturally extend to sentences, as sentences are not frequently occurring units of  corpora of text.

Models of natural language are not restricted to distributional semantics. A tangential approach puts the focus on the compositional nature of meaning and its relationship with language constructions. This approach is inspired by a hypothesis often assigned to Frege that meaning of a sentence is a function of the meanings of its parts \cite{Frege}. Informally speaking  and very roughly put, meaning of a transitive sentence such as `dogs chase cats' is a binary function of its subject and object. For instance, here the binary function is the verb `chase' and the arguments are  `dogs' and `cats'. This idea has been formalised in different ways, examples  are the early works of Bar-Hillel \cite{Bar-Hillel} and Ajdukiewicz \cite{Ajdukiewicz} on using classical logic, the context free grammars of Chomsky \cite{Chomsky}, and the first order logical approach  of Montague \cite{Montague1970}. One criticism to all these settings, however, is that they do not say much about the meanings of the parts of the sentence. For instance, here we do not know anything more about the meaning of `chase' and of `dogs' and `cats', apart from the fact that they one is a function and others its arguments. 

Compositional distributional semantics  aims to combine the compositional models of grammar with the statistical models of distributional semantics in order to overcome the above mentioned criticisms.  Among the early grammar-based formalisms  of the field is the work of Clark and Pulman \cite{ClarkPulman}, and among the first corpus-based approaches is the work of Mitchell and Lapata \cite{Mitchell-Lapata}.  The former model pairs the distributional meaning representation of a word with its grammatical role in a sentence and defines the meaning of a sentence to be a function of such pairs. The latter, takes the distributional meaning of a sentence to be the addition or multiplication of the distributional meanings of its words. The model of Clark and Pulman has not been experimentally successful and  its theory does not allow comparing  meanings of different sentences. The model of Mitchell and Lapata has been experimentally successful but forgets the grammatical structure of  sentences, since addition and multiplication are commutative. 

Categorical compositional distributional semantics is an attempt to overcome these shortcomings and unify these models. This model was first described in \cite{Clark-Coecke-Sadr} and later published in \cite{Coeckeetal}. It is based on two major developments: first is the mathematical models of grammar introduced in the  work of Lambek \cite{Lambek0,Lambek99}, which either explicitly or implicitly use the theory of monoidal categories; second, is the formulation of the distributional  representations  in terms of vectors, by many  e.g. Salton and Lund  \cite{Lund,Salton}. The categorical model uses the fact that the grammatical structures of language can be described  within a compact closed category  \cite{Preller,Lambek2010} and that finite dimensional vector spaces and linear maps  form such a category \cite{KellyLaplaza80}.  The original  formulation of this model consisted of the product of these two categories, which was later recast using a strongly monoidal functor \cite{PrellerSadr,KartSadr2013,APAL}. 
The theoretical constructions of this model on an elementary fragment of language (adjective nouns phrases and transitive sentences) were  evaluated in \cite{GrefenSadr,GrefenSadrCL}  and in   \cite{kartsaklis2012,KartSadr}.  Much of recent work of the field is focused on using methods from machine learning (regression, tensor decomposition, neural embeddings)  to implement them more efficiently \cite{Milajevs,Nal,MultiStep,Tamara}.   

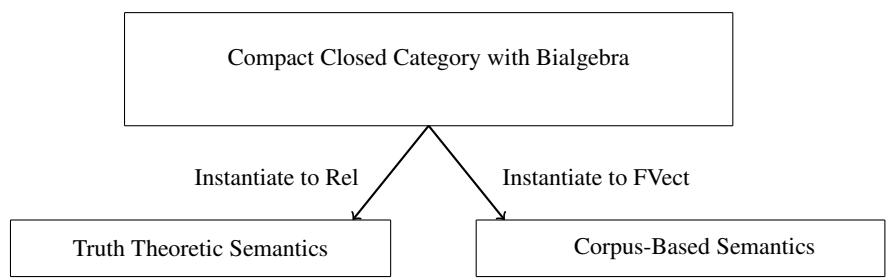
\begin{figure}[h]
\begin{center}
  {%
\beginpgfgraphicnamed{abstract-concrete-diag}
\begin{tikzpicture}
	\begin{pgfonlayer}{nodelayer}
		\node [style=none] (0) at (-8, 3) {};
		\node [style=none] (1) at (8, 3) {};
		\node [style=none] (2) at (-8, 0) {};
		\node [style=none] (3) at (8, 0) {};
		\node [style=none] (4) at (0, 1.75) {Compact Closed Category with Bialgebra};
		\node [style=none] (5) at (0, 0) {};
		\node [style=none] (6) at (-2, -2.5) {};
		\node [style=none] (7) at (2, -2.5) {};
		\node [style=none] (8) at (-11, -2.5) {};
		\node [style=none] (9) at (-1, -2.5) {};
		\node [style=none] (10) at (-1, -4) {};
		\node [style=none] (11) at (-11, -4) {};
		\node [style=none] (12) at (-6, -3.25) {Truth Theoretic Semantics};
		\node [style=none] (13) at (7, -3.25) {Corpus-Based Semantics};
		\node [style=none] (14) at (1.25, -4) {};
		\node [style=none] (15) at (12, -2.5) {};
		\node [style=none] (16) at (1.25, -2.5) {};
		\node [style=none] (17) at (12, -4) {};
		\node [style=none] (4) at (-4, -1.35) {Instantiate to Rel};
		\node [style=none] (4) at (4.4, -1.35) {Instantiate to FVect};
	\end{pgfonlayer}
	\begin{pgfonlayer}{edgelayer}
		\draw [style=none] (0.center) to (1.center);
		\draw [style=none] (1.center) to (3.center);
		\draw [style=none] (0.center) to (2.center);
		\draw [style=none] (2.center) to (3.center);
		\draw [style=none] (8.center) to (9.center);
		\draw [style=none] (9.center) to (10.center);
		\draw [style=none] (8.center) to (11.center);
		\draw [style=none] (11.center) to (10.center);
		\draw [style=none] (16.center) to (15.center);
		\draw [style=none] (15.center) to (17.center);
		\draw [style=none] (16.center) to (14.center);
		\draw [style=none] (14.center) to (17.center);
		\draw [style=thick,->] (5.center) to (6.center);
		\draw [style=thick,->] (5.center) to (7.center);
	\end{pgfonlayer}
\end{tikzpicture}}
\endpgfgraphicnamed}  
\end{center} 
\caption{Abstract and Concrete Models for Generalized Quantifiers in Compositional Distributional Semantics}
\label{fig:AbsConc}
\end{figure}

Despite all these,  dealing with  meanings of logical words such as  pronouns, prepositions, quantifiers, and conjunctives has posed challenges and open problems. In recent work \cite{SadrClarkCoecke1,SadrClarkCoecke2} and also in \cite{kartsaklisphd} we showed how Frobenius algebras over compact closed categories can become useful in modelling relative pronouns and prepositions. In this paper, we take a step further and show how bialgebras over compact closed categories  model generalised quantifiers \cite{BarwiseCooper81}.  We  first present  a  preliminary account of compact closed categories and bialgebras over them and review  how vector spaces and relations provide instances. The contributions of the paper start from section 3, where we develop an abstract categorical semantics for the generalised quantifier theory in terms of diagrams and morphisms of compact closed categories with bialgebras. We present two concrete interpretations of this abstract setting: sets and relations, as well as  finite dimensional vector spaces and linear maps.

The former is the basis for   truth theoretic semantics and the latter  for corpus-base distributional semantics. We prove that the relational instantiation of the abstract model is equivalent to the truth theoretic model of generalised quantifier theory (as presented by Barwise and Cooper). We then prove how the relational model embeds into finite dimensional vector spaces  and more importantly, show how it generalises to a compositional distributional semantic model of language.  We provide vector interpretations for quantified sentences, based on the grammatical structure of the sentences and the meaning vectors of their words. The meaning vectors of nouns, noun phrases, and verbs are as  previously developed. The meaning vectors of determiners and quantised phrases and sentences are novel.

The are two predecessors to this paper:  \cite{RypSadr}, where  Frobenius algebras were used and the equivalence between relational instantiation  and  truth theoretic semantics could  not be established, and  \cite{PrellerSadr2}, where a two-sorted functional logic was used, but  only a case for  semantics of universal quantification was presented.


\section{Preliminaries}
\label{prelim}

\subsection{Vector Space Models of Natural Language}
\label{subsec:DistVect}


Given a corpus of text, a set of contexts and a set of target words, a co-occurrence  matrix has at each of its entries   `the degree of co-occurrence between the target word and the context'. This degree is determined using the notion of a \emph{window}: a span of words or grammatical relations that  slides across the corpus and records the co-occurrences that happen  within it. A context can be a word, a lemma, or a feature. A lemma is the canonical form of a word; it represents the set of  different forms a word can take when used in a corpus. For example, the set  $\{$kills, killed, to kill, killing, killer, killers, $\cdots\}$ is represented by the lemma `kill'. A feature represents a set of words that together express a pertinent linguistic property of a word. These properties can be topical, lexical, grammatical, or semantic. For example the set $\{$bark, miaow, neigh$\}$ represents a semantic feature of  animal, namely the noise that it makes, whereas the set $\{$fiction, poetry, science$\}$ represents the topical features of a book.

The lengths of the corpus and window are parameters of the model, as are the sizes of the feature and target sets. All of these  depend on the task; for  studies  on these parameters, see  for example \cite{Evert,BullinariaLevy}. 

Given  an $m \times n$ co-occurrence matrix, every target word $t$ can be represented  by a row vector of length $n$.  For each feature $f$, the entries of this vector are a function of the raw co-occurrence counts, computed as follows:
\[
\textmd{raw}_f(t) = \frac{\sum_c N(f,t)}{k}  
\]
for $N(f,t)$ the number of times the $t$ and $f$ have co-occurred in the window. Based on $L$,  the total number of times that $t$ has occurred  in the corpus, the raw count is  turned into various normalised degrees. Some common examples are probability, conditional probability, likelihood ratio and its logarithm:

\[
\textmd{P}_f(t) = \frac{\textmd{raw}_f(t)}{L},  \quad \textmd{P}(f|t) = {{P(f, t)}\over {P(t)}}, \quad 
 \textmd{LR}(f,t) = {{\textmd{P}(f \mid t)} \over{\textmd{P}(f)}},  \quad \log \textmd{LR}(f, t) =  \log{{\textmd{P}(f \mid t)} \over{\textmd{P}(f)}}
\]

We denote a vector space model of natural language produced in this way with   $V_{\Sigma}$, where $\Sigma$ is the set of features, and $V_{\Sigma}$ is the vector space spanned by it.

As an example, consider a corpus of $10^8$ words, $10^6$   target words and  $10^5$ features. Fix the window size to be 5 and suppose the co-occurrence matrix with raw counts to be as follows, where the column entries are the feature words and the row entries are the target words.

\newcommand{\dolphinblood}{0}
\newcommand{\sharkblood}{400}
\newcommand{\textlen}{10^8}

\begin{center}
\begin{tabular}{c|c|c|c|c|c|c}
 &{\bf fish} & {\bf horse} & {\bf pet} & {\bf blood}  & ... & total\\
\hline
{\it dolphin}& 500 & 10 & 700& \dolphinblood &... & 2000\\
{\it shark} & 250 & 10 & 20& \sharkblood & ... & 1000 \\
{\it plankton} & 250& 10 & 1000& 10 & ... &  1700\\
{\it pony} & 10 & 1000 & 10 & 10 & ... &  1500\\
\hline
\end{tabular}
\end{center}

The vector representations of the target word `dolphin' with the raw  counts and its functions, as discussed above, are as follows:

\begin{align*}
\textmd{raw}&=  (500, 10, 700, \dolphinblood)&\\
\textmd{P}: &= (\frac{5}{20}, \frac{1}{200}, \frac{7}{20},0) &\\
\textmd{LR}: &= (25000, 500, 17500, 0)&\\
\log \textmd{LR}:&= (1.397, -0.301, 1.2430, 0)&
\end{align*}

Various notions of distance (length, angle) between the  vectors have been used to measure the degree of similarity (semantic, lexical, information content) between the words.  For instance, for the cosine of the angle between the  vectors of dolphin and other target words we obtain:

\[
\cos( \ov{\textmd{dolphin}}, \ov{\textmd{shark}}) = 0.87 \qquad
\cos( \ov{\textmd{dolphin}}, \ov{\textmd{pony}}) = 0.009
\]

This indicates that  the degree of similarity between dolphin and shark is much higher than that of  dolphin and pony. These degrees directly follow  the co-occurrence degrees we have set above,  that dolphin and shark have co-occurred often with the same fearture, but dolphin and pony have done so to a much lesser degree.

\subsection{Generalised Quantifier Theory in Natural Language}
\label{subsec:GenQuant}

We briefly review the theory of generalised quantifiers in natural language as presented in \cite{BarwiseCooper81}.  
Consider the fragment of English generated  by the following context free grammar, referred to by $G_Q$: 

%
%

\begin{center}
\begin{minipage}{5cm}
\begin{tabular}{ccc}
S & $\to$ & NP VP \\
VP & $\to$ & V NP \\
NP &$\to$ & Det N
\end{tabular}
\end{minipage}
\qquad
\begin{minipage}{9cm}
\begin{tabular}{ccl}
NP & $\to$ &  John, Mary, something, $\cdots$\\
N& $\to$ &   cats, dogs, men, $\cdots$\\
VP & $\to$ &  sneeze, sleep,$\cdots$\\
V & $\to$ &   love, kiss, $\cdots$\\
Det & $\to$ &  a, the, some, every, each, all, no, most, few, one, two, $\cdots$ 
\end{tabular}
\end{minipage}
\end{center}

A model for the language generated by this grammar  is a pair $(U, \semantics{\ })$, where $U$ is a universal reference set and $\semantics{\ }$ is an interpretation  function defined by induction as follows. 
\begin{enumerate}
\item  On terminals. 
\begin{enumerate}
\item The interpretation of a  determiner  $d$ generated by  `$\text{Det} \to d$'   is  a map  with the following type:
\[
 \semantics{d} \colon {\cal P}(U) \to {\cal P}{\cal P}(U)
\]
It  assigns to each $A \subseteq U$, a family of subsets of $U$. 
 These interpretations are referred to as \emph{generalised quantifiers}. For logical quantifiers, they are  defined as follows:

\begin{eqnarray*}
\semantics{\text{some}}(A) &=& \{X \subseteq U \mid X \cap A \neq \emptyset\}\\
\semantics{\text{every}}(A) &=& \{X \subseteq U \mid A \subseteq X\}\\
\semantics{\text{no}}(A) &=& \{X \subseteq U \mid  A \cap X = \emptyset\}\\
\semantics{n}(A) &=& \{X \subseteq U \mid \ \mid X \cap A \mid = n\}
\end{eqnarray*}

\noindent
A similar method is used to define non-logical quantifiers. For example for non-logical quantifiers most, few, and many, they are defined as follows:
\begin{eqnarray*}
\semantics{\text{most}}(A) &=& \{X \subseteq U \mid X \ \mbox{has most elements of} \ U\}\\
\semantics{\text{few}}(A) &=& \{X \subseteq U \mid X \ \mbox{has few elements of} \ U\}\\
\semantics{\text{many}}(A) &=& \{X \subseteq U \mid X \ \mbox{has many elements of} \ U\}
\end{eqnarray*}

\item The interpretation of a  terminal $y \in \{np, n, vp\}$ generated by either of the rules `NP $\to$ np, N $\to$ n, VP $\to$ vp'  is $\semantics{y} \subseteq U$. That is, noun phrases, nouns and verb phrases are interpreted as subsets of  the reference set. 
\item The interpretation of a  terminal $y$ generated by the rule V $\to$ y is  $\semantics{y} \subseteq U \times U$. That is,  verbs are interpreted as binary relations over the reference set.  
\end{enumerate}
\item On non-terminals.
\begin{enumerate}
\item The interpretation of  expressions generated by the  rule `$\text{NP} \to \mbox{Det N}$'   is  as follows:

\begin{eqnarray*}
\semantics{\mbox{Det N}} \ = \  \semantics{d}(\semantics{n})  
&\ \text{where} \ &  X \in \semantics{d}(\semantics{n}) \ \text{\bf iff} \ X \cap \semantics{n} \in \semantics{d}(\semantics{n})\\
&\quad \text{for \ all} \ d,n \ & \text{generated \ by} \   \mbox{Det} \to d  \ \text{and} \  \mbox{N} \to n 
\end{eqnarray*}

\item The interpretation of expressions generated by  the  rule `VP $\to$ V NP'  is as follows:
\[
\semantics{\mbox{V NP}} =    \semantics{v}(\semantics{np})  \quad \text{for \ all} \  v, np \ \text{generated \ by} \  \text{VP} \to v\  np
\]
where, $R$ is a unary relation $R \subseteq U$ and for $A \subseteq U$,   $R(A)$ is the forward image of $R$ on $A$, that is $R(A) = \{y \mid y \in R, \text{for} \ x \in A, \text{s.t.} x = y \}$, indeed $R \cap A$.  

\item The interpretation of expressions generated by  the  rule `S $\to$ NP VP'  is as follows
\[
\semantics{\mbox{NP VP}} = \semantics{vp}( \semantics{np}) \quad \text{for \ all} \  np, vp \ \text{generated \ by} \  \text{S} \to  np \ vp
\]
where,    $R \subseteq U \times U$ is a binary relation and  for $A \subseteq U$,  $R(A)$ is the forward image of $R$ on $A$, that is $R(A) = \{y \mid (x,y) \in R, \mbox{for \ some}  \ x \in A\}$. 
\end{enumerate}

In either of the  2.(b) and 2.(c),  the cases where $\semantics{np}$ is a family of sets, i.e. obtained by  application of a determiner to a noun,  are defined using the  induction hypothesis. More specifically, these definitions go through  item 2.(a), which in turn  is obtained by going through   item 1.(a).

\end{enumerate}

Generalised quantifiers in natural language satisfy a property referred to by \emph{living on} or \emph{conservativity}, defined below.

\begin{definition}\label{def:livingon}
For a generalised quantifier $Q : \mathcal P (U) \to \mathcal P \mathcal P (U)$, we say that $Q$ satisfies the \emph{living on property} if, for all $X, A \subseteq U$, $X \in Q (A)$ iff $X \cap A \in Q (A)$.
\end{definition}

In  the models of natural language, a generalised quantifier is the interpretation of a determiner. In  $G_Q$,  a determiner is generated by the syntactic rule  `Det $\to$ d'. Determiners  are applied to nouns via the syntactic rule  `NP $\to$ Det N'.  Thus, when employed in  the language generated by $G_Q$, the above definition  gets the following form:

\begin{quote}
For the expressions of $G_Q$,  a determiner $\semantics{d}$  satisfies the living on property  if,   for all $A, X \subseteq U$,  $X \in \semantics{d}(A)$ iff  $X \cap A \in \semantics{d}(A)$. 
\end{quote}

Originally discussed  in \cite{BarwiseCooper81} and subsequently in almost all  the literature on generalised quantifier theory, it is easy to verify that the living on property  makes the  following  true:

\begin{lemma}
\label{lemma:livingon}
For $d, n, np, vp$,  a determiner, a noun, a noun phrase, and a verb phrase of $G_Q$, if $\semantics{d}$  satisfies the \emph{living on} property, equivalences of the following kind hold on the expressions  of $G_Q$: 
\begin{quote}
$d \ n \ vp \ \iff \ d \ n  \ \text{\bf are} \  n \ \text{\bf who}  \ vp$\\
$np \ v \ d \ n  \ \iff \  np  \ v\  d\ n \ \text{\bf who \ are}  \ n$
\end{quote}

\end{lemma}

Examples are as follows:

\begin{quote}
All men eat. \ $\iff$ \ All men are men who eat.\\
Many men run. \ $\iff$ \ Many men are men who run. \\
Men love some cats. \ $\iff$ \ Men love some cats who are  cats.
\end{quote}

Thus the quantifiers modifying the subjects and objects of sentences,  \emph{live on} these subjects and objects.  The  equivalent sentences of the right hand sides seem redundant. They seem to be a redundant way of expressing the same as their left hand side sentences. However, they express the fact  that only the part of the $vp$  that is restricted to the quantified $np$ matters.    Barwise and Cooper note that this is a property of natural language, that every natural language has determiners whose semantic role is to assign to  nouns  (more precisely to the common count nouns) quantifiers that \emph{live on} them. For instance, the determiner `all' in the first example above assigns to  the noun phrase `men'  the quantifier \emph{all} such that it  lives on `men'.  This criteria is thus used to filter out determiners whose semantics is not definable by the generalised quantifier theory, an example is the determiner `only'.  There are also  mathematical quantifiers for whom this property fails, an example is H\"{a}rtig's  \emph{equinumerous} quantifier  defined by $B \in I (A) \iff |A| = |B|$.

\medskip
The `meaning' of a sentence  is its truth value, defined as follows:
\begin{definition}
\label{def:truth-genquant}
A sentence in generalised quantifier theory is true  iff $\semantics{\mbox{NP VP}} \neq \emptyset$.
\end{definition}

As an example,   meaning of a sentence with a quantified phrase at its subject position becomes as follows:
\[
\semantics{\mbox{Det N VP}} \ = \ \begin{cases} \text{\it true} & \text{ if } \semantics{vp} \cap \semantics{n} \in \semantics{\mbox{Det N}}  \\
\text{\it false} & \mbox{otherwise}
\end{cases}
\]
obtained by applying the inductive definition  of $\semantics{\ }$ to    $\semantics{\mbox{NP \ VP}}$ to generate  $\semantics{S}$. Herein,   $\semantics{np}$ is generated by  the inductive step $\semantics{\mbox{Det \ N}}$.  For instance, meaning of `some men sneeze', which is of this form,  is true iff $\semantics{\text{sneeze}} \cap \semantics{\text{men}} \in \semantics{\mbox{some men}}$, that is, whenever the set of things that sneeze and are men is a non-empty set.  As another example, consider the  meaning of a sentence with a quantified phrase at its object position, whose meaning is as follows:
\[
\semantics{\mbox{NP V Det N}} \ = \ \begin{cases} \text{\it true} & \text{ if } \semantics{n} \cap 
\semantics{v}(\semantics{np}) \in \semantics{\mbox{Det N}}  \\
\text{\it false} & \mbox{otherwise}
\end{cases}
\] 
This is obtained by applying the inductive definition to $\semantics{\mbox{NP \ VP}}$ to generate $\semantics{S}$, wherein $\semantics{vp}$ is obtained by the inductive step   $\semantics{\mbox{V \ NP}}$ and $\semantics{np}$ is obtained by the inductive step  $\semantics{\mbox{Det \ N}}$. An example of this case is the meaning of `John liked some trees', which  is  true iff $\semantics{\text{trees}} \cap \semantics{\text{like}}(\semantics{\text{John}}) \in \semantics{\mbox{some trees}}$, that is, whenever, the set of things that are liked by John and are trees is a non-empty set. Similarly, the sentence `John liked five trees' is true iff the set of things that are liked by John and are trees has five elements in it.

\subsection{From Context Free   to Pregroup Grammars}

A pregroup algebra $P = (P, \leq, \cdot, (-)^r, (-)^l)$   is a partially ordered monoid where every element has a left and a right adjoint \cite{Lambek99}.  That is, for  $p \in P$, there are $p^l, p^r \in P$ that satisfy the following four inequalities:
\[
p \cdot p^r \leq 1 \leq p^r \cdot p\qquad
p^l \cdot p \leq 1 \leq p \cdot p^l
\]

Let $P$ be a pregroup algebra; a pregroup grammar based on $P$  is a tuple $P = (P, \Sigma,  \beta, s)$,  where  $\Sigma$  is the vocabulary of the language,    $s \in P$   is a designated sentence type,   and $\beta$ is a  relation  $\beta \subseteq \Sigma \times P$ that assigns to words in $\Sigma$  elements of the pregroup $P$.  This relation is  referred to as a `type dictionary' and the elements of the pregroup as  `types'.  

A pregroup grammar $P$ assigns a type $p$ to a string  of words $w_1 \cdots w_n$,  for  $w_i \in \Sigma$, 
 if there exist types  $p_i \in \beta(w_i)$ for $1 \leq i \leq n$ such that  $p_1 \cdot \ \cdots \ \cdot p_n \leq p$. 
 We refer to this latter   inequality as the \emph{grammatical reduction} of the string. If $p_1 \ \cdot \ \cdots \ \cdot \ p_n \leq s$ then the string is a grammatical sentence. 

A context free grammar (CFG) is transformed into a pregroup grammar via the procedure described in  \cite{Buszkowski-prg}. In a nutshell, one first transforms the CFG into an Ajdukiewicz grammar \cite{Ajdukiewicz}, using the procedure developed by Bar-Hillel, Gaifman, and Shamir \cite{Shamir}.  The procedure developed by Buszkowski  is then applied to transform  the result into a Lambek calculus \cite{Buszkowski-syncal}. Via a translation between Lambek calculi and pregroup grammars \cite{Lambek08}, the result is finally turned into a pregroup grammar. 
\begin{center}
CFG $\stackrel{\mbox{\cite{Ajdukiewicz}}}{\longrightarrow}$ Ajdukiewicz Grammar $\stackrel{\mbox{\cite{Buszkowski-syncal}}}{\longrightarrow}$ Lambek Calculus $\stackrel{\mbox{\cite{Buszkowski-prg}}}{\longrightarrow}$ Pregroup Grammar
\end{center}


In a  context free grammar  in Chomsky normal form, the rules are either of the form $A \to BC$ or $A \to x$, for $A, B, C$ non-terminals and $x$ a terminal.   The rules of such a grammar are classified based on the model defined over their generated language. We recall that a rule $A \to BC$ is referred to by \emph{right-to-left} if $\semantics{A} := \semantics{C}(\semantics{B})$;   it is referred to by  \emph{left-to-right} if $\semantics{A} := \semantics{B}(\semantics{C})$; the rest of the rules are referred to by \emph{atomic}. Based on this given classification and by collating the above, we define the concept of a  pregroup grammar generated over a context free grammar  as follows:
\begin{definition}
\label{def:cfg-prg}
A  pregroup grammar $P_G$ generated over a context free grammar  $G=(T, N, S, {\cal R})$ and a set of atomic types ${\cal A}$  is the pregroup grammar $(P({\cal A}), T,  \beta, \sigma(S))$ defined as follows:
\begin{itemize}
\item  $P({\cal A})$ is the free pregroup algebra generated over the set of atomic types ${\cal A}$.
\item $\beta$ is $\{(x, \sigma(x)) \mid x \in T\}$.
\item $\sigma \colon  N \cup T  \to P({\cal A})$ is as given below:
\begin{itemize}
\item To a non-terminal $C$ in  a left-to-right  rule $A \to BC$ of $G$, it assigns $\sigma(C) := \sigma(B)^r \cdot \sigma(A)$. 
\item To a  non-terminal $B$  in a right-to-left rule $A \to BC$, it assigns  $\sigma(B) := \sigma(A) \cdot \sigma(C)^l$. 
\item To all the other non-terminals $A$, it assigns an atomic type $\sigma(A)$.
\item To all terminals $x$, generated by an atomic rule  $A \to x$, it assigns the type $\sigma(A)$.  
\end{itemize}
\end{itemize}
\end{definition}


As an example, consider  $G_Q$ and the set of atomic types $\{s,n,p\}$. The pregroup grammar generated over  these is  $(P(\{s,n,p\}), T, \beta, s)$.  This grammar is in Chomsky normal form.  The rule `S $\to$ NP \ VP' is  left-to-right and  the rules `VP $\to$ V \ NP' and `NP $\to$ Det \ N' are right-to-left;  the rest of the rules are atomic.  On the non-terminals VP, V, Det,    $\sigma$ is  thus defined as follows:
\[
\sigma(VP) := \sigma(NP)^r \cdot \sigma(S) \qquad
\sigma(V) := \sigma(VP) \cdot \sigma(NP)^l \qquad
\sigma(Det)  := \sigma(NP) \cdot \sigma(N)^l
\]
On the rest of the non-terminals $\sigma$ is defined as follows:
\[
\sigma(NP) = p \qquad
\sigma(N) = n \qquad
\sigma(S) = s
\] 
 $\beta$ is as follows:
\[
\{(np,p), (n,n), (vp, p^r \cdot s), (v, p^r \cdot s \cdot p^l), (d, p \cdot n^l)\}
\]
Noun phrases  take type $p$, nouns type $n$, intransitive verbs type $p^r \cdot s$, transitive verbs type $p^r \cdot s \cdot p^l$. Determiners  take type $p \cdot n^l$.   Sample elements of $\beta$ are as follows:
\[
\{(\text{John}, p), (\text{cats}, n),  (\text{sneeze}, p^r \cdot s), (\text{stroked}, p^r \cdot s \cdot p^l), (\text{some}, p \cdot n^l), \cdots\}
\]

%
%

In this pregroup grammar,  a quantified noun phrase such as `some cats', a sentence with a quantified phrase in its subject position such as  `some cats sneeze', and  a sentence with a quantified phrase in its object position such as `John stroked some cats'. The grammatical reductions of  these in the pregroup grammar are as follows:

\begin{center}
\begin{tabular}{ccccc}
&&some & cats &\\
&&$(p \cdot n^l)$ &$ \cdot n$  & $\leq p \cdot 1 = p$\\
&some & cats & sneeze &\\
&$(p \cdot n^l)$ &$ \cdot n$ & $\cdot (p^r \cdot s)$ & $\leq p\cdot 1 \cdot (p^r \cdot s) = p \cdot (p^r \cdot s) \leq 1 \cdot s = s$\\
John & stroked & some & cats&\\
$p$ & $\cdot (p^r \cdot s \cdot p^l)$ & $\cdot (p \cdot n^l)$ & $\cdot n$ & $\quad \leq \quad 1 \cdot (s \cdot p^l) \cdot p  \cdot 1 = (s \cdot p^l) \cdot p \leq  s\cdot 1 = s$
\end{tabular}
\end{center}

\noindent
In the first example, `some' inputs `cats' and outputs a noun phrase; in the second example, first `some' inputs `cats' and outputs a noun phrase, then `sneeze' inputs this noun phrase and outputs a sentence; in the last example, again first `some' inputs `cats' and outputs a noun phrase, at the same time the verb inputs `John' and outputs a verb phrase of type $s \cdot p^l$, which then  inputs the $p$ from the phrase `some cats' and outputs a sentence. 

In the pregroup grammar of English presented  in \cite{Lambek08}, Lambek proposes to  type the quantifiers as follows:
\[
\mbox{when modifying the subject}: ss^l\pi \pi^l \qquad
\mbox{when modifying the object}: os^rso^l
\]
For the subject case, we have the identity $ss^l\pi \pi^l  = s (\pi^rs)^l \pi^l$, which means that the quantifier inputs the subject (of type $\pi$) and the whole verb phrase and produces a sentence. Similarly, in the object case we have $os^rso^l = (so^l)^r so^l$. These types are translations of the original Lambek calculus types for quantifiers, where they were designed such that they would get  a first order logic semantics through a correspondence with lambda calculus \cite{vanBenthem}.  However, as  explained in \cite{Lambek08}, due to the ambiguities in  Lambek calculus-pregroup translations such a correspondence fails for pregroups.   Consequently,  the above types fail to provide a logical semantics for quantifiers. In this paper, we have taken a different approach and go by the types coming from the CFG of generalised quantifier theory. It will become apparent in the proceeding sections how this together with the use of compact closed categories offers a solution.

\subsection{Category Theoretic and Diagrammatic Definitions}
This subsection briefly reviews compact closed
categories and bialgebras. For a formal presentation, see
\cite{KellyLaplaza80,Kock72,McCurdy}.  A compact closed category, $\cC$, has objects $A, B$; morphisms $f \colon A
\to B$; and a monoidal tensor $A \otimes B$ that has a unit $I$, that is we have $A \otimes I \cong I \otimes A \cong A$. Furthermore, for
each object $A$ there are two objects $A^r$ and $A^l$ and  the
following morphisms:
\begin{align*}
A \otimes A^r \stackrel{\epsilon_A^r} {\longrightarrow} \; &I
\stackrel{\eta_A^r}{\longrightarrow} A^r \otimes A \hspace{1cm}
A^l \otimes A \stackrel{\epsilon_A^l}{\longrightarrow} \; I
\stackrel{\eta_A^l}{\longrightarrow} A \otimes A^l\
\end{align*}
These morphisms satisfy the following equalities, where $1_A$ is the
identity morphism on object $A$:
\begin{align*}
& (1_A \otimes \epsilon_A^l) \circ (\eta_A^l \otimes 1_A)  = 1_A 
\hspace{1cm}
&&(\epsilon_A^r \otimes 1_A) \circ (1_A \otimes
  \eta_A^r)   = 1_A\\
& (\epsilon_A^l \otimes 1_A) \circ (1_{A^l} \otimes
  \eta_A^l) = 1_{A^l}  
    \hspace{1cm}
&&    (1_{A^r} \otimes \epsilon_A^r) \circ (\eta_A^r \otimes 1_{A^r}) = 1_{A^r}
\end{align*}
\noindent These express the fact the $A^l$ and $A^r$ are the left and right
adjoints, respectively, of $A$ in the 1-object bicategory whose
1-cells are objects of $\cC$. A self adjoint compact closed category is one in which for even object $A$ we have $A^l \equiv A^r \equiv A$. 

Given a morphism $f : X \to Y$ in a self-adjoint compact closed category, its \emph{transpose} is the morphism $f^\top : Y \to X$ defined by
\[ (\epsilon_Y \otimes X) \circ (Y \otimes f \otimes X) \otimes (Y \otimes \eta_X) \]

Given two compact closed categories ${\cal C}$ and ${\cal D}$ a strongly monoidal functor $F \colon {\cal C} \to {\cal D}$  is defined as follows:
\[
F(A \otimes B) = F(A) \otimes F(B) \qquad
F(I) = I
\]
One can show that this functor preserves the compact closed structure, that is we have:
\[
F(A^l) = F(A)^l \qquad
F(A^r) = F(A)^r
\]

A bialgebra  in a  symmetric monoidal  category $({\cal
  C}, \otimes, I, \sigma)$ is a tuple $(X,  \delta, \iota, \mu, \zeta)$ where,
for $X$ an object of ${\cal C}$, the triple $(X, \delta, \iota)$ is  an internal comonoid; 
i.e.~the following are  coassociative and counital  morphisms of ${\cal
  C}$:
\begin{align*}
\delta \colon X \to X \otimes X&\qquad& \iota \colon X \to I
\end{align*}
Moreover $(X, \mu, \zeta)$ is  an internal  monoid; i.e.~the following are  associative and unital  morphisms:
\begin{align*}
\mu \colon  X \otimes X \to X  &\qquad& \zeta \colon I \to X
\end{align*}
And finally  $\delta$ and $\mu$ satisfy the four equations \cite{McCurdy}
\begin{align*}
\iota \circ \mu &= \iota \otimes \iota		&&\text{(Q1)} \\
\delta \circ \zeta &= \zeta \otimes \zeta	&&\text{(Q2)} \\
\delta \circ \mu &= (\mu \otimes \mu) \circ (\operatorname{id}_X \otimes \sigma_{X,X} \otimes \operatorname{id}_X) \circ (\delta \otimes \delta) && \text{(Q3)} \\
\iota \circ \zeta &= \operatorname{id}_I	&&\text{(Q4)}
\end{align*}

Informally, the  comultiplication $\delta$  dispatches to copies  the information contained in
one object into two objects, and the  multiplication $\mu$ unifies  or merges  the
information of two objects into one. In what follows, we present three examples of compact closed categories, two of which with bialgebras. 

\subsection{Three Examples of Compact Closed Categories}

\medskip
\noindent
{\bf Example 1. Pregroup Algebras}
A pregroup algebra $P = (P, \leq, \cdot, (-)^l, (-)^r)$  is a compact closed category whose objects are the elements of the set $p \in P$ are the objects of the category and the partial ordering between the elements are the morphisms. That is,  for $p,q \in P$, we have that $p \to q$ is a morphism of the category iff $p \leq q$ in the partial order. The tensor product of the category is the monoid multiplication, whose unit is 1, and the adjoints of objects are the adjoints of the elements of the algebra.  The epsilon and eta morpshism are thus as follows:

\begin{align*}
p \cdot p^r \stackrel{\epsilon_p^r} {\longrightarrow} \; &1
\stackrel{\eta_p^r}{\longrightarrow} p^r \cdot p \hspace{1cm}
p^l \cdot p \stackrel{\epsilon_p^l}{\longrightarrow} \; 1
\stackrel{\eta_p^l}{\longrightarrow} p \cdot p^l\
\end{align*}
The above directly follow from the preroup inequalities on the adjoints.  A pregroup with a bialgebra structure on it becomes degenerate. To see this, suppose we have such an  algebra on the object $p$ of such a pregroup. Then the unit morphism of the internal comonoid of this algebra becomes the partial ordering  $\iota \colon p \leq 1$; taking the right adjoints of both sides of this inequality will yield $1 = 1^r \leq p^r$, and by the multiplying both sides of this with $p$   we will obtain $p \leq p \cdot p^r$, which by adjunction results in $p \leq p \cdot p^r \leq 1$, hence we have $p \leq 1$ and also $1 \leq p$, thus $p$ must be equal to 1. That is, assuming that we have  a bialgebra on an object will mean that that object is 1. 

\medskip
\noindent
{\bf Example 2. Finite Dimensional Vector Spaces over $\mathbb{R}$.}
These structures  together with  linear maps  form a compact
closed category, which we refer to as $\FdVect$.  Finite dimensional
vector spaces $V, W$ are objects of this category; linear maps $f
\colon V \to W$ are its morphisms with composition being the
composition of linear maps. The tensor product $V
\otimes W$ is the 
linear algebraic tensor product,
whose unit is the scalar
field of vector spaces; in our case this is the field of reals
$\mathbb{R}$.  Here, there is  a natural
isomorphism $V \otimes W \cong W \otimes V$. As a result of the
symmetry of the tensor, the two adjoints reduce to one and we obtain the  isomorphism $V^l \cong V^r \cong V^*$, 
where $V^*$ is the dual space of $V$. When the
basis vectors of the vector spaces are fixed, it is further the case
that we have $V^* \cong V$.  Thus, the compact closed category of finite dimensional  vector spaces with fixed basis is  self adjoint.

%
Given a basis $\{r_i\}_i$ for a vector space $V$, the epsilon maps are
given by the inner product extended by linearity; i.e. we have:
\[
\epsilon^l  =  \epsilon^r \colon   V \otimes V \to \mathbb{R} \quad \mbox{given by} \quad
\sum_{ij} c_{ij} \ (\psi_i \otimes \phi_j)  \quad \mapsto \quad \sum_{ij} c_{ij} \langle \psi_i \mid \phi_j \rangle\]
Similarly, eta maps   are defined as follows:
\[
\eta^l = \eta^r \colon   \mathbb{R} \to V \otimes V
\quad \mbox{given by} \quad 
1 \; \mapsto \; \sum_i (\ket{r_i} \otimes \ket{r_i})
\]

Transposes in the category of finite dimensional vector spaces are given by linear-algebraic transposes of linear maps.

Let $V$ be a vector space with basis $\mathcal P (U)$, where $U$ is an arbitrary set.
We give $V$ a bialgebra structure as follows:
\begin{align*}
\iota \ket A &= 1 \\
\delta \ket A &= \ket A \otimes \ket A \\
\zeta &= \ket U \\
\mu (\ket A \otimes \ket B) &= \ket{A \cap B}
\end{align*}
Note that an arbitrary basis element of $V \otimes V$ is of the form $\left| A \right> \otimes \left| B \right>$ for $A, B \subseteq U$.
For example, the verification of the bialgebra axiom (Q3) is as follows:
\begin{align*}
((\mu \otimes \mu) \circ (\operatorname{id} \otimes \sigma \otimes \operatorname{id}) \circ (\delta \otimes \delta)) (\ket A \otimes \ket B)
=\ &((\mu \otimes \mu) \circ (\operatorname{id} \otimes \sigma \otimes \operatorname{id})) (\ket A \otimes \ket A \otimes \ket B \otimes \ket B) \\
=\ &(\mu \otimes \mu) (\ket A \otimes \ket B \otimes \ket A \otimes \ket B) \\
=\ &\ket{A \cap B} \otimes \ket{A \cap B} \\
=\ &\delta \ket{A \cap B} \\
=\ &(\delta \circ \mu) (\ket A \otimes \ket B)
\end{align*}

\medskip
\noindent
{\bf Example 3. Sets and Relations}.
Another important example of a  compact closed category is
$\Rel$, the cateogry of sets and relations. Here, $\otimes$ is
cartesian product with the singleton set as its unit $I = \{\star\}$, and $^*$ is identity on objects.  Hence $\Rel$ is also self adjoint. Closure reduces to the
fact that a relation between sets $A\times B$ and $C$ is equivalently a relation between $A$ and $B \times C$.   Given a set $S$ with elements $s_i, s_j \in S$,  the epsilon and eta maps are given as follows:

\begin{eqnarray*}
&\epsilon^l  =  \epsilon^r &\colon   S \times S \relto I \quad \mbox{given by} \quad
(s_i, s_j) \epsilon \star \iff s_i = s_j \\
&\eta^l = \eta^r& \colon   I  \relto S \times S
\quad \mbox{given by} \quad 
\star \eta (s_i, s_j) \iff s_i = s_j
\end{eqnarray*}

Transposes in the category of relations are given by inverse relations.


For an object in $\Rel$ of the form $S = \mathcal P (U)$, we give $S$ a bialgebra structure by taking
\begin{align*}
\delta &\colon   S \relto S \times S \quad &&\mbox{given by} \quad
A \delta (B, C) \iff A = B = C \\
\iota& \colon S \relto  I    \quad &&\mbox{given by} \quad A \iota \star \iff \text{ (always true)} \\
\mu & \colon   S \times S \relto S
\quad &&\mbox{given by} \quad 
(A, B) \mu C \iff A \cap B = C \\
\zeta& \colon  I  \relto S  \quad &&\mbox{given by} \quad \star \zeta A \iff A = U
\end{align*}
The axioms (Q1) -- (Q4) can be easily verified by the reader.

It should be noted that since both $\FdVect$ and $\Rel$ are $\dagger$-categories, these constructions dualize to give two pairs of bialgebras.
However these bialgebras are not interacting in the sense of \cite{bonchi14}, and the Frobenius axiom does not hold for either.

%
%
%


\subsection{String Diagrams} 
\label{string}

The framework of compact closed categories and bialgebras
comes with a  diagrammatic calculus that visualises
derivations, and which also simplifies the
categorical and vector space computations. Morphisms are depicted by
boxes and objects by lines, representing their identity morphisms. For
instance a morphism $f \colon A \to B$, and an object $A$ with the
identity arrow $1_A \colon A \to A$, are depicted as follows:

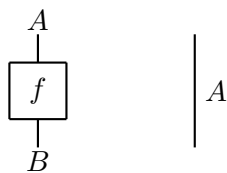
\begin{center}
  {%
\beginpgfgraphicnamed{compact-diag}
\begin{tikzpicture}[scale=0.75]
	\begin{pgfonlayer}{nodelayer}
		\node [style=none] (0) at (-9, -1) {};
		\node [style=none] (1) at (-7, -1) {};
		\node [style=none] (2) at (-7, 1) {};
		\node [style=none] (3) at (-9, 1) {};
		\node [style=none] (4) at (-8, 0) {$f$};
		\node [style=none] (5) at (-8, 1) {};
		\node [style=none] (6) at (-8, 2) {};
		\node [style=none] (7) at (-8, -1) {};
		\node [style=none] (8) at (-8, -2) {};
		\node [style=none] (9) at (-8, 2.5) {$A$};
		\node [style=none] (10) at (-8, -2.5) {$B$};
		\node [style=none] (11) at (-2.5, 2) {};
		\node [style=none] (12) at (-2.5, -2) {};
		\node [style=none] (13) at (-1.75, 0) {$A$};
	\end{pgfonlayer}
	\begin{pgfonlayer}{edgelayer}
		\draw [style=thick] (3.center) to (0.center);
		\draw [style=thick] (3.center) to (2.center);
		\draw [style=thick] (2.center) to (1.center);
		\draw [style=thick] (1.center) to (0.center);
		\draw [style=thick] (6.center) to (5.center);
		\draw [style=thick] (7.center) to (8.center);
		\draw [style=thick] (11.center) to (12.center);
	\end{pgfonlayer}
\end{tikzpicture}}
\endpgfgraphicnamed}
\end{center}

Morphisms from $I$ to objects are depicted by triangles with strings emanating from them. In concrete categories, these morphisms represent  elements within the  objects. For instance, an element $a$ in $A$ is represented by the morphism $a: I \to A$ and depicted by a  triangle with one string emanating from it. The number of strings of such triangles depict the tensor rank of the element; for instance, the
diagrams for ${a} \in A, {a'} \in A \otimes B$, and ${a''}
\in A \otimes B \otimes C$ are as follows:

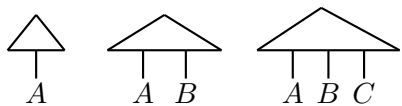
\begin{center}
  {%
\beginpgfgraphicnamed{compact-diag-triangle}
\begin{tikzpicture}[scale=0.75]
	\begin{pgfonlayer}{nodelayer}
		\node [style=none] (0) at (-5, 0) {};
		\node [style=none] (1) at (-1, 0) {};
		\node [style=none] (2) at (-3, 1.25) {};
		\node [style=none] (3) at (-3.75, 0) {};
		\node [style=none] (4) at (-3.75, -1) {};
		\node [style=none] (5) at (-3.75, -1.5) {$A$};
		\node [style=none] (6) at (-2.25, 0) {};
		\node [style=none] (7) at (-2.25, -1.5) {$B$};
		\node [style=none] (8) at (-2.25, -1) {};
		\node [style=none] (9) at (2.75, -1) {};
		\node [style=none] (10) at (2.75, -1.5) {$B$};
		\node [style=none] (11) at (2.5, 1.5) {};
		\node [style=none] (12) at (1.5, 0) {};
		\node [style=none] (13) at (5.25, 0) {};
		\node [style=none] (14) at (1.5, -1.5) {$A$};
		\node [style=none] (15) at (1.5, -1) {};
		\node [style=none] (16) at (0.25, 0) {};
		\node [style=none] (17) at (2.75, 0) {};
		\node [style=none] (18) at (4, -1.5) {$C$};
		\node [style=none] (19) at (4, 0) {};
		\node [style=none] (20) at (4, -1) {};
		\node [style=none] (21) at (-8.5, 0) {};
		\node [style=none] (22) at (-7.5, -1.5) {$A$};
		\node [style=none] (23) at (-7.5, 0) {};
		\node [style=none] (24) at (-6.5, 0) {};
		\node [style=none] (25) at (-7.5, -1) {};
		\node [style=none] (26) at (-7.5, 1.25) {};
		\node [style=none] (29) at (2.5, 0.5) {};
	\end{pgfonlayer}
	\begin{pgfonlayer}{edgelayer}
		\draw  [style = thick](0.center) to (2.center);
		\draw  [style = thick](2.center) to (1.center);
		\draw  [style = thick](1.center) to (0.center);
		\draw  [style = thick](3.center) to (4.center);
		\draw  [style = thick](6.center) to (8.center);
		\draw  [style = thick](16.center) to (11.center);
		\draw  [style = thick](11.center) to (13.center);
		\draw  [style = thick](13.center) to (16.center);
		\draw  [style = thick](12.center) to (15.center);
		\draw  [style = thick](17.center) to (9.center);
		\draw  [style = thick](19.center) to (20.center);
		\draw  [style = thick](21.center) to (26.center);
		\draw  [style = thick](26.center) to (24.center);
		\draw  [style = thick](24.center) to (21.center);
		\draw  [style = thick](23.center) to (25.center);
	\end{pgfonlayer}
\end{tikzpicture}}
\endpgfgraphicnamed}  
\end{center}

The tensor products of the objects and morphisms are depicted by
juxtaposing their diagrams side by side, whereas compositions of
morphisms are depicted by putting one on top of the other; for instance
the object $A \otimes B$, and the morphisms $f \otimes g$ and $h \circ
f$, for $f \colon A \to B, g \colon C \to D$, and $h \colon B \to C$,
are depicted as follows:

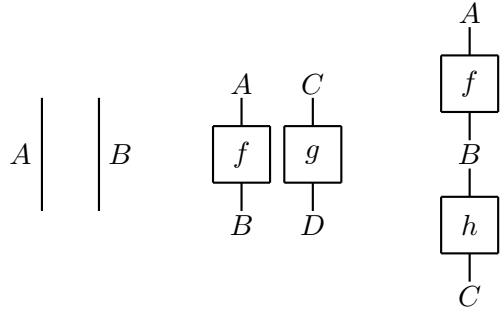
\begin{center}
  {%
\beginpgfgraphicnamed{compact-diag-tensor}
\begin{tikzpicture}[scale=0.75]
	\begin{pgfonlayer}{nodelayer}
		\node [style=none] (0) at (-4, 0) {};
		\node [style=none] (1) at (-2, 0) {};
		\node [style=none] (2) at (-2, 2) {};
		\node [style=none] (3) at (-4, 2) {};
		\node [style=none] (4) at (-3, 1) {$f$};
		\node [style=none] (5) at (-3, 2) {};
		\node [style=none] (6) at (-3, 3) {};
		\node [style=none] (7) at (-3, 0) {};
		\node [style=none] (8) at (-3, -1) {};
		\node [style=none] (9) at (-3, 3.5) {$A$};
		\node [style=none] (10) at (-3, -1.5) {$B$};
		\node [style=none] (11) at (-0.5, -1.5) {$D$};
		\node [style=none] (12) at (-0.5, 1) {$g$};
		\node [style=none] (13) at (-0.5, 2) {};
		\node [style=none] (14) at (-1.5, 0) {};
		\node [style=none] (15) at (-0.5, -1) {};
		\node [style=none] (16) at (0.5, 0) {};
		\node [style=none] (17) at (-0.5, 3.5) {$C$};
		\node [style=none] (18) at (-0.5, 3) {};
		\node [style=none] (19) at (0.5, 2) {};
		\node [style=none] (20) at (-1.5, 2) {};
		\node [style=none] (21) at (-0.5, 0) {};
		\node [style=none] (22) at (5, 2.5) {};
		\node [style=none] (23) at (5, 1.5) {};
		\node [style=none] (24) at (4, 4.5) {};
		\node [style=none] (25) at (5, 4.5) {};
		\node [style=none] (26) at (5, 5.5) {};
		\node [style=none] (27) at (5, 3.5) {$f$};
		\node [style=none] (28) at (6, 2.5) {};
		\node [style=none] (29) at (5, 6) {$A$};
		\node [style=none] (30) at (4, 2.5) {};
		\node [style=none] (31) at (5, 1) {$B$};
		\node [style=none] (32) at (6, 4.5) {};
		\node [style=none] (33) at (5, -2.5) {};
		\node [style=none] (34) at (5, -3.5) {};
		\node [style=none] (35) at (4, -0.5) {};
		\node [style=none] (36) at (5, -0.5) {};
		\node [style=none] (37) at (5, 0.5) {};
		\node [style=none] (38) at (5, -1.5) {$h$};
		\node [style=none] (39) at (6, -2.5) {};
		\node [style=none] (40) at (4, -2.5) {};
		\node [style=none] (41) at (5, -4) {$C$};
		\node [style=none] (42) at (6, -0.5) {};
		\node [style=none] (43) at (-10, -1) {};
		\node [style=none] (44) at (-10, 3) {};
		\node [style=none] (45) at (-10.75, 1) {$A$};
		\node [style=none] (46) at (-8, -1) {};
		\node [style=none] (47) at (-8, 3) {};
		\node [style=none] (48) at (-7.25, 1) {$B$};
	\end{pgfonlayer}
	\begin{pgfonlayer}{edgelayer}
		\draw  [style=thick] (3.center) to (0.center);
		\draw  [style=thick] (3.center) to (2.center);
		\draw  [style=thick] (2.center) to (1.center);
		\draw  [style=thick] (1.center) to (0.center);
		\draw  [style=thick] (6.center) to (5.center);
		\draw  [style=thick] (7.center) to (8.center);
		\draw  [style=thick] (20.center) to (14.center);
		\draw  [style=thick] (20.center) to (19.center);
		\draw  [style=thick] (19.center) to (16.center);
		\draw  [style=thick] (16.center) to (14.center);
		\draw  [style=thick] (18.center) to (13.center);
		\draw  [style=thick] (21.center) to (15.center);
		\draw  [style=thick] (24.center) to (30.center);
		\draw  [style=thick] (24.center) to (32.center);
		\draw  [style=thick] (32.center) to (28.center);
		\draw  [style=thick] (28.center) to (30.center);
		\draw  [style=thick] (26.center) to (25.center);
		\draw  [style=thick] (22.center) to (23.center);
		\draw  [style=thick] (35.center) to (40.center);
		\draw  [style=thick] (35.center) to (42.center);
		\draw  [style=thick] (42.center) to (39.center);
		\draw  [style=thick] (39.center) to (40.center);
		\draw  [style=thick] (37.center) to (36.center);
		\draw  [style=thick] (33.center) to (34.center);
		\draw [style=thick] (44.center) to (43.center);
		\draw [style=thick] (47.center) to (46.center);
	\end{pgfonlayer}
\end{tikzpicture}}
\endpgfgraphicnamed}
\end{center}

The $\epsilon$ maps are depicted by cups, $\eta$ maps by caps, and
yanking by their composition and straightening of the strings.  For
instance, the diagrams for $\epsilon^l \colon A^l \otimes A \to I$,
$\eta \colon I \to A\otimes A^l$ and $(\epsilon^l \otimes 1_A) \circ
(1_A \otimes \eta^l) = 1_A$ are as follows:

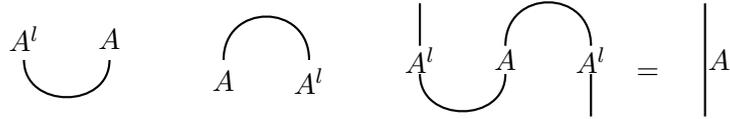
\begin{center}
  {%
\beginpgfgraphicnamed{compact-cap-cup}
\begin{tikzpicture}[scale=0.75]
	\begin{pgfonlayer}{nodelayer}
		\node [style=none] (0) at (-5, 0) {};
		\node [style=none] (1) at (-2, 0) {};
		\node [style=none] (2) at (-5, 0.75) {$A^l$};
		\node [style=none] (3) at (2, -0.75) {$A$};
		\node [style=none] (4) at (5, -0.75) {$A^l$};
		\node [style=none] (5) at (2, 0) {};
		\node [style=none] (6) at (5, 0) {};
		\node [style=none] (7) at (-2, 0.75) {$A$};
	\end{pgfonlayer}
	\begin{pgfonlayer}{edgelayer}
		\draw [thick, bend right=90, looseness=1.50] (0.center) to (1.center);
		\draw [thick, bend left=90, looseness=1.75] (5.center) to (6.center);
	\end{pgfonlayer}
\end{tikzpicture}}
\endpgfgraphicnamed}
  \qquad
    {%
\beginpgfgraphicnamed{compact-yank}
\begin{tikzpicture}[scale=0.75]
	\begin{pgfonlayer}{nodelayer}
		\node [style=none] (0) at (-5, 0) {};
		\node [style=none] (1) at (-2, 0) {};
		\node [style=none] (2) at (-5, 0.5) {$A^l$};
		\node [style=none] (3) at (-2, 0.5) {$A$};
		\node [style=none] (4) at (1, 0.5) {$A^l$};
		\node [style=none] (5) at (-2, 1) {};
		\node [style=none] (6) at (1, 1) {};
		\node [style=none] (7) at (3, 0) {$=$};
		\node [style=none] (8) at (5, 2.5) {};
		\node [style=none] (9) at (5, -1.5) {};
		\node [style=none] (10) at (5.5, 0.5) {$A$};
		\node [style=none] (11) at (-5, 1) {};
		\node [style=none] (12) at (-5, 2.5) {};
		\node [style=none] (13) at (1, 0) {};
		\node [style=none] (14) at (1, -1.5) {};
	\end{pgfonlayer}
	\begin{pgfonlayer}{edgelayer}
		\draw [thick, bend right=90, looseness=1.50] (0.center) to (1.center);
		\draw [thick, bend left=90, looseness=1.75] (5.center) to (6.center);
		\draw [style=thick] (8.center) to (9.center);
		\draw [style=thick] (12.center) to (11.center);
		\draw [style=thick] (13.center) to (14.center);
	\end{pgfonlayer}
\end{tikzpicture}}
\endpgfgraphicnamed}
\end{center}

As for  the bialgebra, the diagrams for the  monoid and  comonoid 
morphisms and their interaction (the bialgebra law Q3) are as follows:

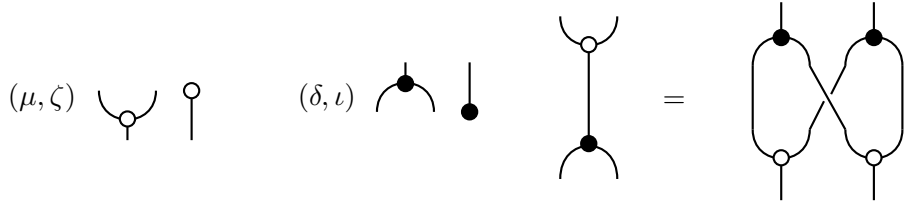
\begin{center}
{%
\beginpgfgraphicnamed{comp-alg-coalg}
\begin{tikzpicture}[scale=0.75]
	\begin{pgfonlayer}{nodelayer}
		\node [style=none] (0) at (-3.5, 2.25) {};
		\node [draw, thick, style=none, minimum size=0.2 cm, circle, fill=white] (1) at (-4.5, 1.25) {};
		\node [style=none] (2) at (-5.5, 2.25) {};
		\node [style=none] (3) at (-4.5, 0.5) {};
		\node [style=none] (4) at (4.25, 1.5) {};
		\node [style=none] (5) at (6.25, 1.5) {};
		\node [style=none] (6) at (4.25, 1.5) {};
		\node [draw, thick, style=none, minimum size=0.2 cm, circle, fill=black] (7) at (5.25, 2.5) {};
		\node [style=none] (8) at (5.25, 3.25) {};
		\node [style=none] (9) at (6.25, 1.5) {};
		\node [style=none] (10) at (-7.5, 2) {$(\mu,\zeta)$};
		\node [draw, thick, style=none, minimum size=0.2 cm, circle, fill=white] (11) at (-2.25, 2.25) {};
		\node [style=none] (12) at (-2.25, 0.5) {};
		\node [style=none] (13) at (2.5, 2) {$(\delta,\iota)$};
		\node [draw, thick, style=none, minimum size=0.2 cm, circle, fill=black] (14) at (7.5, 1.5) {};
		\node [style=none] (15) at (7.5, 3.25) {};
	\end{pgfonlayer}
	\begin{pgfonlayer}{edgelayer}
		\draw [style=thick] (1.center) to (3.center);
		\draw [style=thick, bend left=90, looseness=1.75] (6.center) to (5.center);
		\draw [style=thick] (8.center) to (7.center);
		\draw [thick, bend right=90, looseness=1.75] (2.center) to (0.center);
		\draw [style=thick] (11.center) to (12.center);
		\draw [style=thick] (15.center) to (14.center);
	\end{pgfonlayer}
\end{tikzpicture}}
\endpgfgraphicnamed} \qquad {%
\beginpgfgraphicnamed{bialg-equation}
\begin{tikzpicture}[scale=0.75]
	\begin{pgfonlayer}{nodelayer}
		\node [style=none] (0) at (-3, 0.5) {$=$};
		\node [style=none] (1) at (-5, 3.5) {};
		\node [style=none] (2) at (-7, 3.5) {};
		\node [draw, thick, style=none, minimum size=0.2 cm, circle, fill=white] (3) at (-6, 2.5) {};
		\node [style=none] (4) at (-7, -2.25) {};
		\node [style=none] (5) at (-5, -2.25) {};
		\node [draw, thick, style=none, minimum size=0.2 cm, circle, fill=black] (6) at (-6, -1) {};
		\node [style=none] (7) at (-0.25, -0.5) {};
		\node [style=none] (8) at (-0.25, 1.75) {};
		\node [style=none] (9) at (1.75, 1.75) {};
		\node [style=none] (10) at (-0.25, 1.75) {};
		\node [draw, thick, style=none, minimum size=0.2 cm, circle, fill=black] (11) at (0.75, 2.75) {};
		\node [style=none] (12) at (0.75, 4) {};
		\node [style=none] (13) at (1.75, 1.75) {};
		\node [style=none] (14) at (5, 1.75) {};
		\node [draw, thick, style=none, minimum size=0.2 cm, circle, fill=black] (15) at (4, 2.75) {};
		\node [style=none] (16) at (3, 1.75) {};
		\node [style=none] (17) at (4, 4) {};
		\node [style=none] (18) at (5, 1.75) {};
		\node [style=none] (19) at (3, 1.75) {};
		\node [draw, thick, style=none, minimum size=0.2 cm, circle, fill=white] (20) at (0.75, -1.5) {};
		\node [style=none] (21) at (1.75, -0.5) {};
		\node [style=none] (22) at (-0.25, -0.5) {};
		\node [draw, thick, style=none, minimum size=0.2 cm, circle, fill=white] (23) at (4, -1.5) {};
		\node [style=none] (24) at (5, -0.5) {};
		\node [style=none] (25) at (3, -0.5) {};
		\node [style=none] (26) at (0.75, -1.5) {};
		\node [style=none] (27) at (4, -1.5) {};
		\node [style=none] (28) at (5, -0.5) {};
		\node [style=none] (29) at (5, 1.75) {};
		\node [style=none] (30) at (1.75, 1.75) {};
		\node [style=none] (31) at (3, 1.75) {};
		\node [style=none] (32) at (2.5, 0.75) {};
		\node [style=none] (33) at (0.75, -3) {};
		\node [style=none] (34) at (4, -3) {};
		\node [style=none] (35) at (2.25, 0.5) {};
	\end{pgfonlayer}
	\begin{pgfonlayer}{edgelayer}
		\draw [thick, bend left=90, looseness=1.75] (1.center) to (2.center);
		\draw [style=thick, bend left=90, looseness=2.00] (4.center) to (5.center);
		\draw [style=thick, in=90, out=270] (3.center) to (6.center);
		\draw [style=thick] (7.center) to (8.center);
		\draw [style=thick, in=90, out=90, looseness=1.75] (10.center) to (9.center);
		\draw [style=thick] (12.center) to (11.center);
		\draw [style=thick, in=90, out=90, looseness=1.75] (16.center) to (14.center);
		\draw [style=thick] (17.center) to (15.center);
		\draw [thick, bend left=90, looseness=1.75] (21.center) to (22.center);
		\draw [thick, bend left=90, looseness=1.75] (24.center) to (25.center);
		\draw [style=thick] (29.center) to (28.center);
		\draw [style=thick] (13.center) to (25.center);
		\draw [style=thick] (31.center) to (32.center);
		\draw [style=thick] (26.center) to (33.center);
		\draw [style=thick] (27.center) to (34.center);
		\draw [style=thick] (35.center) to (21.center);
	\end{pgfonlayer}
\end{tikzpicture}}
\endpgfgraphicnamed}
\end{center} 
 
%
%
%
%


\section{Abstract Compact Closed Semantics}


\begin{definition}\label{ccc-model}
An abstract compact closed categorical model for the  language  generated by the grammar  $G = (T, N, S, {\cal R})$    is  a tuple $({\cal C}, {W,S}, \overline{\semantics{\ }})_{ {\cal B}}$ where:
\begin{itemize}
\item ${\cal C}$  is a self adjoint compact closed category  with two distinguished objects $W$ and $S$, where  $W$ has a bialgebra on it,  
\item ${\cal B}$ is  the compact closed category freely generated over the generators of the pregroup grammar $P_{G}$ as in Definition \ref{def:cfg-prg}  and the  atomic morphisms  $I \to \sigma(A)$ introduced by the atomic rules  $A \to x$ of $G$.
\item  $\overline{\semantics{\ }} \colon {\cal B} \to {\cal C}$ is a strongly monoidal functor defined as follows:
\[
 \overline{\semantics{x}}  := 
\begin{cases}
S &x = \sigma(S)  \ \text{for} \ S \ \mbox{the designated starting symbol of} \ G\\
 W &  x \in {\cal A} \setminus \{\sigma(S)\}\\
I \to \overline{\semantics{\sigma(x)}} & x \in \{\text{NP}, \text{N}, \text{VP}, \text{V}\}\\
\overline{\semantics{\sigma(x)}} \to \overline{\semantics{\sigma(x)}} & x \in \{ \text{Det}\}\\
 I \to \overline{\semantics{\sigma(x)}} & x \in T, (x, \sigma(A)) \in \beta
\end{cases}
\]
\end{itemize}
\end{definition}

Using a free compact closed category as opposed to a free pregroup is a solution suggested in \cite{PrellerFunct} to the fact that there is no corresponding strongly monoidal functor on a free pregroup that assigns to atomic types,  spaces of more than one dimension in the relational and distributional instantiations. In what follows, we drop the  source category ${\cal B}$ and denote the semantics by $({\cal C}, {W,S}, \overline{\semantics{\ }})$ in cases where the source category is fixed.  In particular, the abstract compact closed categorical model for the language generated by   $G_Q$  is the one used in the rest of this paper and thus  we will drop the ${\cal B}$ from the tuple notation  in the relational and vector space instantiations that follow.  

As an example of the application of the above definition, consider  the language of  $G_Q$, wherein the map $\overline{\semantics{\ }}$ on the terminals is defined  as follows:

\begin{eqnarray*}
 \text{NP} \to np &\quad \implies \quad& \overline{\semantics{np}} :=  I \to \overline{\semantics{\sigma(np)}} \colon I \to W \\
\text{N} \to n  &\quad \implies \quad& \overline{\semantics{n}} :=  I \to  \overline{\semantics{\sigma(n)}} \colon I \to W \\
\text{VP} \to vp&\quad \implies \quad& \overline{\semantics{vp}}   :=  I \to \overline{\semantics{\sigma(vp)}} \colon  I \to W^r \otimes S \\
 \text{V} \to v&\quad \implies \quad& \overline{\semantics{v}}  :=   I \to  \overline{\semantics{\sigma(v)}} \colon I \to W^r \otimes S \otimes W^l  \\
 \text{Det} \to d& \quad \implies \quad & \overline{\semantics{d}}  :=    \overline{\semantics{\sigma(d)}}  \to \overline{\semantics{\sigma(d)}} \colon W \to W
\end{eqnarray*}

\noindent
In diagrammatic form we have :

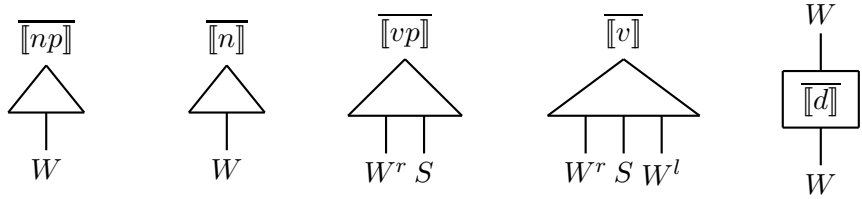
\begin{center}
{%
\beginpgfgraphicnamed{N-NP}
\begin{tikzpicture}
	\begin{pgfonlayer}{nodelayer}
		\node [style=none] (0) at (-0.25, 1) {};
		\node [style=none] (1) at (-1.25, 1) {};
		\node [style=none] (2) at (-2.25, 1) {};
		\node [style=none] (3) at (-1.25, 2.25) {};
		\node [style=none] (4) at (-1.25, -0.5) {$W$};
		\node [style=none] (5) at (-1.25, 0) {};
		\node [style=none] (6) at (3.5, -0.5) {$W$};
		\node [style=none] (7) at (2.5, 1) {};
		\node [style=none] (8) at (4.5, 1) {};
		\node [style=none] (9) at (3.5, 0) {};
		\node [style=none] (10) at (3.5, 1) {};
		\node [style=none] (11) at (3.5, 2.25) {};
		\node [style=none] (12) at (-1.25, 3) {$\overline{\semantics{np}}$};
		\node [style=none] (13) at (3.5, 3) {$\overline{\semantics{n}}$};
	\end{pgfonlayer}
	\begin{pgfonlayer}{edgelayer}
		\draw [style=thick] (2.center) to (3.center);
		\draw [style=thick] (3.center) to (0.center);
		\draw [style=thick] (0.center) to (2.center);
		\draw [style=thick] (1.center) to (5.center);
		\draw [style=thick] (7.center) to (11.center);
		\draw [style=thick] (11.center) to (8.center);
		\draw [style=thick] (8.center) to (7.center);
		\draw [style=thick] (10.center) to (9.center);
	\end{pgfonlayer}
\end{tikzpicture}}
\endpgfgraphicnamed} \qquad {%
\beginpgfgraphicnamed{V-VP}
\begin{tikzpicture}
	\begin{pgfonlayer}{nodelayer}
		\node [style=none] (0) at (4.5, 0) {};
		\node [style=none] (1) at (1.5, 0) {};
		\node [style=none] (2) at (0.5, 0) {};
		\node [style=none] (3) at (2.5, 1.5) {};
		\node [style=none] (4) at (1.5, -1.5) {$W^r$};
		\node [style=none] (5) at (1.5, -1) {};
		\node [style=none] (6) at (-4.75, 0) {};
		\node [style=none] (7) at (-3.25, 1.5) {};
		\node [style=none] (8) at (-1.75, 0) {};
		\node [style=none] (9) at (-3.75, -1.5) {$W^r$};
		\node [style=none] (10) at (-3.75, 0) {};
		\node [style=none] (11) at (-3.75, -1) {};
		\node [style=none] (12) at (-2.75, -1.5) {$S$};
		\node [style=none] (13) at (-2.75, -1) {};
		\node [style=none] (14) at (-2.75, 0) {};
		\node [style=none] (15) at (2.5, -1) {};
		\node [style=none] (16) at (2.5, 0) {};
		\node [style=none] (17) at (2.5, -1.5) {$S$};
		\node [style=none] (18) at (3.5, -1) {};
		\node [style=none] (19) at (3.5, 0) {};
		\node [style=none] (20) at (3.5, -1.5) {$W^l$};
		\node [style=none] (21) at (-3.25, 2.25) {$\overline{\semantics{vp}}$};
		\node [style=none] (22) at (2.5, 2.25) {$\overline{\semantics{v}}$};
	\end{pgfonlayer}
	\begin{pgfonlayer}{edgelayer}
		\draw [style=thick] (2.center) to (3.center);
		\draw [style=thick] (3.center) to (0.center);
		\draw [style=thick] (0.center) to (2.center);
		\draw [style=thick] (1.center) to (5.center);
		\draw [style=thick] (6.center) to (7.center);
		\draw [style=thick] (7.center) to (8.center);
		\draw [style=thick] (8.center) to (6.center);
		\draw [style=thick] (10.center) to (11.center);
		\draw [style=thick] (14.center) to (13.center);
		\draw [style=thick] (16.center) to (15.center);
		\draw [style=thick] (19.center) to (18.center);
	\end{pgfonlayer}
\end{tikzpicture}}
\endpgfgraphicnamed} \qquad {%
\beginpgfgraphicnamed{Det}
\begin{tikzpicture}
	\begin{pgfonlayer}{nodelayer}
		\node [style=none] (0) at (-2.5, -0.25) {$\overline{\semantics{d}}$};
		\node [style=none] (1) at (-3.5, 0.5) {};
		\node [style=none] (2) at (-2.5, -2.5) {$W$};
		\node [style=none] (3) at (-1.5, 0.5) {};
		\node [style=none] (4) at (-2.5, -1) {};
		\node [style=none] (5) at (-3.5, -1) {};
		\node [style=none] (6) at (-2.5, 0.5) {};
		\node [style=none] (7) at (-1.5, -1) {};
		\node [style=none] (8) at (-2.5, -2) {};
		\node [style=none] (9) at (-2.5, 2) {$W$};
		\node [style=none] (10) at (-2.5, 1.5) {};
	\end{pgfonlayer}
	\begin{pgfonlayer}{edgelayer}
		\draw [style=thick] (1.center) to (3.center);
		\draw [style=thick] (3.center) to (7.center);
		\draw [style=thick] (7.center) to (5.center);
		\draw [style=thick] (5.center) to (1.center);
		\draw [style=thick] (4.center) to (8.center);
		\draw [style=thick] (10.center) to (6.center);
	\end{pgfonlayer}
\end{tikzpicture}}
\endpgfgraphicnamed}
\end{center}

\noindent
Intuitively, noun phrases and nouns are elements  within the  object $W$. Verb phrases are elements within the object $W^r \otimes S$; the intuition behind this representation is that in a compact closed category we have that $W^r \otimes S \cong W \to S$, where $W^r \to S = \hom(W,S)$ is an internal hom object of the category, coming from its monoidal closedness.  Hence,  we are modelling verb phrases as morphisms  with  input $W$ and  output $S$. Similarly, verbs are elements within the object $W^r \otimes S \otimes W^r$, equivalent to morphisms $W \otimes W \to S$ with pairs of  input  from $W$ and output  $S$. Determiners  are morphisms $W \to W$ that further satisfy the categorical version of the \emph{living on} property, defined below.

\begin{definition}  
A determiner $d$ satisfies the \emph{categorical living-on} property  in the abstract compact closed categorical model $({\cal C}, W, S, \overline{\semantics{\ }})$ generated by $G_Q$, if
{\small
\begin{align*}
\overline{\semantics{d}} = &(1_W \otimes \epsilon_W)
\circ (1_W \otimes \mu_W \otimes \epsilon_W \otimes 1_W)
\circ (1_W \otimes \overline{\semantics{d}}^\top \otimes \delta_W \otimes 1_{W\otimes W})
 \circ (1_W \otimes \eta_W \otimes 1_{W \otimes W})
 \circ (\eta_W \otimes 1_W)
 \end{align*}}
 where $-^\top$ denotes transposition in ${\cal C}$.
 \end{definition}
 
 In the concrete relational and boolean vector interpretations that we will define, this condition will be equivalent to Barwise and Cooper's living on property.
 Diagrammatically,  this stipulation means that we have the following equality of diagrams:

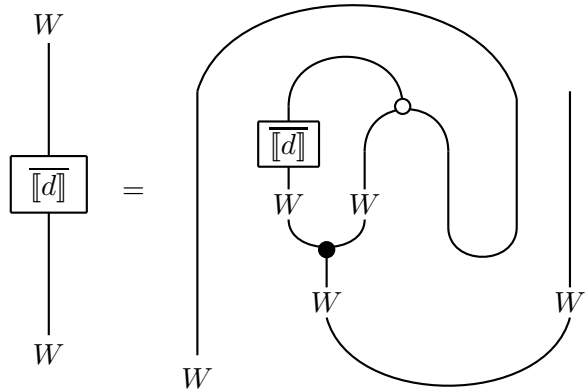
\begin{center}
\begin{minipage}{1cm}{%
\beginpgfgraphicnamed{Det-N-simple}
\begin{tikzpicture}
	\begin{pgfonlayer}{nodelayer}
		\node [style=none] (0) at (-1, -1) {};
		\node [style=none] (1) at (0, -1) {};
		\node [style=none] (2) at (0, 0.5) {};
		\node [style=none] (3) at (1, 0.5) {};
		\node [style=none] (4) at (1, -1) {};
		\node [style=none] (5) at (0, -0.25) {$\overline{\semantics{d}}$};
		\node [style=none] (6) at (0, 3.5) {};
		\node [style=none] (7) at (0, -4.25) {};
		\node [style=none] (8) at (0, -4.75) {$W$};
		\node [style=none] (9) at (0, 4) {$W$};
		\node [style=none] (10) at (-1, 0.5) {};
	\end{pgfonlayer}
	\begin{pgfonlayer}{edgelayer}
		\draw [style=thick] (10.center) to (3.center);
		\draw [style=thick] (3.center) to (4.center);
		\draw [style=thick] (4.center) to (0.center);
		\draw [style=thick] (0.center) to (10.center);
		\draw [style=thick] (1.center) to (7.center);
		\draw [style=thick] (6.center) to (2.center);
	\end{pgfonlayer}
\end{tikzpicture}}
\endpgfgraphicnamed} \end{minipage}
\quad $=$ \quad 
\begin{minipage}{7.5cm}{%
\beginpgfgraphicnamed{Det-N}
\begin{tikzpicture}[scale=0.8]
	\begin{pgfonlayer}{nodelayer}
		\node [draw, thick, style=none, minimum size=0.2 cm, circle, fill=black] (0) at (-4.75, -1.25) {};
		\node [style=none] (1) at (-6, 3.5) {};
		\node [style=none] (2) at (-2.25, 3.5) {};
		\node [style=none] (3) at (-7, 3) {};
		\node [style=none] (4) at (-5, 3) {};
		\node [style=none] (5) at (-5, 1.5) {};
		\node [style=none] (6) at (-7, 1.5) {};
		\node [style=none] (7) at (-6, 1.5) {};
		\node [style=none] (8) at (-4.75, -2.5) {};
		\node [style=none] (9) at (-6, 2.25) {$\overline{\semantics{d}}$};
		\node [style=none] (10) at (-4.75, -3) {$W$};
		\node [style=none] (11) at (-4.75, -1.25) {};
		\node [style=none] (12) at (-6, 0.75) {};
		\node [style=none] (13) at (-6, 3) {};
		\node [style=none] (14) at (-6, 0.25) {$W$};
		\node [style=none] (15) at (-6, -0.25) {};
		\node [style=none] (16) at (-3.5, -0.25) {};
		\node [style=none] (17) at (-3.5, 0.25) {$W$};
		\node [style=none] (18) at (-9, 4) {};
		\node [style=none] (19) at (-9, -4.75) {};
		\node [style=none] (20) at (-9, -5.5) {$W$};
		\node [style=none] (21) at (3.25, -3) {$W$};
		\node [style=none] (22) at (3.25, 6.75) {};
		\node [style=none] (23) at (3.25, -2.5) {};
		\node [style=none] (24) at (-4.75, -3.5) {};
		\node [style=none] (25) at (3.25, -3.5) {};
		\node [style=none] (27) at (-3.5, 2) {};
		\node [style=none] (28) at (-0.75, -0.5) {};
		\node [draw, thick, style=none, minimum size=0.2 cm, circle, fill=white] (29) at (-2.25, 3.5) {};
		\node [style=none] (30) at (-3.5, 2) {};
		\node [style=none] (31) at (-3.5, 0.75) {};
		\node [style=none] (32) at (1.5, -0.5) {};
		\node [style=none] (33) at (1.5, 3.75) {};
		\node [style=none] (34) at (-0.75, -0.5) {};
		\node [style=none] (35) at (-0.75, 2) {};
		\node [style=none] (36) at (3.25, 4) {};
	\end{pgfonlayer}
	\begin{pgfonlayer}{edgelayer}
		\draw [style=thick, in=90, out=90, looseness=1.50] (1.center) to (2.center);
		\draw [style=thick] (3.center) to (4.center);
		\draw [style=thick] (4.center) to (5.center);
		\draw [style=thick] (5.center) to (6.center);
		\draw [style=thick] (6.center) to (3.center);
		\draw [style=thick] (11.center) to (8.center);
		\draw [style=thick] (7.center) to (12.center);
		\draw [style=thick, bend right=90, looseness=1.25] (15.center) to (16.center);
		\draw [style=thick, bend right=75] (24.center) to (25.center);
		\draw [style=thick] (30.center) to (31.center);
		\draw [style=thick] (18.center) to (19.center);
		\draw [style=thick] (1.center) to (13.center);
		\draw [style=thick, bend right=90, looseness=1.50] (28.center) to (32.center);
		\draw [style=thick] (33.center) to (32.center);
		\draw [style=thick, bend left=75] (18.center) to (33.center);
		\draw [style=thick] (35.center) to (34.center);
		\draw [style=thick, bend left=90, looseness=1.75] (30.center) to (35.center);
		\draw [style=thick] (36.center) to (23.center);
	\end{pgfonlayer}
\end{tikzpicture}}
\endpgfgraphicnamed} \end{minipage} 
\end{center}

Note that we also have
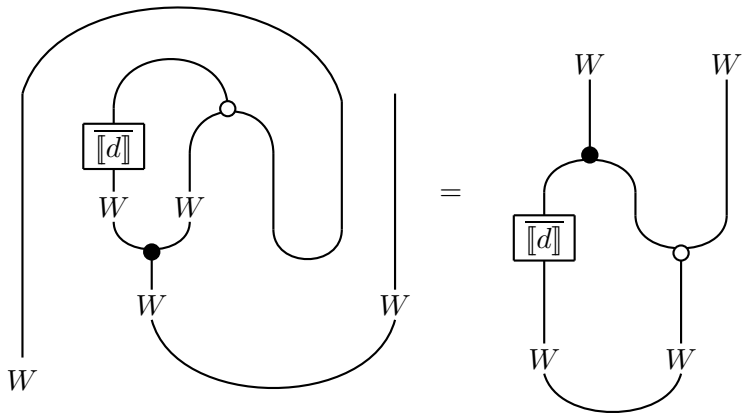
\begin{center}
\begin{minipage}{7.5cm}{%
\beginpgfgraphicnamed{Det-N}
\begin{tikzpicture}[scale=0.8]
	\begin{pgfonlayer}{nodelayer}
		\node [draw, thick, style=none, minimum size=0.2 cm, circle, fill=black] (0) at (-4.75, -1.25) {};
		\node [style=none] (1) at (-6, 3.5) {};
		\node [style=none] (2) at (-2.25, 3.5) {};
		\node [style=none] (3) at (-7, 3) {};
		\node [style=none] (4) at (-5, 3) {};
		\node [style=none] (5) at (-5, 1.5) {};
		\node [style=none] (6) at (-7, 1.5) {};
		\node [style=none] (7) at (-6, 1.5) {};
		\node [style=none] (8) at (-4.75, -2.5) {};
		\node [style=none] (9) at (-6, 2.25) {$\overline{\semantics{d}}$};
		\node [style=none] (10) at (-4.75, -3) {$W$};
		\node [style=none] (11) at (-4.75, -1.25) {};
		\node [style=none] (12) at (-6, 0.75) {};
		\node [style=none] (13) at (-6, 3) {};
		\node [style=none] (14) at (-6, 0.25) {$W$};
		\node [style=none] (15) at (-6, -0.25) {};
		\node [style=none] (16) at (-3.5, -0.25) {};
		\node [style=none] (17) at (-3.5, 0.25) {$W$};
		\node [style=none] (18) at (-9, 4) {};
		\node [style=none] (19) at (-9, -4.75) {};
		\node [style=none] (20) at (-9, -5.5) {$W$};
		\node [style=none] (21) at (3.25, -3) {$W$};
		\node [style=none] (22) at (3.25, 6.75) {};
		\node [style=none] (23) at (3.25, -2.5) {};
		\node [style=none] (24) at (-4.75, -3.5) {};
		\node [style=none] (25) at (3.25, -3.5) {};
		\node [style=none] (27) at (-3.5, 2) {};
		\node [style=none] (28) at (-0.75, -0.5) {};
		\node [draw, thick, style=none, minimum size=0.2 cm, circle, fill=white] (29) at (-2.25, 3.5) {};
		\node [style=none] (30) at (-3.5, 2) {};
		\node [style=none] (31) at (-3.5, 0.75) {};
		\node [style=none] (32) at (1.5, -0.5) {};
		\node [style=none] (33) at (1.5, 3.75) {};
		\node [style=none] (34) at (-0.75, -0.5) {};
		\node [style=none] (35) at (-0.75, 2) {};
		\node [style=none] (36) at (3.25, 4) {};
	\end{pgfonlayer}
	\begin{pgfonlayer}{edgelayer}
		\draw [style=thick, in=90, out=90, looseness=1.50] (1.center) to (2.center);
		\draw [style=thick] (3.center) to (4.center);
		\draw [style=thick] (4.center) to (5.center);
		\draw [style=thick] (5.center) to (6.center);
		\draw [style=thick] (6.center) to (3.center);
		\draw [style=thick] (11.center) to (8.center);
		\draw [style=thick] (7.center) to (12.center);
		\draw [style=thick, bend right=90, looseness=1.25] (15.center) to (16.center);
		\draw [style=thick, bend right=75] (24.center) to (25.center);
		\draw [style=thick] (30.center) to (31.center);
		\draw [style=thick] (18.center) to (19.center);
		\draw [style=thick] (1.center) to (13.center);
		\draw [style=thick, bend right=90, looseness=1.50] (28.center) to (32.center);
		\draw [style=thick] (33.center) to (32.center);
		\draw [style=thick, bend left=75] (18.center) to (33.center);
		\draw [style=thick] (35.center) to (34.center);
		\draw [style=thick, bend left=90, looseness=1.75] (30.center) to (35.center);
		\draw [style=thick] (36.center) to (23.center);
	\end{pgfonlayer}
\end{tikzpicture}}
\endpgfgraphicnamed} \end{minipage} 
\hspace{-2cm} $=$\quad  \hspace{2cm}
\begin{minipage}{3cm}{%
\beginpgfgraphicnamed{Det-N-norm2}
\begin{tikzpicture}[scale=0.8]
	\path [use as bounding box] (2.5,0.75) rectangle (2.5,3.75);
	\begin{pgfonlayer}{nodelayer}
		\node [style=none] (0) at (-2, 6.75) {};
		\node [style=none] (1) at (-2, 0) {};
		\node [style=none] (2) at (-1, 1.5) {};
		\node [style=none] (3) at (0.5, 6.5) {$W$};
		\node [style=none] (4) at (2, 2.25) {};
		\node [style=none] (5) at (-1, -2.75) {};
		\node [style=none] (6) at (-1, 2.25) {};
		\node [style=none] (7) at (-1, 0.75) {$\overline{\semantics{d}}$};
		\node [draw, thick, style=none, minimum size=0.2 cm, circle, fill=black] (8) at (0.5, 3.5) {};
		\node [style=none] (9) at (3.5, -3.25) {$W$};
		\node [style=none] (10) at (0, 0) {};
		\node [style=none] (11) at (-1, 0) {};
		\node [style=none] (12) at (3.5, 0.25) {};
		\node [style=none] (13) at (-2, 1.5) {};
		\node [style=none] (14) at (3.5, 0.25) {};
		\node [style=none] (15) at (0.5, 3.5) {};
		\node [style=none] (16) at (0, 1.5) {};
		\node [style=none] (17) at (0.5, 6) {};
		\node [style=none] (18) at (-1, -3.25) {$W$};
		\node [style=none] (19) at (-1, -3.75) {};
		\node [style=none] (20) at (3.5, -3.75) {};
		\node [style=none] (21) at (5, 1.5) {};
		\node [style=none] (22) at (5, 6.5) {$W$};
		\node [draw, thick, style=none, minimum size=0.2 cm, circle, fill=white] (23) at (3.5, 0.25) {};
		\node [style=none] (24) at (2, 1.5) {};
		\node [style=none] (25) at (3.5, -2.75) {};
		\node [style=none] (26) at (5, 6) {};
		\node [style=none] (27) at (5, 1.5) {};
	\end{pgfonlayer}
	\begin{pgfonlayer}{edgelayer}
		\draw [style=thick] (13.center) to (16.center);
		\draw [style=thick] (16.center) to (10.center);
		\draw [style=thick] (10.center) to (1.center);
		\draw [style=thick] (1.center) to (13.center);
		\draw [style=thick] (6.center) to (2.center);
		\draw [style=thick] (11.center) to (5.center);
		\draw [style=thick, in=90, out=270] (17.center) to (15.center);
		\draw [style=thick, bend left=90, looseness=1.25] (6.center) to (4.center);
		\draw [style=thick, bend right=75] (19.center) to (20.center);
		\draw [style=thick] (24.center) to (4.center);
		\draw [style=thick, bend right=90, looseness=1.25] (24.center) to (21.center);
		\draw [style=thick] (26.center) to (27.center);
		\draw [style=thick] (23.center) to (25.center);
	\end{pgfonlayer}
\end{tikzpicture}}
\endpgfgraphicnamed} \end{minipage}
\end{center}

Intuitively, semantics of  $\overline{\semantics{d}}$ ends up being in $W \otimes W$,  obtained by making a copy (via the bialgebra map $\delta$) of one of the  inputs  in $W$,  applying the determiner  to one copy and taking the intersection of the other copy (via the bialgebra map $\mu$) with the other input in $W$.  

Meanings of expressions of language are obtained according to the following definition:

\begin{definition}\label{composition}
The interpretation of a string $w_1 \cdots w_n$,  for $w_i \in T$ with a grammatical reduction $\alpha$ is
\[
\overline{\semantics{w_1 \cdots w_n}} := \overline{\semantics{\alpha}} \circ (\overline{\semantics{w_1}} \otimes \cdots \otimes \overline{\semantics{w_n}})
\] 
\end{definition}

For example, the interpretation of an intransitive  sentence with a  quantified phrase in  subject position  and its simplified forms are as follows:

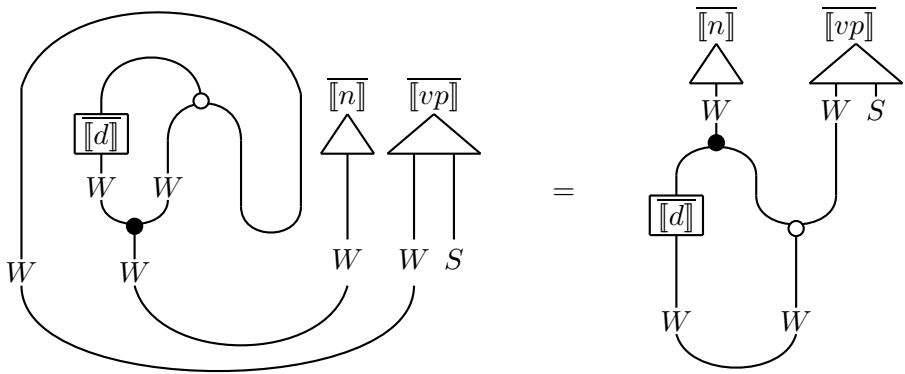
\begin{minipage}{20cm}
\begin{minipage}{7cm}
{%
\beginpgfgraphicnamed{Q-Sbj-Frob-Sent}
\begin{tikzpicture}[scale=0.7]
	\begin{pgfonlayer}{nodelayer}
		\node [style=none] (0) at (6.25, 2.5) {};
		\node [style=none] (1) at (8.25, 2.5) {};
		\node [style=none] (2) at (7.25, 4) {};
		\node [style=none] (3) at (7.25, 4.75) {$\overline{\semantics{n}}$};
		\node [style=none] (4) at (8.75, 2.5) {};
		\node [style=none] (5) at (12.25, 2.5) {};
		\node [style=none] (6) at (10.5, 4) {};
		\node [style=none] (7) at (10.5, 4.75) {$\overline{\semantics{vp}}$};
		\node [style=none] (8) at (7.25, 2.5) {};
		\node [style=none] (9) at (9.75, 2.5) {};
		\node [style=none] (10) at (11.25, 2.5) {};
		\node [style=none] (11) at (7.25, -0.75) {};
		\node [style=none] (12) at (7.25, -1.5) {$W$};
		\node [style=none] (13) at (9.75, -0.75) {};
		\node [style=none] (14) at (9.75, -1.5) {$W$};
		\node [style=none] (15) at (11.25, -0.75) {};
		\node [style=none] (16) at (11.25, -1.5) {$S$};
		\node [draw, thick, style=none, minimum size=0.2 cm, circle, fill=black] (17) at (-0.75, -0.25) {};
		\node [style=none] (18) at (-2, 1.25) {$W$};
		\node [style=none] (19) at (-0.75, -2.5) {};
		\node [style=none] (20) at (1.75, 4.5) {};
		\node [style=none] (21) at (0.5, 3) {};
		\node [style=none] (22) at (-3, 2.5) {};
		\node [style=none] (23) at (-0.75, -0.25) {};
		\node [style=none] (24) at (-3, 4) {};
		\node [style=none] (25) at (-2, 4.5) {};
		\node [style=none] (26) at (7.25, -2.5) {};
		\node [style=none] (27) at (-2, 3.25) {$\overline{\semantics{d}}$};
		\node [style=none] (28) at (-5, 5) {};
		\node [style=none] (29) at (-2, 4) {};
		\node [style=none] (30) at (-0.75, -2) {$W$};
		\node [style=none] (31) at (0.5, 1.75) {};
		\node [style=none] (32) at (-1, 2.5) {};
		\node [style=none] (33) at (0.5, 1.25) {$W$};
		\node [style=none] (34) at (-2, 2.5) {};
		\node [style=none] (35) at (5.5, 4.75) {};
		\node [style=none] (36) at (-2, 0.75) {};
		\node [style=none] (37) at (0.5, 3) {};
		\node [draw, thick, style=none, minimum size=0.2 cm, circle, fill=white] (38) at (1.75, 4.5) {};
		\node [style=none] (39) at (-0.75, -1.5) {};
		\node [style=none] (40) at (-1, 4) {};
		\node [style=none] (41) at (3.25, 0.5) {};
		\node [style=none] (42) at (5.5, 0.5) {};
		\node [style=none] (43) at (-5, -1.5) {};
		\node [style=none] (44) at (-2, 1.75) {};
		\node [style=none] (45) at (3.25, 3) {};
		\node [style=none] (46) at (0.5, 0.75) {};
		\node [style=none] (47) at (-5, -2) {$W$};
		\node [style=none] (48) at (3.25, 0.5) {};
		\node [style=none] (49) at (-5, -2.5) {};
		\node [style=none] (50) at (9.75, -2.5) {};
	\end{pgfonlayer}
	\begin{pgfonlayer}{edgelayer}
		\draw  [style=thick](0.center) to (1.center);
		\draw [style=thick] (2.center) to (0.center);
		\draw  [style=thick](2.center) to (1.center);
		\draw  [style=thick](4.center) to (5.center);
		\draw  [style=thick](6.center) to (4.center);
		\draw  [style=thick](6.center) to (5.center);
		\draw  [style=thick](8.center) to (11.center);
		\draw  [style=thick](9.center) to (13.center);
		\draw  [style=thick](10.center) to (15.center);
		\draw [style=thick, in=90, out=90, looseness=1.50] (25.center) to (20.center);
		\draw [style=thick] (24.center) to (40.center);
		\draw [style=thick] (40.center) to (32.center);
		\draw [style=thick] (32.center) to (22.center);
		\draw [style=thick] (22.center) to (24.center);
		\draw [style=thick] (23.center) to (39.center);
		\draw [style=thick] (34.center) to (44.center);
		\draw [style=thick, bend right=90, looseness=1.25] (36.center) to (46.center);
		\draw [style=thick, bend right=75] (19.center) to (26.center);
		\draw [style=thick] (37.center) to (31.center);
		\draw [style=thick] (28.center) to (43.center);
		\draw [style=thick] (25.center) to (29.center);
		\draw [style=thick, bend right=90, looseness=1.50] (48.center) to (42.center);
		\draw [style=thick] (35.center) to (42.center);
		\draw [style=thick, bend left=75] (28.center) to (35.center);
		\draw [style=thick] (45.center) to (41.center);
		\draw [style=thick, bend left=90, looseness=1.75] (37.center) to (45.center);
		\draw [style=thick, bend right=90, looseness=0.75] (49.center) to (50.center);
	\end{pgfonlayer}
\end{tikzpicture}}
\endpgfgraphicnamed}
\end{minipage}
\ $=$ \ \qquad
\begin{minipage}{5cm}
 {%
\beginpgfgraphicnamed{Q-Sbj-Norm}
\begin{tikzpicture}[scale=0.7]
	\begin{pgfonlayer}{nodelayer}
		\node [style=none] (0) at (-2.5, 3.75) {};
		\node [style=none] (1) at (-1.75, 5.25) {};
		\node [style=none] (2) at (-7, 3.25) {};
		\node [style=none] (3) at (-1, 3.75) {};
		\node [style=none] (4) at (-1, 2.75) {$S$};
		\node [style=none] (5) at (-2.5, 2.75) {$W$};
		\node [style=none] (6) at (-7, 3.75) {};
		\node [style=none] (7) at (-7, 5.25) {};
		\node [style=none] (8) at (-6, 3.75) {};
		\node [style=none] (9) at (-1, 3.25) {};
		\node [style=none] (10) at (0, 3.75) {};
		\node [style=none] (11) at (-3.5, 3.75) {};
		\node [style=none] (12) at (-7, 6) {$\overline{\semantics{n}}$};
		\node [style=none] (13) at (-2, 6) {$\overline{\semantics{vp}}$};
		\node [style=none] (14) at (-8, 3.75) {};
		\node [style=none] (15) at (-2.5, 3.25) {};
		\node [style=none] (16) at (-8.5, -5.25) {$W$};
		\node [style=none] (17) at (-5.5, -0.5) {};
		\node [style=none] (18) at (-8.5, -0.5) {};
		\node [style=none] (19) at (-4, -4.75) {};
		\node [draw, thick, style=none, minimum size=0.2 cm, circle, fill=white] (20) at (-4, -1.75) {};
		\node [style=none] (21) at (-9.5, -2) {};
		\node [style=none] (22) at (-8.5, -2) {};
		\node [style=none] (23) at (-7.5, -0.5) {};
		\node [style=none] (24) at (-7, 1.5) {};
		\node [style=none] (25) at (-4, -1.75) {};
		\node [style=none] (26) at (-2.5, -0.5) {};
		\node [style=none] (27) at (-8.5, -1.25) {$\overline{\semantics{d}}$};
		\node [style=none] (28) at (-4, -5.75) {};
		\node [style=none] (29) at (-7, 2.75) {$W$};
		\node [style=none] (30) at (-9.5, -0.5) {};
		\node [style=none] (31) at (1.5, 6.75) {};
		\node [style=none] (32) at (-7, 2.25) {};
		\node [style=none] (33) at (-4, -1.75) {};
		\node [style=none] (34) at (-5.5, 0.25) {};
		\node [style=none] (35) at (-2.5, -0.5) {};
		\node [style=none] (36) at (-7.5, -2) {};
		\node [style=none] (37) at (-4, -5.25) {$W$};
		\node [style=none] (38) at (-8.5, 0.25) {};
		\node [style=none] (40) at (-2.5, 2.25) {};
		\node [style=none] (41) at (-8.5, -4.75) {};
		\node [draw, thick, style=none, minimum size=0.2 cm, circle, fill=black] (42) at (-7, 1.5) {};
		\node [style=none] (43) at (-8.5, -5.75) {};
	\end{pgfonlayer}
	\begin{pgfonlayer}{edgelayer}
		\draw  [style=thick](14.center) to (8.center);
		\draw  [style=thick](7.center) to (14.center);
		\draw  [style=thick](7.center) to (8.center);
		\draw  [style=thick](11.center) to (10.center);
		\draw  [style=thick](1.center) to (11.center);
		\draw  [style=thick](1.center) to (10.center);
		\draw  [style=thick](6.center) to (2.center);
		\draw  [style=thick](0.center) to (15.center);
		\draw  [style=thick](3.center) to (9.center);
		\draw [style=thick] (30.center) to (23.center);
		\draw [style=thick] (23.center) to (36.center);
		\draw [style=thick] (36.center) to (21.center);
		\draw [style=thick] (21.center) to (30.center);
		\draw [style=thick] (38.center) to (18.center);
		\draw [style=thick] (22.center) to (41.center);
		\draw [style=thick, in=90, out=270] (32.center) to (24.center);
		\draw [style=thick, bend left=90, looseness=1.25] (38.center) to (34.center);
		\draw [style=thick, bend right=75] (43.center) to (28.center);
		\draw [style=thick] (17.center) to (34.center);
		\draw [style=thick, bend right=90, looseness=1.25] (17.center) to (26.center);
		\draw [style=thick] (40.center) to (35.center);
		\draw [style=thick] (20.center) to (19.center);
	\end{pgfonlayer}
\end{tikzpicture}}
\endpgfgraphicnamed}
\end{minipage}
\end{minipage}

The interpretation of a  transitive sentence with a  quantified phrase in  object position is as follows:

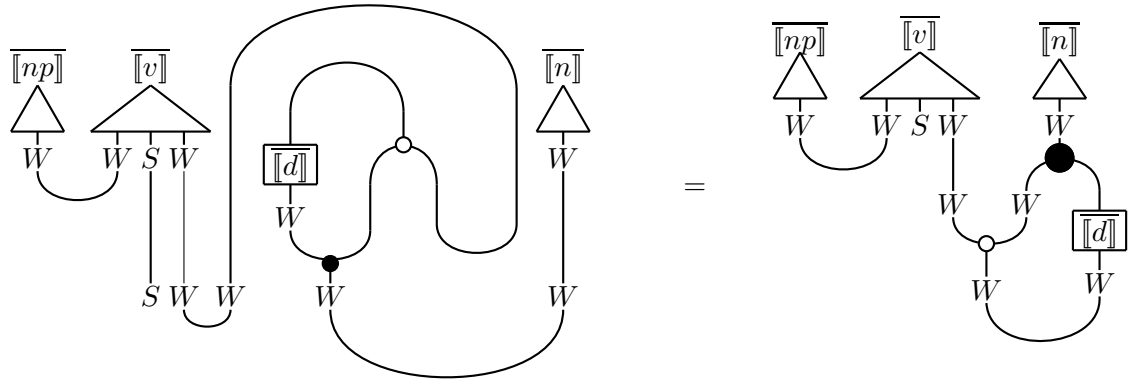
\begin{minipage}{20cm}
\begin{minipage}{8cm}
{%
\beginpgfgraphicnamed{Q-Obj-Frob-Sent}
\begin{tikzpicture}[scale=0.7]
	\begin{pgfonlayer}{nodelayer}
		\node [style=none] (0) at (9.75, 5.5) {$\overline{\semantics{n}}$};
		\node [style=none] (1) at (9.75, 3) {};
		\node [style=none] (2) at (-10, 2) {$W$};
		\node [style=none] (3) at (-3.5, 3) {};
		\node [style=none] (4) at (-5.75, 2) {$S$};
		\node [style=none] (5) at (9.75, 4.75) {};
		\node [style=none] (6) at (-4.5, 3) {};
		\node [style=none] (7) at (9.75, 2.5) {};
		\node [style=none] (8) at (10.75, 3) {};
		\node [style=none] (9) at (9.75, 2) {$W$};
		\node [style=none] (10) at (-5.75, 4.75) {};
		\node [style=none] (11) at (-4.5, 2) {$W$};
		\node [style=none] (12) at (-8, 3) {};
		\node [style=none] (13) at (8.75, 3) {};
		\node [style=none] (14) at (-10, 2.5) {};
		\node [style=none] (15) at (-10, 5.5) {$\overline{\semantics{np}}$};
		\node [style=none] (16) at (-11, 3) {};
		\node [style=none] (17) at (-5.75, 2.5) {};
		\node [style=none] (18) at (-10, 3) {};
		\node [style=none] (19) at (-10, 4.75) {};
		\node [style=none] (20) at (-9, 3) {};
		\node [style=none] (21) at (-7, 2) {$W$};
		\node [style=none] (22) at (-4.5, 2.5) {};
		\node [style=none] (23) at (-5.75, 3) {};
		\node [style=none] (24) at (-7, 3) {};
		\node [style=none] (25) at (-7, 2.5) {};
		\node [style=none] (26) at (-5.75, 5.5) {$\overline{\semantics{v}}$};
		\node [style=none] (27) at (-10, 1.5) {};
		\node [style=none] (28) at (-7, 1.5) {};
		\node [style=none] (29) at (-4.5, 1.5) {};
		\node [style=none] (30) at (-4.5, -2.75) {};
		\node [style=none] (31) at (-5.75, -3.25) {$S$};
		\node [style=none] (32) at (-4.5, -3.75) {};
		\node [style=none] (33) at (-2.75, -3.75) {};
		\node [style=none] (34) at (9.75, 1.5) {};
		\node [style=none] (35) at (9.75, -3.25) {$W$};
		\node [style=none] (36) at (9.75, -2.75) {};
		\node [style=none] (37) at (1, -3.75) {};
		\node [style=none] (38) at (9.75, -3.75) {};
		\node [style=none] (39) at (-5.75, 1.5) {};
		\node [style=none] (40) at (-5.75, -2.75) {};
		\node [style=none] (41) at (-4.5, -3.25) {$W$};
		\node [style=none] (42) at (1, -2.75) {};
		\node [draw, thick, style=none, minimum size=0.2 cm, circle, fill=black] (43) at (1, -2) {};
		\node [style=none] (44) at (-0.5, 1.75) {$\overline{\semantics{d}}$};
		\node [style=none] (45) at (-2.75, -2.75) {};
		\node [style=none] (46) at (-1.5, 2.5) {};
		\node [style=none] (47) at (-0.5, 0.25) {};
		\node [style=none] (48) at (8, 2) {};
		\node [style=none] (49) at (1, -2) {};
		\node [style=none] (50) at (2.5, -0.75) {};
		\node [style=none] (51) at (-2.75, 4.75) {};
		\node [style=none] (52) at (-0.5, 2.5) {};
		\node [style=none] (53) at (0.5, 1) {};
		\node [style=none] (54) at (-2.75, -3.25) {$W$};
		\node [style=none] (55) at (-1.5, 1) {};
		\node [style=none] (56) at (8, 4.5) {};
		\node [style=none] (57) at (-0.5, -0.25) {$W$};
		\node [style=none] (58) at (-0.5, 0.25) {};
		\node [style=none] (59) at (-0.5, 3.75) {};
		\node [style=none] (60) at (2.5, -0.75) {};
		\node [style=none] (61) at (-0.5, 1) {};
		\node [style=none] (62) at (-0.5, -0.75) {};
		\node [style=none] (63) at (-0.5, 3.75) {};
		\node [style=none] (64) at (0.5, 2.5) {};
		\node [style=none] (65) at (1, -3.25) {$W$};
		\node [draw, thick, style=none, minimum size=0.2 cm, circle, fill=white] (66) at (3.75, 2.5) {};
		\node [style=none] (67) at (2.5, 1) {};
		\node [style=none] (68) at (5, 1) {};
		\node [style=none] (69) at (8, -0.5) {};
		\node [style=none] (70) at (5, -0.5) {};
		\node [style=none] (71) at (3.75, 3.75) {};
		\node [style=none] (72) at (3.75, 2.5) {};
		\node [style=none] (73) at (3.75, 3.75) {};
	\end{pgfonlayer}
	\begin{pgfonlayer}{edgelayer}
		\draw  [style=thick](16.center) to (20.center);
		\draw  [style=thick](19.center) to (16.center);
		\draw  [style=thick](19.center) to (20.center);
		\draw  [style=thick](12.center) to (3.center);
		\draw  [style=thick](10.center) to (12.center);
		\draw  [style=thick](10.center) to (3.center);
		\draw  [style=thick](13.center) to (8.center);
		\draw  [style=thick](5.center) to (13.center);
		\draw  [style=thick](5.center) to (8.center);
		\draw  [style=thick](18.center) to (14.center);
		\draw  [style=thick](24.center) to (25.center);
		\draw  [style=thick](23.center) to (17.center);
		\draw  [style=thick](6.center) to (22.center);
		\draw  [style=thick](1.center) to (7.center);
		\draw [style=thick, bend right=90, looseness=1.25] (27.center) to (28.center);
		\draw (29.center) to (30.center);
		\draw [style=thick, bend right=90, looseness=1.25] (32.center) to (33.center);
		\draw [style=thick ](34.center) to (36.center);
		\draw [style=thick, bend right=90] (37.center) to (38.center);
		\draw [style=thick](39.center) to (40.center);
		\draw [style=thick](46.center) to (64.center);
		\draw [style=thick](64.center) to (53.center);
		\draw [style=thick](53.center) to (55.center);
		\draw [style=thick](55.center) to (46.center);
		\draw [style=thick](49.center) to (42.center);
		\draw  [style=thick](61.center) to (47.center);
		\draw  [style=thick](63.center) to (52.center);
		\draw [style=thick, in=270, out=-90, looseness=1.25] (62.center) to (60.center);
		\draw  [style=thick](56.center) to (48.center);
		\draw [style=thick, bend left=90] (51.center) to (56.center);
		\draw  [style=thick](51.center) to (45.center);
		\draw [style=thick, in=90, out=90, looseness=2.00] (67.center) to (68.center);
		\draw [style=thick, in=270, out=-90, looseness=1.25] (70.center) to (69.center);
		\draw  [style=thick] (73.center) to (72.center);
		\draw [style=thick, bend left=90, looseness=1.50] (63.center) to (71.center);
		\draw [style=thick] (67.center) to (50.center);
		\draw  [style=thick](68.center) to (70.center);
		\draw  [style=thick] (48.center) to (69.center);
	\end{pgfonlayer}
\end{tikzpicture}}
\endpgfgraphicnamed}
\end{minipage}
\qquad $=$  \qquad
\begin{minipage}{6cm}
 {%
\beginpgfgraphicnamed{Q-Obj-Norm}
\begin{tikzpicture}[scale=0.7]
	\begin{pgfonlayer}{nodelayer}
		\node [style=none] (0) at (2, 4) {};
		\node [style=none] (1) at (-7.75, 2) {};
		\node [style=none] (2) at (2, 2.5) {};
		\node [style=none] (3) at (2, 2) {};
		\node [style=none] (4) at (-7.75, 2.5) {};
		\node [style=none] (5) at (-7.75, 4.25) {};
		\node [style=none] (6) at (-6.75, 2.5) {};
		\node [style=none] (7) at (2, 1.5) {$W$};
		\node [style=none] (8) at (-7.75, 1.5) {$W$};
		\node [style=none] (9) at (2, 4.75) {$\overline{\semantics{n}}$};
		\node [style=none] (10) at (-7.75, 4.75) {$\overline{\semantics{np}}$};
		\node [style=none] (11) at (3, 2.5) {};
		\node [style=none] (12) at (1, 2.5) {};
		\node [style=none] (13) at (-8.75, 2.5) {};
		\node [style=none] (14) at (-7.75, 1) {};
		\node [style=none] (15) at (-4.5, 1) {};
		\node [draw, style=thick, minimum size=0.2 cm, circle, fill=white] (16) at (2, 0.25) {};
		\node [style=none] (17) at (3.5, -4.5) {$W$};
		\node [style=none] (18) at (0.75, -1.5) {$W$};
		\node [style=none] (19) at (-0.75, -5.25) {};
		\node [style=none] (20) at (2, 1) {};
		\node [style=none] (21) at (3.5, -4) {};
		\node [style=none] (22) at (-2, -2) {};
		\node [style=none] (23) at (3.5, -2.5) {$\overline{\semantics{d}}$};
		\node [style=none] (24) at (-2, -1) {};
		\node [style=none] (25) at (4.5, -3.25) {};
		\node [style=none] (26) at (-2, -1.5) {$W$};
		\node [style=none] (27) at (3.5, -1) {};
		\node [style=none] (28) at (3.5, -5) {};
		\node [style=none] (29) at (2.5, -1.75) {};
		\node [style=none] (30) at (-2, 1) {};
		\node [style=none] (31) at (3.5, -3.25) {};
		\node [style=none] (32) at (2.5, -3.25) {};
		\node [style=none] (33) at (0.75, -1) {};
		\node [style=none] (34) at (-0.75, -4.75) {$W$};
		\node [style=none] (35) at (0.75, -2) {};
		\node [style=none] (36) at (-0.75, -4.25) {};
		\node [style=none] (37) at (4.5, -1.75) {};
		\node [style=none] (38) at (3.5, -1.75) {};
		\node [style=none] (39) at (-1, 2.5) {};
		\node [style=none] (40) at (-3.25, 2) {};
		\node [style=none] (41) at (-3.25, 1.5) {$S$};
		\node [style=none] (42) at (-5.5, 2.5) {};
		\node [style=none] (43) at (-2, 2) {};
		\node [style=none] (44) at (-3.25, 4.25) {};
		\node [style=none] (45) at (-2, 2.5) {};
		\node [style=none] (46) at (-4.5, 2.5) {};
		\node [style=none] (47) at (-4.5, 2) {};
		\node [style=none] (48) at (-2, 1.5) {$W$};
		\node [style=none] (49) at (-3.25, 2.5) {};
		\node [style=none] (50) at (-4.5, 1.5) {$W$};
		\node [draw, style=thick, minimum size=0.1 cm, circle, fill=black]  (51) at (2, 0.25) {};
		\node [style=none] (52) at (-0.75, -3) {};
		\node [draw, none, style=thick, minimum size=0.2 cm, circle, fill=white] (53) at (-0.75, -3) {};
		\node [style=none] (54) at (-3.25, 5) {$\overline{\semantics{v}}$};
	\end{pgfonlayer}
	\begin{pgfonlayer}{edgelayer}
		\draw  [style=thick]  (13.center) to (6.center);
		\draw  [style=thick]  (5.center) to (13.center);
		\draw   [style=thick] (5.center) to (6.center);
		\draw  [style=thick] (12.center) to (11.center);
		\draw  [style=thick]  (0.center) to (12.center);
		\draw  [style=thick]  (0.center) to (11.center);
		\draw  [style=thick]  (4.center) to (1.center);
		\draw  [style=thick]  (2.center) to (3.center);
		\draw [style=thick, bend right=90, looseness=1.25] (14.center) to (15.center);
		\draw  [style=thick]  (29.center) to (37.center);
		\draw  [style=thick]  (37.center) to (25.center);
		\draw  [style=thick]  (25.center) to (32.center);
		\draw  [style=thick] (32.center) to (29.center);
		\draw [style= thick, in=270, out=90] (16.center) to (20.center);
		\draw   [style=thick] (31.center) to (21.center);
		\draw [style=thick, in=270, out=-90, looseness=1.25] (22.center) to (35.center);
		\draw [style=thick, bend left=90, looseness=1.50] (33.center) to (27.center);
		\draw  [style=thick] (30.center) to (24.center);
		\draw [style = thick, bend right=90, looseness=1.25] (19.center) to (28.center);
		\draw  [style=thick] (27.center) to (38.center);
		\draw  [style=thick] (42.center) to (39.center);
		\draw  [style=thick] (44.center) to (42.center);
		\draw  [style=thick] (44.center) to (39.center);
		\draw  [style=thick] (46.center) to (47.center);
		\draw  [style=thick] (49.center) to (40.center);
		\draw  [style=thick] (45.center) to (43.center);
		\draw  [style=thick]  (52.center) to (36.center);
	\end{pgfonlayer}
\end{tikzpicture}}
\endpgfgraphicnamed}
\end{minipage}
\end{minipage}

Putting the two cases together, the  interpretation  of a sentence with  quantified phrases both at subject and at an object position is as follows:

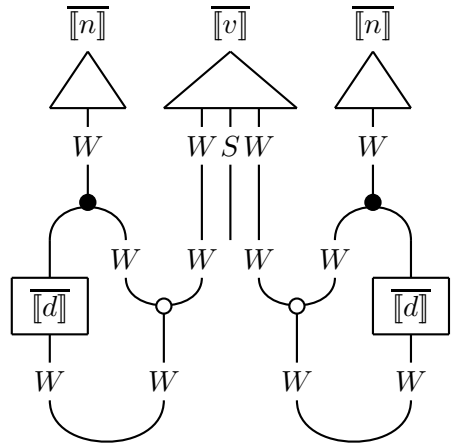
\begin{center}
\begin{minipage}{7cm}
{%
\beginpgfgraphicnamed{Q-Frob-Sbj-Obj-Norm}
\begin{tikzpicture}
	\begin{pgfonlayer}{nodelayer}
		\node [style=none] (0) at (1.5, 3.75) {};
		\node [style=none] (1) at (-2.25, 2.75) {$S$};
		\node [style=none] (2) at (-1.5, 2.75) {$W$};
		\node [style=none] (3) at (-7, -1.5) {$\overline{\semantics{d}}$};
		\node [style=none] (4) at (-7, -4) {};
		\node [style=none] (5) at (-6, 3.75) {};
		\node [style=none] (6) at (-7, 3.75) {};
		\node [style=none] (7) at (-3, -0.25) {$W$};
		\node [style=none] (8) at (-2.25, 3.25) {};
		\node [style=none] (9) at (-4, -3) {};
		\node [style=none] (10) at (-8, -0.75) {};
		\node [style=none] (11) at (-6, 3.25) {};
		\node [style=none] (12) at (-2.25, 2.25) {};
		\node [style=none] (13) at (-7, -3.5) {$W$};
		\node [style=none] (14) at (-7, -0.75) {};
		\node [style=none] (15) at (-8, -2.25) {};
		\node [style=none] (16) at (-6, 6) {$\overline{\semantics{n}}$};
		\node [style=none] (17) at (-3, 2.25) {};
		\node [style=none] (18) at (-2.25, 5.25) {};
		\node [style=none] (19) at (0.5, 3.75) {};
		\node [style=none] (20) at (-5, -0.25) {$W$};
		\node [style=none] (21) at (-7, -2.25) {};
		\node [style=none] (22) at (-6, 2.25) {};
		\node [style=none] (23) at (1.5, 5.25) {};
		\node [style=none] (24) at (-6, -2.25) {};
		\node [style=none] (25) at (1.5, 2.75) {$W$};
		\node [style=none] (26) at (-5, -0.75) {};
		\node [style=none] (27) at (-6, -0.75) {};
		\node [style=none] (28) at (-2.25, 3.75) {};
		\node [style=none] (29) at (-7, -3) {};
		\node [style=none] (30) at (-4, 3.75) {};
		\node [style=none] (31) at (-1.5, 3.75) {};
		\node [draw, thick, style=none, minimum size=0.2 cm, circle, fill=white] (32) at (-4, -1.5) {};
		\node [style=none] (33) at (-3, -0.75) {};
		\node [draw, thick, style=none, minimum size=0.2 cm, circle, fill=black] (34) at (-6, 1.25) {};
		\node [style=none] (35) at (-3, 3.75) {};
		\node [style=none] (36) at (-7, 0.25) {};
		\node [style=none] (37) at (-6, 5.25) {};
		\node [style=none] (38) at (-0.5, 3.75) {};
		\node [style=none] (39) at (-4, -3.5) {$W$};
		\node [style=none] (40) at (-4, -4) {};
		\node [style=none] (41) at (-3, 2.75) {$W$};
		\node [style=none] (42) at (-2.25, 0.25) {};
		\node [style=none] (43) at (-3, 0.25) {};
		\node [style=none] (44) at (-3, 3.25) {};
		\node [style=none] (45) at (-5, 0.25) {};
		\node [style=none] (46) at (-6, 2.75) {$W$};
		\node [style=none] (47) at (1.5, 3.25) {};
		\node [style=none] (48) at (1.5, 6) {$\overline{\semantics{n}}$};
		\node [style=none] (49) at (-5, 3.75) {};
		\node [style=none] (50) at (2.5, 3.75) {};
		\node [style=none] (51) at (-2.25, 6) {$\overline{\semantics{v}}$};
		\node [style=none] (52) at (-1.5, 3.25) {};
		\node [style=none] (53) at (2.5, -1.5) {$\overline{\semantics{d}}$};
		\node [style=none] (54) at (2.5, -0.75) {};
		\node [style=none] (55) at (-0.5, -4) {};
		\node [style=none] (56) at (-1.5, -0.75) {};
		\node [style=none] (57) at (2.5, -3.5) {$W$};
		\node [style=none] (58) at (3.5, -2.25) {};
		\node [style=none] (59) at (2.5, -2.25) {};
		\node [style=none] (60) at (0.5, -0.75) {};
		\node [draw, thick, style=none, minimum size=0.2 cm, circle, fill=white] (61) at (-0.5, -1.5) {};
		\node [style=none] (62) at (3.5, -0.75) {};
		\node [style=none] (63) at (1.5, -0.75) {};
		\node [style=none] (64) at (-1.5, -0.25) {$W$};
		\node [style=none] (65) at (-0.5, -3.5) {$W$};
		\node [style=none] (66) at (0.5, 0.25) {};
		\node [style=none] (67) at (0.5, -0.25) {$W$};
		\node [style=none] (68) at (2.5, -3) {};
		\node [style=none] (69) at (2.5, -4) {};
		\node [style=none] (70) at (1.5, -2.25) {};
		\node [style=none] (71) at (2.5, 0.25) {};
		\node [draw, thick, style=none, minimum size=0.2 cm, circle, fill=black] (72) at (1.5, 1.25) {};
		\node [style=none] (73) at (-0.5, -3) {};
		\node [style=none] (74) at (-1.5, 0.25) {};
		\node [style=none] (75) at (1.5, 2.25) {};
		\node [style=none] (76) at (-1.5, 2.25) {};
	\end{pgfonlayer}
	\begin{pgfonlayer}{edgelayer}
		\draw  [style=thick] (6.center) to (49.center);
		\draw [style=thick](37.center) to (6.center);
		\draw [style=thick] (37.center) to (49.center);
		\draw [style=thick](30.center) to (38.center);
		\draw [style=thick](18.center) to (30.center);
		\draw [style=thick](18.center) to (38.center);
		\draw [style=thick](19.center) to (50.center);
		\draw [style=thick](23.center) to (19.center);
		\draw [style=thick](23.center) to (50.center);
		\draw [style=thick](5.center) to (11.center);
		\draw [style=thick](35.center) to (44.center);
		\draw [style=thick](28.center) to (8.center);
		\draw [style=thick] (31.center) to (52.center);
		\draw [style=thick](0.center) to (47.center);
		\draw [style=thick](10.center) to (27.center);
		\draw [style=thick](27.center) to (24.center);
		\draw [style=thick] (24.center) to (15.center);
		\draw [style=thick](15.center) to (10.center);
		\draw [style=thick](34.center) to (22.center);
		\draw [style=thick](12.center) to (42.center);
		\draw [style=thick] (21.center) to (29.center);
		\draw [style = thick, bend right=90, looseness=1.25] (26.center) to (33.center);
		\draw [style=thick, bend left=90, looseness=1.50] (36.center) to (45.center);
		\draw [style=thick](17.center) to (43.center);
		\draw [style=thick, bend right=90, looseness=1.25] (4.center) to (40.center);
		\draw [style=thick](32.center) to (9.center);
		\draw [style=thick](36.center) to (14.center);
		\draw [style=thick](63.center) to (62.center);
		\draw [style=thick](62.center) to (58.center);
		\draw [style=thick](58.center) to (70.center);
		\draw [style=thick](70.center) to (63.center);
		\draw [style=thick](59.center) to (68.center);
		\draw [style= thick, bend right=90, looseness=1.25] (56.center) to (60.center);
		\draw [style = thick, bend left=90, looseness=1.50] (66.center) to (71.center);
		\draw [style = thick, bend right=90, looseness=1.25] (55.center) to (69.center);
		\draw [style=thick](61.center) to (73.center);
		\draw [style=thick](71.center) to (54.center);
		\draw [style = thick, style=none, in=90, out=270] (75.center) to (72.center);
		\draw [style=thick] (76.center) to (74.center);
	\end{pgfonlayer}
\end{tikzpicture}}
\endpgfgraphicnamed}
\end{minipage}
\end{center}

%




\section{Truth Theoretic Interpretation in $\Rel$}

A model $(U, \semantics{\ })$ of the language of generalised quantifier theory is made categorical via the  instantiation to $\Rel$ of the abstract compact closed categorical model.  


\begin{definition}\label{def:concrete-REL}
The   instantiation of the abstract  model of definition \ref{ccc-model} to $\Rel$  is a tuple   $(\Rel, \cal P (U), \{\star\}, \overline{\semantics{\ }})$, for ${\cal U}$ the universe of reference. The interpretations of words in this model are defined by the following relations:
\begin{itemize}
\item The interpretation  of a  terminal $x$  generated by any of the non-terminals N,NP, and VP  is 
\[ \star \overline{\semantics{x}} A \iff A = \semantics{x} \]

\item The interpretation of a  terminal  $x$  generated by the non-terminal V  
is  
\[ \star \overline{\semantics{x}} (A, \star, B) \iff \semantics x (A) = B \]
where $\semantics x (A)$ is the forward image of $A$ in the binary relation $\semantics x$.

\item The interpretation of a terminal $d$ generated by the non-terminal  Det is   
\[ A \overline{\semantics{d}} B \iff B \in \semantics{d} (A) \]
\end{itemize}
\end{definition}

For the types, note that the interpretation  of a  terminal $x$  generated by any of the non-terminals N,NP, and VP has type   $\overline{\semantics{x}} : \{\star\} \relto \cal P (U)$.  The interpretation of a VP is the initial morphism to  ${\cal P}(U) \otimes \{\star\}$, which is isomorphic to ${\cal P}(U)$, hence it gets the same concrete instantiation as N and NP. 
The interpretation of a  terminal  $x$  generated by the non-terminal V has type 
$\overline{\semantics{x}} : \{\star\} \relto \cal P (U) \otimes \{\star\} \otimes \cal P (U) \cong \cal P (U) \otimes \cal P (U)$.  Finally, the interpretation of a terminal $d$ generated by the non-terminal  Det has type 
$\overline{\semantics{d}} : {\cal P (U)} \relto {\cal P (U)}$. 

%
%
%

Informally, the bialgebra map $ \mu$ is the analogue of  set-theoretic intersection and the compact closed  epsilon map is the analogue of  set-theoretic application.  It is not hard to show that  the truth-theoretic interpretation of the compact closed semantics of quantified sentences provides us with the same meaning as the generalised quantifier semantics. We make this formal as follows. 

\begin{definition}
\label{def:truth-rel}
The interpretation of a quantified sentence $s$  is true in $(\Rel, \cal P (U), \{\star\}, \overline{\semantics{\ }})$ iff   $\star \overline{\semantics{\mbox{s}}} \star$. 
\end{definition}

\begin{theorem}
\label{thm:truth-rel}
$\star \overline{\semantics{\mbox{s}}} \star$ in $(\Rel, \cal P (U), \{\star\}, \overline{\semantics{\ }})$ iff  $\semantics{S}$ is true in generalised quantifier theory, as defined in Definition \ref{def:truth-genquant}. 
\end{theorem}
\begin{proof}

If a sentence is quantified, it is either of the form `Det N VP' or of the form `NP V Det N'.   For either case, since $\{\star\}$ is the unit of tensor in $\Rel$,  the $S$ objects and morphisms can be dropped from the meaning morphism. 
\begin{itemize}
\item For the first case,   we have to calculate the $\overline{\semantics{\mbox{s}}}$ relation:
\[ \epsilon_{\cal P (U)} \circ (\overline{\semantics{d}} \otimes \mu_{\cal P (U)} ) \circ (\delta_{\cal P (U)}  \otimes \operatorname{id}_{\cal P (U)}) \circ (\overline{\semantics{n}} \otimes \overline{\semantics{vp}}): \{\star\} \relto \{\star\} 
\]
We will calculate this relation in stages. First:
\begin{align*}
\star (\overline{\semantics{n}} \otimes \overline{\semantics{vp}}) (A, B) &\iff \star \overline{\semantics{n}} A \text{ and } \star \overline{\semantics{vp}} B \\
&\iff A = \semantics{n} \text{ and } B = \semantics{vp}
\end{align*}
since $(\star, \star) \cong \star$. Second:
\begin{align*}
\star ((\delta_{\cal P (U)}  \otimes \operatorname{id}_{\cal P (U)}) \circ (\overline{\semantics{n}} \otimes \overline{\semantics{vp}})) (A, B, C) &\iff \star (\overline{\semantics{n}} \otimes \overline{\semantics{vp}}) (A, C) \text{ and } A = B \\
&\iff A = B = \semantics{n} \text{ and } C = \semantics{vp}
\end{align*}
Third:
\begin{align*}
&\star ((\overline{\semantics{d}} \otimes \mu_{\cal P (U)} ) \circ (\delta_{\cal P (U)}  \otimes \operatorname{id}_{\cal P (U)}) \circ (\overline{\semantics{n}} \otimes \overline{\semantics{vp}})) (A, B) \\
\iff\ &A' \overline{\semantics{d}} A \text{ and } B = B' \cap C' \text{ for some } \star ((\delta_{\cal P (U)}  \otimes \operatorname{id}_{\cal P (U)}) \circ (\overline{\semantics{n}} \otimes \overline{\semantics{vp}})) (A', B', C') \\
\iff\ &A \in \semantics{d} (\semantics{n}) \text{ and } B = \semantics{n} \cap \semantics{vp}
\end{align*}
Finally:
\begin{align*}
&\star (\epsilon_{\cal P (U)} \circ (\overline{\semantics{d}} \otimes \mu_{\cal P (U)} ) \circ (\delta_{\cal P (U)}  \otimes \operatorname{id}_{\cal P (U)}) \circ (\overline{\semantics{n}} \otimes \overline{\semantics{vp}})) \star \\
\iff\ &\star ((\overline{\semantics{d}} \otimes \mu_{\cal P (U)} ) \circ (\delta_{\cal P (U)}  \otimes \operatorname{id}_{\cal P (U)}) \circ (\overline{\semantics{n}} \otimes \overline{\semantics{vp}})) (A, A) \text{ for some } A \\
\iff\ &\semantics{n} \cap \semantics{vp} \in \semantics{d} (\semantics{n})
\end{align*}
This is the same as the set theoretic  meaning of the sentence in generalised quantifier theory.

\item For the second case, we have:
\[
 \overline{\semantics{\mbox{s}}}   \quad = \quad \epsilon_{\mathcal P (U)} \circ (\mu_{\cal P (U)}  \otimes \overline{\semantics d}) \circ (\epsilon_{\mathcal P (U)} \otimes \operatorname{id}_{\mathcal P (U)} \otimes \delta_{\cal P (U)} ) \circ (\overline{\semantics{np}} \otimes \overline{\semantics v} \otimes \overline{\semantics n}) \]
Again we calculate in stages. First:
\begin{align*}
	\star (\overline{\semantics{np}} \otimes \overline{\semantics v} \otimes \overline{\semantics n}) (A, B, C, D) &\iff \star \overline{\semantics{np}} A \text{ and } \star \overline{\semantics v} (B, C) \text{ and } \star \overline{\semantics n} D \\
	&\iff A = \semantics{np} \text{ and } C = \semantics v (B) \text{ and } D = \semantics n
\end{align*}
Second:
\begin{align*}
	&\star ((\epsilon_{\mathcal P (U)} \otimes \operatorname{id}_{\mathcal P (U)} \otimes \delta_{\cal P (U)} ) \circ (\overline{\semantics{np}} \otimes \overline{\semantics v} \otimes \overline{\semantics n})) (C, D, E) \\
	\iff\ &D = E \text{, and } \star (\overline{\semantics{np}} \otimes \overline{\semantics v} \otimes \overline{\semantics n}) (A, A, C, D) \text{ for some } A \\
	\iff\ &C = \semantics v (\semantics{np}) \text{ and } D = E = \semantics n
\end{align*}
Third:
\begin{align*}
	&\star ((\mu_{\cal P (U)}  \otimes \overline{\semantics d}) \circ (\epsilon_{\mathcal P (U)} \otimes \operatorname{id}_{\mathcal P (U)} \otimes \delta_{\cal P (U)} ) \circ (\overline{\semantics{np}} \otimes \overline{\semantics v} \otimes \overline{\semantics n})) (F, G) \\
	\iff\ &F = C \cap D \text{ and } D \overline{\semantics d} G \text{ for some } \star ((\epsilon_{\mathcal P (U)} \otimes \operatorname{id}_{\mathcal P (U)} \otimes \delta_{\cal P (U)} ) \circ (\overline{\semantics{np}} \otimes \overline{\semantics v} \otimes \overline{\semantics n})) (C, D, E) \\
	\iff\ &F = \semantics v (\semantics{np}) \cap \semantics n \text{ and } G \in \semantics d (\semantics n)
\end{align*}
Fourth:
\begin{align*}
	&\star (\epsilon_{\mathcal P (U)} \circ (\mu_{\cal P (U)}  \otimes \overline{\semantics d}) \circ (\epsilon_{\mathcal P (U)} \otimes \operatorname{id}_{\mathcal P (U)} \otimes \delta_{\cal P (U)} ) \circ (\overline{\semantics{np}} \otimes \overline{\semantics v} \otimes \overline{\semantics n})) \star \\
	\iff\ &\star ((\mu_{\cal P (U)}  \otimes \overline{\semantics d}) \circ (\epsilon_{\mathcal P (U)} \otimes \operatorname{id}_{\mathcal P (U)} \otimes \delta_{\cal P (U)} ) \circ (\overline{\semantics{np}} \otimes \overline{\semantics v} \otimes \overline{\semantics n})) \text{ for some } F \\
	\iff\ &\semantics v (\semantics{np}) \cap \semantics n \in \semantics d (\semantics n)
\end{align*}
Again, this is exactly the truth theoretic definition of the meaning of the sentence in generalised quantifier theory. This completes the proof.
\end{itemize}
\end{proof}

In previous work \cite{RelPronMoL,SadrClarkCoecke1,SadrClarkCoecke2}, we modelled relative pronouns in compact closed categories with Frobenius algebras.  Using those results and Theorem  \ref{thm:truth-rel}, we show that the living on equivalences of Lemma \ref{lemma:livingon}  hold in $\Rel$. 

\begin{corollary}
\label{cor:livingon}
If $d$ satisfies the categorical living on property, then the following equivalences hold in $(\Rel, \cal P (U), \{\star\}, \overline{\semantics{\ }})$:
\begin{eqnarray*}
\star \semantics{\overline{d \ n \ vp}} \star  &\iff& \star \semantics{\overline{d \ n \ \text{\bf  are} \ n \ \text{\bf who} \ vp}} \star \\
\star \semantics{\overline{np \ v \ d \ n}} \star \ &\iff& \  \star \semantics{\overline{np  \ v\  d\ n \ \text{\bf who \ are}  \ n}} \star
\end{eqnarray*}
where $\overline{\semantics{\text{who}}} = (1_{{\cal P}(U)} \otimes \mu_{{\cal P}(U)} \otimes 1_{{\cal P}(U)}) \circ (\eta_{{\cal P}(U)}  \otimes \eta_{{\cal P}(U)})$ and $\overline{\semantics{\text{are}}} = \eta_{{\cal P}(U)}$.\end{corollary}

\begin{proof}
For the first case, consider the diagram corresponding to the relation $ \star \semantics{\overline{d \ n \ \text{\bf  are} \ n \ \text{\bf who} \ vp}} \star $:

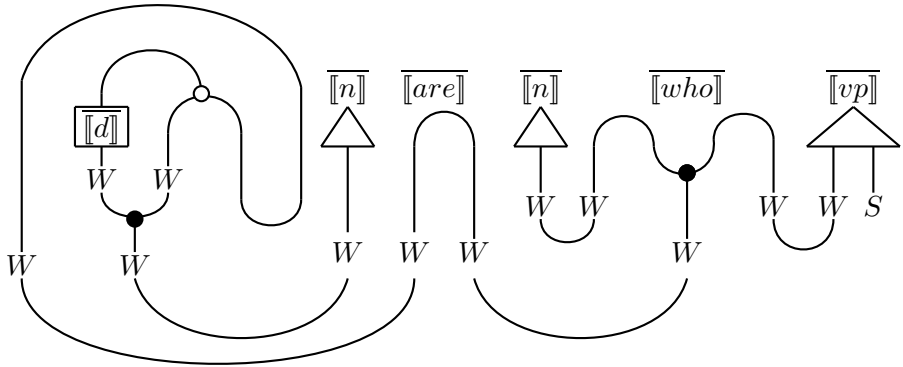
\begin{center}
{%
\beginpgfgraphicnamed{conserv-sbj}
\begin{tikzpicture}[scale=0.7]
	\begin{pgfonlayer}{nodelayer}
		\node [style=none] (0) at (-4.75, 2.5) {};
		\node [style=none] (1) at (-2.75, 2.5) {};
		\node [style=none] (2) at (-3.75, 4) {};
		\node [style=none] (3) at (-3.75, 4.75) {$\overline{\semantics{n}}$};
		\node [style=none] (4) at (-0.5, 4.75) {$\overline{\semantics{are}}$};
		\node [style=none] (5) at (-3.75, 2.5) {};
		\node [style=none] (6) at (-1.25, 2.5) {};
		\node [style=none] (7) at (1, 2.5) {};
		\node [style=none] (8) at (-3.75, -0.75) {};
		\node [style=none] (9) at (-3.75, -1.5) {$W$};
		\node [style=none] (10) at (-1.25, -0.75) {};
		\node [style=none] (11) at (-1.25, -1.5) {$W$};
		\node [style=none] (12) at (1, -1) {};
		\node [style=none] (13) at (1, -1.5) {$W$};
		\node [draw, thick, style=none, minimum size=0.2 cm, circle, fill=black] (14) at (-11.75, -0.25) {};
		\node [style=none] (15) at (-13, 1.25) {$W$};
		\node [style=none] (16) at (-11.75, -2.5) {};
		\node [style=none] (17) at (-9.25, 4.5) {};
		\node [style=none] (18) at (-10.5, 3) {};
		\node [style=none] (19) at (-14, 2.5) {};
		\node [style=none] (20) at (-11.75, -0.25) {};
		\node [style=none] (21) at (-14, 4) {};
		\node [style=none] (22) at (-13, 4.5) {};
		\node [style=none] (23) at (-3.75, -2.5) {};
		\node [style=none] (24) at (-13, 3.25) {$\overline{\semantics{d}}$};
		\node [style=none] (25) at (-16, 5) {};
		\node [style=none] (26) at (-13, 4) {};
		\node [style=none] (27) at (-11.75, -2) {$W$};
		\node [style=none] (28) at (-10.5, 1.75) {};
		\node [style=none] (29) at (-12, 2.5) {};
		\node [style=none] (30) at (-10.5, 1.25) {$W$};
		\node [style=none] (31) at (-13, 2.5) {};
		\node [style=none] (32) at (-5.5, 4.75) {};
		\node [style=none] (33) at (-13, 0.75) {};
		\node [style=none] (34) at (-10.5, 3) {};
		\node [draw, thick, style=none, minimum size=0.2 cm, circle, fill=white] (35) at (-9.25, 4.5) {};
		\node [style=none] (36) at (-11.75, -1.5) {};
		\node [style=none] (37) at (-12, 4) {};
		\node [style=none] (38) at (-7.75, 0.5) {};
		\node [style=none] (39) at (-5.5, 0.5) {};
		\node [style=none] (40) at (-16, -1.5) {};
		\node [style=none] (41) at (-13, 1.75) {};
		\node [style=none] (42) at (-7.75, 3) {};
		\node [style=none] (43) at (-10.5, 0.75) {};
		\node [style=none] (44) at (-16, -2) {$W$};
		\node [style=none] (45) at (-7.75, 0.5) {};
		\node [style=none] (46) at (-16, -2.5) {};
		\node [style=none] (47) at (-1.25, -2.5) {};
		\node [style=none] (48) at (15.25, 4) {};
		\node [style=none] (49) at (14.5, 2.5) {};
		\node [style=none] (50) at (13.5, 2.5) {};
		\node [style=none] (51) at (16, 2.5) {};
		\node [style=none] (52) at (17, 2.5) {};
		\node [style=none] (53) at (4.5, 2.5) {};
		\node [style=none] (54) at (2.5, 2.5) {};
		\node [style=none] (55) at (3.5, 2.5) {};
		\node [style=none] (56) at (3.5, 4) {};
		\node [style=none] (57) at (3.5, 4.75) {$\overline{\semantics{n}}$};
		\node [style=none] (58) at (15.25, 4.75) {$\overline{\semantics{vp}}$};
		\node [style=none] (59) at (9, 1.5) {};
		\node [fill=white, draw, thick, circle, minimum size=0.2 cm, style=none] (60) at (9, 1.5) {};
		\node [style=none] (61) at (10, 2.75) {};
		\node [style=none] (62) at (9, -1.5) {$W$};
		\node [style=none] (63) at (12.25, 0.75) {};
		\node [style=none] (64) at (12.25, 0.25) {$W$};
		\node [style=none] (65) at (12.25, 2.75) {};
		\node [style=none] (66) at (9, -1) {};
		\node [style=none] (67) at (5.5, 2.5) {};
		\node [style=none] (68) at (5.5, 0.25) {$W$};
		\node [style=none] (69) at (5.5, 0.75) {};
		\node [style=none] (70) at (7.75, 2.5) {};
		\node [style=none] (71) at (9, 4.75) {$\overline{\semantics{who}}$};
		\node [style=none] (72) at (14.5, 0.75) {};
		\node [style=none] (73) at (16, 0.75) {};
		\node [style=none] (74) at (14.5, 0.25) {$W$};
		\node [style=none] (75) at (16, 0.25) {$S$};
		\node [style=none] (76) at (12.25, -0.25) {};
		\node [style=none] (77) at (12.25, -0.25) {};
		\node [style=none] (78) at (14.5, -0.25) {};
		\node [draw, thick, style=none, minimum size=0.2 cm, circle, fill=black] (79) at (9, 1.5) {};
		\node [style=none] (80) at (3.5, 0.75) {};
		\node [style=none] (81) at (3.5, 0.25) {$W$};
		\node [style=none] (82) at (3.5, -0.25) {};
		\node [style=none] (83) at (5.5, -0.25) {};
		\node [style=none] (84) at (1, -2.5) {};
		\node [style=none] (85) at (9, -2.5) {};
	\end{pgfonlayer}
	\begin{pgfonlayer}{edgelayer}
		\draw [style=thick](0.center) to (1.center);
		\draw  [style=thick](2.center) to (0.center);
		\draw  [style=thick](2.center) to (1.center);
		\draw  [style=thick] (5.center) to (8.center);
		\draw  [style=thick](6.center) to (10.center);
		\draw [style=thick] (7.center) to (12.center);
		\draw [style=thick, in=90, out=90, looseness=1.50] (22.center) to (17.center);
		\draw [style=thick] (21.center) to (37.center);
		\draw [style=thick] (37.center) to (29.center);
		\draw [style=thick] (29.center) to (19.center);
		\draw [style=thick] (19.center) to (21.center);
		\draw [style=thick] (20.center) to (36.center);
		\draw [style=thick] (31.center) to (41.center);
		\draw [style=thick, bend right=90, looseness=1.25] (33.center) to (43.center);
		\draw [style=thick, bend right=75] (16.center) to (23.center);
		\draw [style=thick] (34.center) to (28.center);
		\draw [style=thick] (25.center) to (40.center);
		\draw [style=thick] (22.center) to (26.center);
		\draw [style=thick, bend right=90, looseness=1.50] (45.center) to (39.center);
		\draw [style=thick] (32.center) to (39.center);
		\draw [style=thick, bend left=75] (25.center) to (32.center);
		\draw [style=thick] (42.center) to (38.center);
		\draw [style=thick, bend left=90, looseness=1.75] (34.center) to (42.center);
		\draw [style=thick, bend right=90, looseness=0.75] (46.center) to (47.center);
		\draw [style=thick](50.center) to (52.center);
		\draw [style=thick](48.center) to (50.center);
		\draw [style=thick](48.center) to (52.center);
		\draw [style = thick, bend left=90, looseness=2.00] (6.center) to (7.center);
		\draw [style=thick](54.center) to (53.center);
		\draw [style=thick](56.center) to (54.center);
		\draw [style=thick](56.center) to (53.center);
		\draw [style=thick] (65.center) to (63.center);
		\draw [style=thick, bend left=90, looseness=1.75] (67.center) to (70.center);
		\draw [style=thick, bend left=90, looseness=1.50] (61.center) to (65.center);
		\draw [style=thick, bend right=90, looseness=1.75] (70.center) to (61.center);
		\draw [style=thick] (67.center) to (69.center);
		\draw [style=thick] (59.center) to (66.center);
		\draw [style=thick](49.center) to (72.center);
		\draw [style=thick](51.center) to (73.center);
		\draw [style=thick, bend right=90, looseness=1.75] (77.center) to (78.center);
		\draw [style=thick](55.center) to (80.center);
		\draw [style=thick, bend right=90, looseness=1.50] (82.center) to (83.center);
		\draw [style=thick, bend right=75] (84.center) to (85.center);
	\end{pgfonlayer}
\end{tikzpicture}}
\endpgfgraphicnamed}
\end{center}

\noindent
It simplifies to the following diagram:

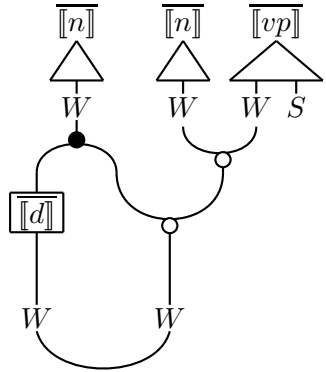
\begin{center}
\vspace{1cm}
\hspace{2cm}
{%
\beginpgfgraphicnamed{Q-Sbj-Living}
\begin{tikzpicture}[scale=0.7]
	\path [use as bounding box] (-1.5,0) rectangle (-1.5,0.75);
	\begin{pgfonlayer}{nodelayer}
		\node [style=none] (0) at (-0.25, 3.75) {};
		\node [style=none] (1) at (0.5, 5.25) {};
		\node [style=none] (2) at (-7, 3.25) {};
		\node [style=none] (3) at (1.25, 3.75) {};
		\node [style=none] (4) at (1.25, 2.75) {$S$};
		\node [style=none] (5) at (-0.25, 2.75) {$W$};
		\node [style=none] (6) at (-7, 3.75) {};
		\node [style=none] (7) at (-7, 5.25) {};
		\node [style=none] (8) at (-6, 3.75) {};
		\node [style=none] (9) at (1.25, 3.25) {};
		\node [style=none] (10) at (2.25, 3.75) {};
		\node [style=none] (11) at (-1.25, 3.75) {};
		\node [style=none] (12) at (-7, 6) {$\overline{\semantics{n}}$};
		\node [style=none] (13) at (0.5, 6) {$\overline{\semantics{vp}}$};
		\node [style=none] (14) at (-8, 3.75) {};
		\node [style=none] (15) at (-0.25, 3.25) {};
		\node [style=none] (16) at (-8.5, -5.25) {$W$};
		\node [style=none] (17) at (-8.5, -0.5) {};
		\node [style=none] (18) at (-3.5, -4.75) {};
		\node [draw, thick, style=none, minimum size=0.2 cm, circle, fill=white] (19) at (-3.5, -1.75) {};
		\node [style=none] (20) at (-9.5, -2) {};
		\node [style=none] (21) at (-8.5, -2) {};
		\node [style=none] (22) at (-7.5, -0.5) {};
		\node [style=none] (23) at (-7, 1.5) {};
		\node [style=none] (24) at (-3.5, -1.75) {};
		\node [style=none] (25) at (-1.5, 0.75) {};
		\node [style=none] (26) at (-8.5, -1.25) {$\overline{\semantics{d}}$};
		\node [style=none] (27) at (-3.5, -5.75) {};
		\node [style=none] (28) at (-7, 2.75) {$W$};
		\node [style=none] (29) at (-9.5, -0.5) {};
		\node [style=none] (30) at (-7, 2.25) {};
		\node [style=none] (31) at (-3.5, -1.75) {};
		\node [style=none] (32) at (-5.5, 0.25) {};
		\node [style=none] (33) at (-7.5, -2) {};
		\node [style=none] (34) at (-3.5, -5.25) {$W$};
		\node [style=none] (35) at (-8.5, 0.25) {};
		\node [style=none] (36) at (-8.5, -4.75) {};
		\node [draw, thick, style=none, minimum size=0.2 cm, circle, fill=black] (37) at (-7, 1.5) {};
		\node [style=none] (38) at (-8.5, -5.75) {};
		\node [style=none] (39) at (-3, 3.25) {};
		\node [style=none] (40) at (-2, 3.75) {};
		\node [style=none] (41) at (-3, 3.75) {};
		\node [style=none] (42) at (-4, 3.75) {};
		\node [style=none] (43) at (-3, 5.25) {};
		\node [style=none] (44) at (-3, 2.75) {$W$};
		\node [style=none] (45) at (-3, 6) {$\overline{\semantics{n}}$};
		\node [style=none] (46) at (-0.25, 2) {};
		\node [style=none] (47) at (-1.5, 0.75) {};
		\node [style=none] (48) at (-1.5, 0.75) {};
		\node [draw, thick, style=none, minimum size=0.2 cm, circle, fill=white] (49) at (-1.5, 0.75) {};
		\node [style=none] (50) at (-3, 2) {};
		\node [style=none] (51) at (-0.25, 2) {};
		\node [style=none] (52) at (-1.5, 0.25) {};
	\end{pgfonlayer}
	\begin{pgfonlayer}{edgelayer}
		\draw [style=thick]  (14.center) to (8.center);
		\draw  [style=thick]  (7.center) to (14.center);
		\draw [style=thick]  (7.center) to (8.center);
		\draw [style=thick]  (11.center) to (10.center);
		\draw [style=thick]  (1.center) to (11.center);
		\draw [style=thick]  (1.center) to (10.center);
		\draw [style=thick]  (6.center) to (2.center);
		\draw [style=thick]  (0.center) to (15.center);
		\draw [style=thick]  (3.center) to (9.center);
		\draw [style=thick] (29.center) to (22.center);
		\draw [style=thick] (22.center) to (33.center);
		\draw [style=thick] (33.center) to (20.center);
		\draw [style=thick] (20.center) to (29.center);
		\draw [style=thick] (35.center) to (17.center);
		\draw [style=thick] (21.center) to (36.center);
		\draw [style=thick, in=90, out=270] (30.center) to (23.center);
		\draw [style=thick, bend left=90, looseness=1.25] (35.center) to (32.center);
		\draw [style=thick, bend right=75] (38.center) to (27.center);
		\draw [style=thick] (19.center) to (18.center);
		\draw [style=thick] (42.center) to (40.center);
		\draw [style=thick] (43.center) to (42.center);
		\draw [style=thick] (43.center) to (40.center);
		\draw [style=thick]  (41.center) to (39.center);
		\draw [style=thick, bend right=90, looseness=1.25] (50.center) to (51.center);
		\draw [style=thick, bend right=90, looseness=1.50] (32.center) to (52.center);
		\draw [style=thick] (49.center) to (52.center);
	\end{pgfonlayer}
\end{tikzpicture}}
\endpgfgraphicnamed}
\end{center}

\vspace{2.3cm}
\noindent
The morphism corresponding to this diagram is as follows:
\[
(\epsilon_{{\cal P}(U)} \otimes 1_{\{\star\}}) \circ
(\overline{\semantics{d}} \otimes \mu_{{\cal P}(U)} \otimes 1_{\{\star\}}) \circ
(\delta_{{\cal P}(U)} \otimes \mu_{{\cal P}(U)} \otimes 1_{\{\star\}})
(\overline{\semantics{n}} \otimes  \overline{\semantics{n}} \otimes \overline{\semantics{vp}})\colon
\{\star\} \relto \{\star\} 
 \]
Following  almost identical  calculation steps as in the proof of Theorem  \ref{thm:truth-rel} for sentences with a quantified subject,  the above calculates to:
 \[
({\semantics{vp}} \cap {\semantics{n}}) \cap {\semantics{n}} \in \semantics{d}(\semantics{n})
 \]
evidently equivalent to 
  \[
{\semantics{vp}} \cap {\semantics{n}} \in \semantics{d}(\semantics{n})
 \]
 which as proved in Theorem  \ref{thm:truth-rel},  is the result of calculating the relation corresponding to $\star \semantics{\overline{d \ n \ vp}} \star $. 
 
 Diagrammatically, we have the following equality of the simplified forms of the diagrams of these sentences:
 
 \medskip
\hspace{3cm}\begin{minipage}{12cm}
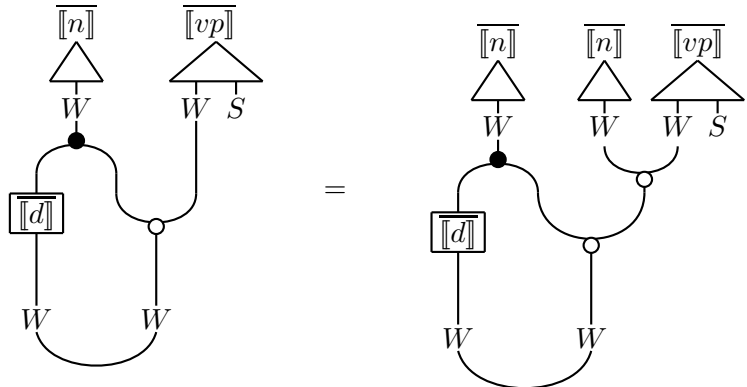
\begin{minipage}{4cm}
{%
\beginpgfgraphicnamed{Q-Sbj-Norm}
\begin{tikzpicture}[scale=0.7]
	\begin{pgfonlayer}{nodelayer}
		\node [style=none] (0) at (-2.5, 3.75) {};
		\node [style=none] (1) at (-1.75, 5.25) {};
		\node [style=none] (2) at (-7, 3.25) {};
		\node [style=none] (3) at (-1, 3.75) {};
		\node [style=none] (4) at (-1, 2.75) {$S$};
		\node [style=none] (5) at (-2.5, 2.75) {$W$};
		\node [style=none] (6) at (-7, 3.75) {};
		\node [style=none] (7) at (-7, 5.25) {};
		\node [style=none] (8) at (-6, 3.75) {};
		\node [style=none] (9) at (-1, 3.25) {};
		\node [style=none] (10) at (0, 3.75) {};
		\node [style=none] (11) at (-3.5, 3.75) {};
		\node [style=none] (12) at (-7, 6) {$\overline{\semantics{n}}$};
		\node [style=none] (13) at (-2, 6) {$\overline{\semantics{vp}}$};
		\node [style=none] (14) at (-8, 3.75) {};
		\node [style=none] (15) at (-2.5, 3.25) {};
		\node [style=none] (16) at (-8.5, -5.25) {$W$};
		\node [style=none] (17) at (-5.5, -0.5) {};
		\node [style=none] (18) at (-8.5, -0.5) {};
		\node [style=none] (19) at (-4, -4.75) {};
		\node [draw, thick, style=none, minimum size=0.2 cm, circle, fill=white] (20) at (-4, -1.75) {};
		\node [style=none] (21) at (-9.5, -2) {};
		\node [style=none] (22) at (-8.5, -2) {};
		\node [style=none] (23) at (-7.5, -0.5) {};
		\node [style=none] (24) at (-7, 1.5) {};
		\node [style=none] (25) at (-4, -1.75) {};
		\node [style=none] (26) at (-2.5, -0.5) {};
		\node [style=none] (27) at (-8.5, -1.25) {$\overline{\semantics{d}}$};
		\node [style=none] (28) at (-4, -5.75) {};
		\node [style=none] (29) at (-7, 2.75) {$W$};
		\node [style=none] (30) at (-9.5, -0.5) {};
		\node [style=none] (31) at (1.5, 6.75) {};
		\node [style=none] (32) at (-7, 2.25) {};
		\node [style=none] (33) at (-4, -1.75) {};
		\node [style=none] (34) at (-5.5, 0.25) {};
		\node [style=none] (35) at (-2.5, -0.5) {};
		\node [style=none] (36) at (-7.5, -2) {};
		\node [style=none] (37) at (-4, -5.25) {$W$};
		\node [style=none] (38) at (-8.5, 0.25) {};
		\node [style=none] (40) at (-2.5, 2.25) {};
		\node [style=none] (41) at (-8.5, -4.75) {};
		\node [draw, thick, style=none, minimum size=0.2 cm, circle, fill=black] (42) at (-7, 1.5) {};
		\node [style=none] (43) at (-8.5, -5.75) {};
	\end{pgfonlayer}
	\begin{pgfonlayer}{edgelayer}
		\draw  [style=thick](14.center) to (8.center);
		\draw  [style=thick](7.center) to (14.center);
		\draw  [style=thick](7.center) to (8.center);
		\draw  [style=thick](11.center) to (10.center);
		\draw  [style=thick](1.center) to (11.center);
		\draw  [style=thick](1.center) to (10.center);
		\draw  [style=thick](6.center) to (2.center);
		\draw  [style=thick](0.center) to (15.center);
		\draw  [style=thick](3.center) to (9.center);
		\draw [style=thick] (30.center) to (23.center);
		\draw [style=thick] (23.center) to (36.center);
		\draw [style=thick] (36.center) to (21.center);
		\draw [style=thick] (21.center) to (30.center);
		\draw [style=thick] (38.center) to (18.center);
		\draw [style=thick] (22.center) to (41.center);
		\draw [style=thick, in=90, out=270] (32.center) to (24.center);
		\draw [style=thick, bend left=90, looseness=1.25] (38.center) to (34.center);
		\draw [style=thick, bend right=75] (43.center) to (28.center);
		\draw [style=thick] (17.center) to (34.center);
		\draw [style=thick, bend right=90, looseness=1.25] (17.center) to (26.center);
		\draw [style=thick] (40.center) to (35.center);
		\draw [style=thick] (20.center) to (19.center);
	\end{pgfonlayer}
\end{tikzpicture}}
\endpgfgraphicnamed}
\end{minipage} 
\ $=$ \ \qquad
\hspace{3cm}\begin{minipage}{5cm}
 {%
\beginpgfgraphicnamed{Q-Sbj-Living}
\begin{tikzpicture}[scale=0.7]
	\path [use as bounding box] (-1.5,0) rectangle (-1.5,0.75);
	\begin{pgfonlayer}{nodelayer}
		\node [style=none] (0) at (-0.25, 3.75) {};
		\node [style=none] (1) at (0.5, 5.25) {};
		\node [style=none] (2) at (-7, 3.25) {};
		\node [style=none] (3) at (1.25, 3.75) {};
		\node [style=none] (4) at (1.25, 2.75) {$S$};
		\node [style=none] (5) at (-0.25, 2.75) {$W$};
		\node [style=none] (6) at (-7, 3.75) {};
		\node [style=none] (7) at (-7, 5.25) {};
		\node [style=none] (8) at (-6, 3.75) {};
		\node [style=none] (9) at (1.25, 3.25) {};
		\node [style=none] (10) at (2.25, 3.75) {};
		\node [style=none] (11) at (-1.25, 3.75) {};
		\node [style=none] (12) at (-7, 6) {$\overline{\semantics{n}}$};
		\node [style=none] (13) at (0.5, 6) {$\overline{\semantics{vp}}$};
		\node [style=none] (14) at (-8, 3.75) {};
		\node [style=none] (15) at (-0.25, 3.25) {};
		\node [style=none] (16) at (-8.5, -5.25) {$W$};
		\node [style=none] (17) at (-8.5, -0.5) {};
		\node [style=none] (18) at (-3.5, -4.75) {};
		\node [draw, thick, style=none, minimum size=0.2 cm, circle, fill=white] (19) at (-3.5, -1.75) {};
		\node [style=none] (20) at (-9.5, -2) {};
		\node [style=none] (21) at (-8.5, -2) {};
		\node [style=none] (22) at (-7.5, -0.5) {};
		\node [style=none] (23) at (-7, 1.5) {};
		\node [style=none] (24) at (-3.5, -1.75) {};
		\node [style=none] (25) at (-1.5, 0.75) {};
		\node [style=none] (26) at (-8.5, -1.25) {$\overline{\semantics{d}}$};
		\node [style=none] (27) at (-3.5, -5.75) {};
		\node [style=none] (28) at (-7, 2.75) {$W$};
		\node [style=none] (29) at (-9.5, -0.5) {};
		\node [style=none] (30) at (-7, 2.25) {};
		\node [style=none] (31) at (-3.5, -1.75) {};
		\node [style=none] (32) at (-5.5, 0.25) {};
		\node [style=none] (33) at (-7.5, -2) {};
		\node [style=none] (34) at (-3.5, -5.25) {$W$};
		\node [style=none] (35) at (-8.5, 0.25) {};
		\node [style=none] (36) at (-8.5, -4.75) {};
		\node [draw, thick, style=none, minimum size=0.2 cm, circle, fill=black] (37) at (-7, 1.5) {};
		\node [style=none] (38) at (-8.5, -5.75) {};
		\node [style=none] (39) at (-3, 3.25) {};
		\node [style=none] (40) at (-2, 3.75) {};
		\node [style=none] (41) at (-3, 3.75) {};
		\node [style=none] (42) at (-4, 3.75) {};
		\node [style=none] (43) at (-3, 5.25) {};
		\node [style=none] (44) at (-3, 2.75) {$W$};
		\node [style=none] (45) at (-3, 6) {$\overline{\semantics{n}}$};
		\node [style=none] (46) at (-0.25, 2) {};
		\node [style=none] (47) at (-1.5, 0.75) {};
		\node [style=none] (48) at (-1.5, 0.75) {};
		\node [draw, thick, style=none, minimum size=0.2 cm, circle, fill=white] (49) at (-1.5, 0.75) {};
		\node [style=none] (50) at (-3, 2) {};
		\node [style=none] (51) at (-0.25, 2) {};
		\node [style=none] (52) at (-1.5, 0.25) {};
	\end{pgfonlayer}
	\begin{pgfonlayer}{edgelayer}
		\draw [style=thick]  (14.center) to (8.center);
		\draw  [style=thick]  (7.center) to (14.center);
		\draw [style=thick]  (7.center) to (8.center);
		\draw [style=thick]  (11.center) to (10.center);
		\draw [style=thick]  (1.center) to (11.center);
		\draw [style=thick]  (1.center) to (10.center);
		\draw [style=thick]  (6.center) to (2.center);
		\draw [style=thick]  (0.center) to (15.center);
		\draw [style=thick]  (3.center) to (9.center);
		\draw [style=thick] (29.center) to (22.center);
		\draw [style=thick] (22.center) to (33.center);
		\draw [style=thick] (33.center) to (20.center);
		\draw [style=thick] (20.center) to (29.center);
		\draw [style=thick] (35.center) to (17.center);
		\draw [style=thick] (21.center) to (36.center);
		\draw [style=thick, in=90, out=270] (30.center) to (23.center);
		\draw [style=thick, bend left=90, looseness=1.25] (35.center) to (32.center);
		\draw [style=thick, bend right=75] (38.center) to (27.center);
		\draw [style=thick] (19.center) to (18.center);
		\draw [style=thick] (42.center) to (40.center);
		\draw [style=thick] (43.center) to (42.center);
		\draw [style=thick] (43.center) to (40.center);
		\draw [style=thick]  (41.center) to (39.center);
		\draw [style=thick, bend right=90, looseness=1.25] (50.center) to (51.center);
		\draw [style=thick, bend right=90, looseness=1.50] (32.center) to (52.center);
		\draw [style=thick] (49.center) to (52.center);
	\end{pgfonlayer}
\end{tikzpicture}}
\endpgfgraphicnamed}
\end{minipage}
\end{minipage}

 \medskip
 For the second case, consider the diagram corresponding to the relation $\star \semantics{\overline{np  \ v\  d\ n \ \text{\bf who \ are}  \ n}} \star$:

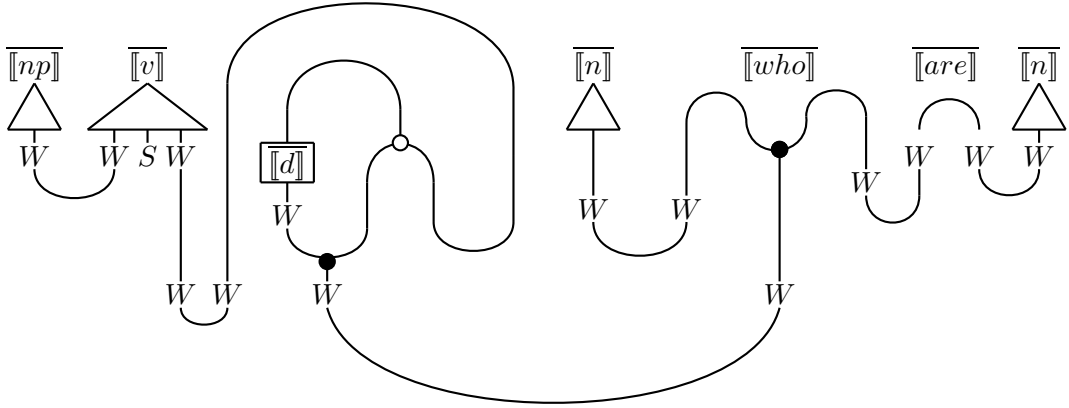
 \begin{center}
{%
\beginpgfgraphicnamed{conserv-obj}
\begin{tikzpicture}[scale=0.7]
	\begin{pgfonlayer}{nodelayer}
		\node [style=none] (0) at (18.25, 5.5) {$\overline{\semantics{n}}$};
		\node [style=none] (1) at (18.25, 3) {};
		\node [style=none] (2) at (-19.5, 2) {$W$};
		\node [style=none] (3) at (-13, 3) {};
		\node [style=none] (4) at (-15.25, 2) {$S$};
		\node [style=none] (5) at (18.25, 4.75) {};
		\node [style=none] (6) at (-14, 3) {};
		\node [style=none] (7) at (18.25, 2.5) {};
		\node [style=none] (8) at (19.25, 3) {};
		\node [style=none] (9) at (18.25, 2) {$W$};
		\node [style=none] (10) at (-15.25, 4.75) {};
		\node [style=none] (11) at (-14, 2) {$W$};
		\node [style=none] (12) at (-17.5, 3) {};
		\node [style=none] (13) at (17.25, 3) {};
		\node [style=none] (14) at (-19.5, 2.5) {};
		\node [style=none] (15) at (-19.5, 5.5) {$\overline{\semantics{np}}$};
		\node [style=none] (16) at (-20.5, 3) {};
		\node [style=none] (17) at (-15.25, 2.5) {};
		\node [style=none] (18) at (-19.5, 3) {};
		\node [style=none] (19) at (-19.5, 4.75) {};
		\node [style=none] (20) at (-18.5, 3) {};
		\node [style=none] (21) at (-16.5, 2) {$W$};
		\node [style=none] (22) at (-14, 2.5) {};
		\node [style=none] (23) at (-15.25, 3) {};
		\node [style=none] (24) at (-16.5, 3) {};
		\node [style=none] (25) at (-16.5, 2.5) {};
		\node [style=none] (26) at (-15.25, 5.5) {$\overline{\semantics{v}}$};
		\node [style=none] (27) at (-19.5, 1.5) {};
		\node [style=none] (28) at (-16.5, 1.5) {};
		\node [style=none] (29) at (-14, 1.5) {};
		\node [style=none] (30) at (-14, -2.75) {};
		\node [style=none] (31) at (-14, -3.75) {};
		\node [style=none] (32) at (-12.25, -3.75) {};
		\node [style=none] (33) at (13.75, 1.5) {};
		\node [style=none] (34) at (13.75, 0.5) {};
		\node [style=none] (35) at (-8.5, -3.75) {};
		\node [style=none] (36) at (8.5, -3.75) {};
		\node [style=none] (37) at (-14, -3.25) {$W$};
		\node [style=none] (38) at (-8.5, -2.75) {};
		\node [draw, thick, style=none, minimum size=0.2 cm, circle, fill=black] (39) at (-8.5, -2) {};
		\node [style=none] (40) at (-10, 1.75) {$\overline{\semantics{d}}$};
		\node [style=none] (41) at (-12.25, -2.75) {};
		\node [style=none] (42) at (-11, 2.5) {};
		\node [style=none] (43) at (-10, 0.25) {};
		\node [style=none] (44) at (-1.5, 2) {};
		\node [style=none] (45) at (-8.5, -2) {};
		\node [style=none] (46) at (-7, -0.75) {};
		\node [style=none] (47) at (-12.25, 4.75) {};
		\node [style=none] (48) at (-10, 2.5) {};
		\node [style=none] (49) at (-9, 1) {};
		\node [style=none] (50) at (-12.25, -3.25) {$W$};
		\node [style=none] (51) at (-11, 1) {};
		\node [style=none] (52) at (-1.5, 4.5) {};
		\node [style=none] (53) at (-10, -0.25) {$W$};
		\node [style=none] (54) at (-10, 0.25) {};
		\node [style=none] (55) at (-10, 3.75) {};
		\node [style=none] (56) at (-7, -0.75) {};
		\node [style=none] (57) at (-10, 1) {};
		\node [style=none] (58) at (-10, -0.75) {};
		\node [style=none] (59) at (-10, 3.75) {};
		\node [style=none] (60) at (-9, 2.5) {};
		\node [style=none] (61) at (-8.5, -3.25) {$W$};
		\node [draw, thick, style=none, minimum size=0.2 cm, circle, fill=white] (62) at (-5.75, 2.5) {};
		\node [style=none] (63) at (-7, 1) {};
		\node [style=none] (64) at (-4.5, 1) {};
		\node [style=none] (65) at (-1.5, -0.5) {};
		\node [style=none] (66) at (-4.5, -0.5) {};
		\node [style=none] (67) at (-5.75, 3.75) {};
		\node [style=none] (68) at (-5.75, 2.5) {};
		\node [style=none] (69) at (-5.75, 3.75) {};
		\node [style=none] (70) at (8.5, 5.5) {$\overline{\semantics{who}}$};
		\node [style=none] (71) at (11.75, 3.5) {};
		\node [style=none] (72) at (11.75, 0.5) {};
		\node [fill=white, draw, thick, circle, minimum size=0.2 cm, style=none] (73) at (8.5, 2.25) {};
		\node [style=none] (74) at (8.5, 2.25) {};
		\node [style=none] (75) at (7.25, 3.25) {};
		\node [style=none] (76) at (5, 0.5) {};
		\node [draw, thick, style=none, minimum size=0.2 cm, circle, fill=black] (77) at (8.5, 2.25) {};
		\node [style=none] (78) at (9.5, 3.5) {};
		\node [style=none] (79) at (11.75, 0.5) {};
		\node [style=none] (80) at (8.5, -3.25) {$W$};
		\node [style=none] (81) at (11.75, 1.5) {};
		\node [style=none] (82) at (8.5, -2.75) {};
		\node [style=none] (83) at (5, 3.25) {};
		\node [style=none] (84) at (5, -0.5) {};
		\node [style=none] (85) at (5, 0) {$W$};
		\node [style=none] (86) at (11.75, 1) {$W$};
		\node [style=none] (87) at (1.5, -0.5) {};
		\node [style=none] (88) at (1.5, 4.75) {};
		\node [style=none] (89) at (1.5, 0.5) {};
		\node [style=none] (90) at (0.5, 3) {};
		\node [style=none] (91) at (1.5, 3) {};
		\node [style=none] (92) at (1.5, 5.5) {$\overline{\semantics{n}}$};
		\node [style=none] (93) at (2.5, 3) {};
		\node [style=none] (94) at (1.5, 0) {$W$};
		\node [style=none] (95) at (13.75, 3) {};
		\node [style=none] (96) at (16, 3) {};
		\node [style=none] (97) at (14.75, 5.5) {$\overline{\semantics{are}}$};
		\node [style=none] (98) at (13.75, 2) {$W$};
		\node [style=none] (99) at (16, 2) {$W$};
		\node [style=none] (100) at (18.25, 1.5) {};
		\node [style=none] (101) at (16, 1.5) {};
	\end{pgfonlayer}
	\begin{pgfonlayer}{edgelayer}
		\draw [style=thick] (16.center) to (20.center);
		\draw  [style=thick](19.center) to (16.center);
		\draw [style=thick] (19.center) to (20.center);
		\draw [style=thick] (12.center) to (3.center);
		\draw [style=thick] (10.center) to (12.center);
		\draw [style=thick] (10.center) to (3.center);
		\draw [style=thick] (13.center) to (8.center);
		\draw [style=thick] (5.center) to (13.center);
		\draw [style=thick] (5.center) to (8.center);
		\draw [style=thick] (18.center) to (14.center);
		\draw [style=thick] (24.center) to (25.center);
		\draw [style=thick]  (23.center) to (17.center);
		\draw [style=thick] (6.center) to (22.center);
		\draw [style=thick] (1.center) to (7.center);
		\draw [style=thick, bend right=90, looseness=1.25] (27.center) to (28.center);
		\draw [style=thick](29.center) to (30.center);
		\draw [style=thick, bend right=90, looseness=1.25] (31.center) to (32.center);
		\draw [style=thick](33.center) to (34.center);
		\draw [style=thick](42.center) to (60.center);
		\draw [style=thick](60.center) to (49.center);
		\draw [style=thick](49.center) to (51.center);
		\draw [style=thick](51.center) to (42.center);
		\draw [style=thick](45.center) to (38.center);
		\draw [style=thick](57.center) to (43.center);
		\draw [style=thick](59.center) to (48.center);
		\draw [style=thick, in=270, out=-90, looseness=1.25] (58.center) to (56.center);
		\draw [style=thick](52.center) to (44.center);
		\draw [style=thick, bend left=90] (47.center) to (52.center);
		\draw [style=thick](47.center) to (41.center);
		\draw [style=thick, in=90, out=90, looseness=2.00] (63.center) to (64.center);
		\draw [style=thick, in=270, out=-90, looseness=1.25] (66.center) to (65.center);
		\draw [style=thick](69.center) to (68.center);
		\draw [style=thick, bend left=90, looseness=1.50] (59.center) to (67.center);
		\draw [style=thick] (63.center) to (46.center);
		\draw [style=thick] (64.center) to (66.center);
		\draw [style=thick] (44.center) to (65.center);
		\draw [style=thick] (71.center) to (81.center);
		\draw [style=thick, bend left=90, looseness=1.75] (83.center) to (75.center);
		\draw [style=thick, bend left=90, looseness=1.50] (78.center) to (71.center);
		\draw [style=thick, bend right=90, looseness=1.75] (75.center) to (78.center);
		\draw [style=thick] (83.center) to (76.center);
		\draw [style=thick] (74.center) to (82.center);
		\draw [style=thick, bend right=90, looseness=1.75] (72.center) to (34.center);
		\draw [style=thick] (90.center) to (93.center);
		\draw [style=thick] (88.center) to (90.center);
		\draw [style=thick] (88.center) to (93.center);
		\draw [style=thick] (91.center) to (89.center);
		\draw [style=thick, bend right=90, looseness=1.25] (87.center) to (84.center);
		\draw [style=thick, bend right=75, looseness=0.75] (35.center) to (36.center);
		\draw [style=thick, bend left=90, looseness=1.75] (95.center) to (96.center);
		\draw [style=thick, bend right=90, looseness=1.50] (101.center) to (100.center);
	\end{pgfonlayer}
\end{tikzpicture}}
\endpgfgraphicnamed}
\end{center}
 
 It simplifies to the following diagram:
 
 \begin{center}
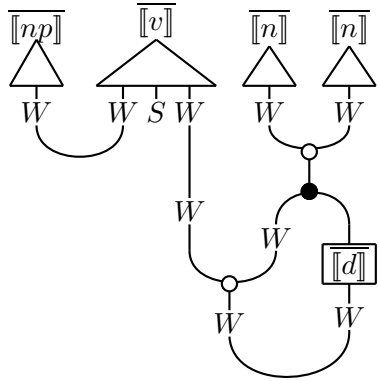
 \begin{minipage}{5cm}
 {%
\beginpgfgraphicnamed{Q-Obj-Living}
\begin{tikzpicture}[scale=0.7]
	\begin{pgfonlayer}{nodelayer}
		\node [style=none] (0) at (1, 5.25) {};
		\node [style=none] (1) at (-7.75, 3.25) {};
		\node [style=none] (2) at (1, 3.75) {};
		\node [style=none] (3) at (1, 3.25) {};
		\node [style=none] (4) at (-7.75, 3.75) {};
		\node [style=none] (5) at (-7.75, 5.5) {};
		\node [style=none] (6) at (-6.75, 3.75) {};
		\node [style=none] (7) at (1, 2.75) {$W$};
		\node [style=none] (8) at (-7.75, 2.75) {$W$};
		\node [style=none] (9) at (1, 6) {$\overline{\semantics{n}}$};
		\node [style=none] (10) at (-7.75, 6) {$\overline{\semantics{np}}$};
		\node [style=none] (11) at (2, 3.75) {};
		\node [style=none] (12) at (0, 3.75) {};
		\node [style=none] (13) at (-8.75, 3.75) {};
		\node [style=none] (14) at (-7.75, 2.25) {};
		\node [style=none] (15) at (-4.5, 2.25) {};
		\node [draw, thick, style=none, minimum size=0.2 cm, circle, fill=white] (16) at (2.5, -0.25) {};
		\node [style=none] (17) at (4, -5) {$W$};
		\node [style=none] (18) at (1.25, -2) {$W$};
		\node [style=none] (19) at (-0.5, -5.75) {};
		\node [style=none] (20) at (2.5, 1.25) {};
		\node [style=none] (21) at (4, -4.5) {};
		\node [style=none] (22) at (-2, -2.5) {};
		\node [style=none] (23) at (4, -3) {$\overline{\semantics{d}}$};
		\node [style=none] (24) at (-2, -0.5) {};
		\node [style=none] (25) at (5, -3.75) {};
		\node [style=none] (26) at (-2, -1) {$W$};
		\node [style=none] (27) at (4, -1.5) {};
		\node [style=none] (28) at (4, -5.5) {};
		\node [style=none] (29) at (3, -2.25) {};
		\node [style=none] (30) at (-2, 2.25) {};
		\node [style=none] (31) at (4, -3.75) {};
		\node [style=none] (32) at (3, -3.75) {};
		\node [style=none] (33) at (1.25, -1.5) {};
		\node [style=none] (34) at (-0.5, -5.25) {$W$};
		\node [style=none] (35) at (1.25, -2.5) {};
		\node [style=none] (36) at (-0.5, -4.75) {};
		\node [style=none] (37) at (5, -2.25) {};
		\node [style=none] (38) at (4, -2.25) {};
		\node [style=none] (39) at (-1, 3.75) {};
		\node [style=none] (40) at (-3.25, 3.25) {};
		\node [style=none] (41) at (-3.25, 2.75) {$S$};
		\node [style=none] (42) at (-5.5, 3.75) {};
		\node [style=none] (43) at (-2, 3.25) {};
		\node [style=none] (44) at (-3.25, 5.5) {};
		\node [style=none] (45) at (-2, 3.75) {};
		\node [style=none] (46) at (-4.5, 3.75) {};
		\node [style=none] (47) at (-4.5, 3.25) {};
		\node [style=none] (48) at (-2, 2.75) {$W$};
		\node [style=none] (49) at (-3.25, 3.75) {};
		\node [style=none] (50) at (-4.5, 2.75) {$W$};
		\node [draw, thick, style=none, minimum size=0.2 cm, circle, fill=black] (51) at (2.5, -0.25) {};
		\node [style=none] (52) at (-0.5, -3.75) {};
		\node [draw, thick, style=none, minimum size=0.2 cm, circle, fill=white] (53) at (-0.5, -3.75) {};
		\node [style=none] (54) at (-3.25, 6.25) {$\overline{\semantics{v}}$};
		\node [style=none] (55) at (4, 3.75) {};
		\node [style=none] (56) at (4, 5.25) {};
		\node [style=none] (57) at (5, 3.75) {};
		\node [style=none] (58) at (3, 3.75) {};
		\node [style=none] (59) at (4, 6) {$\overline{\semantics{n}}$};
		\node [style=none] (60) at (4, 3.25) {};
		\node [style=none] (61) at (4, 2.75) {$W$};
		\node [style=none] (62) at (-2, -1.5) {};
		\node [draw, thick, style=none, minimum size=0.2 cm, circle, fill=white] (63) at (2.5, 1.25) {};
		\node [style=none] (64) at (1, 2.25) {};
		\node [style=none] (65) at (4, 2.25) {};
	\end{pgfonlayer}
	\begin{pgfonlayer}{edgelayer}
		\draw  [style=thick] (13.center) to (6.center);
		\draw [style=thick] (5.center) to (13.center);
		\draw [style=thick] (5.center) to (6.center);
		\draw [style=thick] (12.center) to (11.center);
		\draw [style=thick](0.center) to (12.center);
		\draw [style=thick] (0.center) to (11.center);
		\draw [style=thick] (4.center) to (1.center);
		\draw [style=thick] (2.center) to (3.center);
		\draw [style=thick, bend right=90, looseness=1.25] (14.center) to (15.center);
		\draw [style=thick](29.center) to (37.center);
		\draw [style=thick](37.center) to (25.center);
		\draw [style=thick](25.center) to (32.center);
		\draw [style=thick](32.center) to (29.center);
		\draw [style=thick, in=270, out=90] (16.center) to (20.center);
		\draw [style=thick] (31.center) to (21.center);
		\draw [style=thick, in=270, out=-90, looseness=1.25] (22.center) to (35.center);
		\draw [style=thick, bend left=90, looseness=1.50] (33.center) to (27.center);
		\draw [style=thick](30.center) to (24.center);
		\draw [style=thick, bend right=90, looseness=1.25] (19.center) to (28.center);
		\draw [style=thick] (27.center) to (38.center);
		\draw [style=thick] (42.center) to (39.center);
		\draw [style=thick] (44.center) to (42.center);
		\draw [style=thick](44.center) to (39.center);
		\draw [style=thick](46.center) to (47.center);
		\draw [style=thick] (49.center) to (40.center);
		\draw [style=thick] (45.center) to (43.center);
		\draw [style=thick] (52.center) to (36.center);
		\draw [style=thick] (58.center) to (57.center);
		\draw [style=thick] (56.center) to (58.center);
		\draw [style=thick] (56.center) to (57.center);
		\draw [style=thick] (55.center) to (60.center);
		\draw [style=thick] (62.center) to (22.center);
		\draw [style=thick, bend right=90] (64.center) to (65.center);
	\end{pgfonlayer}
\end{tikzpicture}}
\endpgfgraphicnamed}
\end{minipage}
\end{center}

 The morphism corresponding to this diagram is as follows:
 \begin{align*}
 (1_{\{\star\}} \otimes \epsilon_{{\cal P}(U)}) \circ
 (1_{\{\star\}} \otimes \mu_{{\cal P}(U)} \otimes \overline{\semantics{d}} ) \circ
 (1_{\{\star\}} \otimes 1_{{\cal P}(U)}  \otimes \delta_{{\cal P}(U)} ) \circ
 (\epsilon_{{\cal P}(U)}  \otimes 1_{\{\star\}} \otimes 1_{{\cal P}(U)}  \otimes \mu_{{\cal P}(U)}) &\\
 (\overline{\semantics{np}} \otimes \overline{\semantics{v}} \otimes \overline{\semantics{n}} \otimes \overline{\semantics{n}}) \colon
 \{\star\} \relto \{\star\} &
 \end{align*}
 Similar to the previous case, following  almost identical  calculation steps as in the proof of Theorem  \ref{thm:truth-rel} for sentences with a quantified object,  the above calculates to:
 \[
 \semantics v (\semantics{np}) \cap (\semantics n \cap \semantics n) \in \semantics d (\semantics n)
 \]
which is equivalent to 
 \[
 \semantics v (\semantics{np}) \cap \semantics n  \in \semantics d (\semantics n)
 \]
  which as proved in Theorem  \ref{thm:truth-rel},  is the result of calculating the relation corresponding to  $\star \semantics{\overline{np \ v \ d \ n}} \star $.
  
 Diagrammatically, we have the following equality of the simplified forms of the diagrams of these sentences:
 
 \medskip
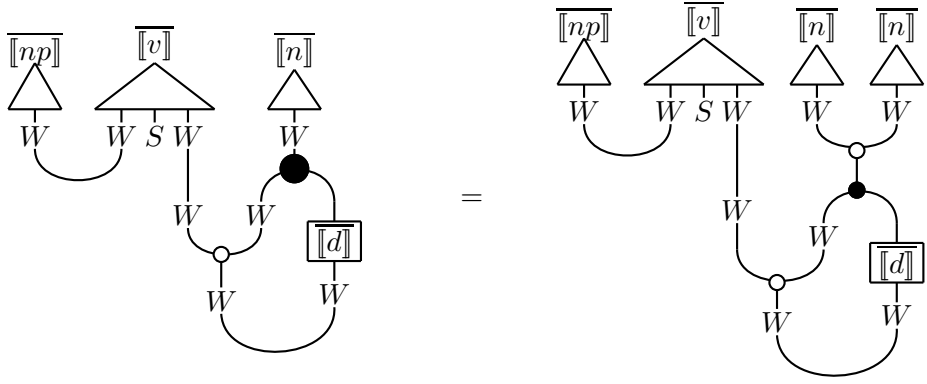
\begin{center}
\begin{minipage}{5cm}
{%
\beginpgfgraphicnamed{Q-Obj-Norm}
\begin{tikzpicture}[scale=0.7]
	\begin{pgfonlayer}{nodelayer}
		\node [style=none] (0) at (2, 4) {};
		\node [style=none] (1) at (-7.75, 2) {};
		\node [style=none] (2) at (2, 2.5) {};
		\node [style=none] (3) at (2, 2) {};
		\node [style=none] (4) at (-7.75, 2.5) {};
		\node [style=none] (5) at (-7.75, 4.25) {};
		\node [style=none] (6) at (-6.75, 2.5) {};
		\node [style=none] (7) at (2, 1.5) {$W$};
		\node [style=none] (8) at (-7.75, 1.5) {$W$};
		\node [style=none] (9) at (2, 4.75) {$\overline{\semantics{n}}$};
		\node [style=none] (10) at (-7.75, 4.75) {$\overline{\semantics{np}}$};
		\node [style=none] (11) at (3, 2.5) {};
		\node [style=none] (12) at (1, 2.5) {};
		\node [style=none] (13) at (-8.75, 2.5) {};
		\node [style=none] (14) at (-7.75, 1) {};
		\node [style=none] (15) at (-4.5, 1) {};
		\node [draw, style=thick, minimum size=0.2 cm, circle, fill=white] (16) at (2, 0.25) {};
		\node [style=none] (17) at (3.5, -4.5) {$W$};
		\node [style=none] (18) at (0.75, -1.5) {$W$};
		\node [style=none] (19) at (-0.75, -5.25) {};
		\node [style=none] (20) at (2, 1) {};
		\node [style=none] (21) at (3.5, -4) {};
		\node [style=none] (22) at (-2, -2) {};
		\node [style=none] (23) at (3.5, -2.5) {$\overline{\semantics{d}}$};
		\node [style=none] (24) at (-2, -1) {};
		\node [style=none] (25) at (4.5, -3.25) {};
		\node [style=none] (26) at (-2, -1.5) {$W$};
		\node [style=none] (27) at (3.5, -1) {};
		\node [style=none] (28) at (3.5, -5) {};
		\node [style=none] (29) at (2.5, -1.75) {};
		\node [style=none] (30) at (-2, 1) {};
		\node [style=none] (31) at (3.5, -3.25) {};
		\node [style=none] (32) at (2.5, -3.25) {};
		\node [style=none] (33) at (0.75, -1) {};
		\node [style=none] (34) at (-0.75, -4.75) {$W$};
		\node [style=none] (35) at (0.75, -2) {};
		\node [style=none] (36) at (-0.75, -4.25) {};
		\node [style=none] (37) at (4.5, -1.75) {};
		\node [style=none] (38) at (3.5, -1.75) {};
		\node [style=none] (39) at (-1, 2.5) {};
		\node [style=none] (40) at (-3.25, 2) {};
		\node [style=none] (41) at (-3.25, 1.5) {$S$};
		\node [style=none] (42) at (-5.5, 2.5) {};
		\node [style=none] (43) at (-2, 2) {};
		\node [style=none] (44) at (-3.25, 4.25) {};
		\node [style=none] (45) at (-2, 2.5) {};
		\node [style=none] (46) at (-4.5, 2.5) {};
		\node [style=none] (47) at (-4.5, 2) {};
		\node [style=none] (48) at (-2, 1.5) {$W$};
		\node [style=none] (49) at (-3.25, 2.5) {};
		\node [style=none] (50) at (-4.5, 1.5) {$W$};
		\node [draw, style=thick, minimum size=0.1 cm, circle, fill=black]  (51) at (2, 0.25) {};
		\node [style=none] (52) at (-0.75, -3) {};
		\node [draw, none, style=thick, minimum size=0.2 cm, circle, fill=white] (53) at (-0.75, -3) {};
		\node [style=none] (54) at (-3.25, 5) {$\overline{\semantics{v}}$};
	\end{pgfonlayer}
	\begin{pgfonlayer}{edgelayer}
		\draw  [style=thick]  (13.center) to (6.center);
		\draw  [style=thick]  (5.center) to (13.center);
		\draw   [style=thick] (5.center) to (6.center);
		\draw  [style=thick] (12.center) to (11.center);
		\draw  [style=thick]  (0.center) to (12.center);
		\draw  [style=thick]  (0.center) to (11.center);
		\draw  [style=thick]  (4.center) to (1.center);
		\draw  [style=thick]  (2.center) to (3.center);
		\draw [style=thick, bend right=90, looseness=1.25] (14.center) to (15.center);
		\draw  [style=thick]  (29.center) to (37.center);
		\draw  [style=thick]  (37.center) to (25.center);
		\draw  [style=thick]  (25.center) to (32.center);
		\draw  [style=thick] (32.center) to (29.center);
		\draw [style= thick, in=270, out=90] (16.center) to (20.center);
		\draw   [style=thick] (31.center) to (21.center);
		\draw [style=thick, in=270, out=-90, looseness=1.25] (22.center) to (35.center);
		\draw [style=thick, bend left=90, looseness=1.50] (33.center) to (27.center);
		\draw  [style=thick] (30.center) to (24.center);
		\draw [style = thick, bend right=90, looseness=1.25] (19.center) to (28.center);
		\draw  [style=thick] (27.center) to (38.center);
		\draw  [style=thick] (42.center) to (39.center);
		\draw  [style=thick] (44.center) to (42.center);
		\draw  [style=thick] (44.center) to (39.center);
		\draw  [style=thick] (46.center) to (47.center);
		\draw  [style=thick] (49.center) to (40.center);
		\draw  [style=thick] (45.center) to (43.center);
		\draw  [style=thick]  (52.center) to (36.center);
	\end{pgfonlayer}
\end{tikzpicture}}
\endpgfgraphicnamed}
\end{minipage} 
\qquad
\ $=$ \ \qquad
\begin{minipage}{5cm}
 {%
\beginpgfgraphicnamed{Q-Obj-Living}
\begin{tikzpicture}[scale=0.7]
	\begin{pgfonlayer}{nodelayer}
		\node [style=none] (0) at (1, 5.25) {};
		\node [style=none] (1) at (-7.75, 3.25) {};
		\node [style=none] (2) at (1, 3.75) {};
		\node [style=none] (3) at (1, 3.25) {};
		\node [style=none] (4) at (-7.75, 3.75) {};
		\node [style=none] (5) at (-7.75, 5.5) {};
		\node [style=none] (6) at (-6.75, 3.75) {};
		\node [style=none] (7) at (1, 2.75) {$W$};
		\node [style=none] (8) at (-7.75, 2.75) {$W$};
		\node [style=none] (9) at (1, 6) {$\overline{\semantics{n}}$};
		\node [style=none] (10) at (-7.75, 6) {$\overline{\semantics{np}}$};
		\node [style=none] (11) at (2, 3.75) {};
		\node [style=none] (12) at (0, 3.75) {};
		\node [style=none] (13) at (-8.75, 3.75) {};
		\node [style=none] (14) at (-7.75, 2.25) {};
		\node [style=none] (15) at (-4.5, 2.25) {};
		\node [draw, thick, style=none, minimum size=0.2 cm, circle, fill=white] (16) at (2.5, -0.25) {};
		\node [style=none] (17) at (4, -5) {$W$};
		\node [style=none] (18) at (1.25, -2) {$W$};
		\node [style=none] (19) at (-0.5, -5.75) {};
		\node [style=none] (20) at (2.5, 1.25) {};
		\node [style=none] (21) at (4, -4.5) {};
		\node [style=none] (22) at (-2, -2.5) {};
		\node [style=none] (23) at (4, -3) {$\overline{\semantics{d}}$};
		\node [style=none] (24) at (-2, -0.5) {};
		\node [style=none] (25) at (5, -3.75) {};
		\node [style=none] (26) at (-2, -1) {$W$};
		\node [style=none] (27) at (4, -1.5) {};
		\node [style=none] (28) at (4, -5.5) {};
		\node [style=none] (29) at (3, -2.25) {};
		\node [style=none] (30) at (-2, 2.25) {};
		\node [style=none] (31) at (4, -3.75) {};
		\node [style=none] (32) at (3, -3.75) {};
		\node [style=none] (33) at (1.25, -1.5) {};
		\node [style=none] (34) at (-0.5, -5.25) {$W$};
		\node [style=none] (35) at (1.25, -2.5) {};
		\node [style=none] (36) at (-0.5, -4.75) {};
		\node [style=none] (37) at (5, -2.25) {};
		\node [style=none] (38) at (4, -2.25) {};
		\node [style=none] (39) at (-1, 3.75) {};
		\node [style=none] (40) at (-3.25, 3.25) {};
		\node [style=none] (41) at (-3.25, 2.75) {$S$};
		\node [style=none] (42) at (-5.5, 3.75) {};
		\node [style=none] (43) at (-2, 3.25) {};
		\node [style=none] (44) at (-3.25, 5.5) {};
		\node [style=none] (45) at (-2, 3.75) {};
		\node [style=none] (46) at (-4.5, 3.75) {};
		\node [style=none] (47) at (-4.5, 3.25) {};
		\node [style=none] (48) at (-2, 2.75) {$W$};
		\node [style=none] (49) at (-3.25, 3.75) {};
		\node [style=none] (50) at (-4.5, 2.75) {$W$};
		\node [draw, thick, style=none, minimum size=0.2 cm, circle, fill=black] (51) at (2.5, -0.25) {};
		\node [style=none] (52) at (-0.5, -3.75) {};
		\node [draw, thick, style=none, minimum size=0.2 cm, circle, fill=white] (53) at (-0.5, -3.75) {};
		\node [style=none] (54) at (-3.25, 6.25) {$\overline{\semantics{v}}$};
		\node [style=none] (55) at (4, 3.75) {};
		\node [style=none] (56) at (4, 5.25) {};
		\node [style=none] (57) at (5, 3.75) {};
		\node [style=none] (58) at (3, 3.75) {};
		\node [style=none] (59) at (4, 6) {$\overline{\semantics{n}}$};
		\node [style=none] (60) at (4, 3.25) {};
		\node [style=none] (61) at (4, 2.75) {$W$};
		\node [style=none] (62) at (-2, -1.5) {};
		\node [draw, thick, style=none, minimum size=0.2 cm, circle, fill=white] (63) at (2.5, 1.25) {};
		\node [style=none] (64) at (1, 2.25) {};
		\node [style=none] (65) at (4, 2.25) {};
	\end{pgfonlayer}
	\begin{pgfonlayer}{edgelayer}
		\draw  [style=thick] (13.center) to (6.center);
		\draw [style=thick] (5.center) to (13.center);
		\draw [style=thick] (5.center) to (6.center);
		\draw [style=thick] (12.center) to (11.center);
		\draw [style=thick](0.center) to (12.center);
		\draw [style=thick] (0.center) to (11.center);
		\draw [style=thick] (4.center) to (1.center);
		\draw [style=thick] (2.center) to (3.center);
		\draw [style=thick, bend right=90, looseness=1.25] (14.center) to (15.center);
		\draw [style=thick](29.center) to (37.center);
		\draw [style=thick](37.center) to (25.center);
		\draw [style=thick](25.center) to (32.center);
		\draw [style=thick](32.center) to (29.center);
		\draw [style=thick, in=270, out=90] (16.center) to (20.center);
		\draw [style=thick] (31.center) to (21.center);
		\draw [style=thick, in=270, out=-90, looseness=1.25] (22.center) to (35.center);
		\draw [style=thick, bend left=90, looseness=1.50] (33.center) to (27.center);
		\draw [style=thick](30.center) to (24.center);
		\draw [style=thick, bend right=90, looseness=1.25] (19.center) to (28.center);
		\draw [style=thick] (27.center) to (38.center);
		\draw [style=thick] (42.center) to (39.center);
		\draw [style=thick] (44.center) to (42.center);
		\draw [style=thick](44.center) to (39.center);
		\draw [style=thick](46.center) to (47.center);
		\draw [style=thick] (49.center) to (40.center);
		\draw [style=thick] (45.center) to (43.center);
		\draw [style=thick] (52.center) to (36.center);
		\draw [style=thick] (58.center) to (57.center);
		\draw [style=thick] (56.center) to (58.center);
		\draw [style=thick] (56.center) to (57.center);
		\draw [style=thick] (55.center) to (60.center);
		\draw [style=thick] (62.center) to (22.center);
		\draw [style=thick, bend right=90] (64.center) to (65.center);
	\end{pgfonlayer}
\end{tikzpicture}}
\endpgfgraphicnamed}
\end{minipage}
\end{center}
\end{proof}

We end this section with three notes. First is that following \cite{BarwiseCooper81}, our $G_Q$ does not generate relative clauses, copulous sentences, and adjectival phrases.  These are, however, needed for stating the equivalences related to the living on property.  In the relational instantiations of Corollary \ref{cor:livingon},   we are implicitly working with an extended form of $G_Q$ that generates these expressions.  In this extended form, interpreting the to-be verb "are"  in its copulous form as  an eta map amounts to stipulating an equality such as    $\semantics{\mbox{a is b}} = \semantics{\text{b}}(\semantics{\text{a}})$.
Secondly, the $\mu$ and $\zeta$  maps of previous work \cite{RelPronMoL,SadrClarkCoecke1,SadrClarkCoecke2} were the coalgebra maps of the Frobenius algebra over $\Rel$. These Frobenius algebras were  defined over the universe of reference $U$.  The above results, however, are over ${\cal P}(U)$ and thus use the $\mu$ and $\zeta$ of the coalgebra maps of our bialgebra  over $\Rel$.  Finally,  since the sentence space of our model is the monoidal unit of $\Rel$, the $\zeta_S$ map becomes  identity;  for simplicity we have dropped it from the type of `who'. 




\section{Corpus-Based Instantiation in $\FdVect$}


The relational model  embeds  into  a  vector spaces model using the usual embedding of sets and relations into vector spaces and linear maps.  This embedding sends a set $T$  to a vector space $V_T$  spanned by elements of $T$ and a relation $R \subseteq T \times T$ to a linear map $V_T \to V_T$. By taking $T$ to be ${\cal P(U)}$ for the distinguished space $W$ and by taking it to be $\{\star\}$ for the distinguished space $S$, this  embedding provides us with a  vector space instantiation of the categorical model. This instantiation   imitates the truth theoretic model presented in $\Rel$. We refer to it by the \emph{boolean} $\FdVect$  instantiation. 

\begin{definition}\label{def:concrete-fdect}
The boolean  instantiation of the abstract model of definition \ref{ccc-model} to $\FdVect$  is the tuple    $(\FdVect, V_{{\cal P(U)}},V_{\{\star\}}, \ovl{\semantics{\ }})$, for  $V_{{\cal P(U)}}$  the free vector space generated over the set of subsets of  ${\cal U}$ and $V_{\{\star\}}$ the one dimensional space.   Words  are interpreted by the following linear maps:
\begin{itemize}
\item The   terminals generated by N, NP,  VP, and V rules are given by:
\[ 
\ovl{\semantics{x}} (\star) = \ket{\semantics{x}} 
\]
\item The interpretation of  a terminal $d$ generated by the Det rule  is defined as follows on subsets $A$ of  ${\cal U}$:
\[
\ovl{\semantics{d}} (\ket{A}) = \sum_{B \in \semantics{d}(A)} \ket{B}
\]
\end{itemize}
\end{definition}

The types of these linear maps are as in definition \ref{def:concrete-REL}, since $V_{\{\star\}}  \cong \mathbb{R}$ is the unit of tensor in $\FdVect$. Thus, the   terminals generated by N, NP, and VP rules have type $V_{\{\star\}} \to V_{{\cal P (U)}}$; the type of terminals generated by the V rule is $V_{\{\star\}} \to V_{{\cal P (U)}} \otimes V_{\{\star\}}\otimes  V_{{\cal P (U)}} \cong  V_{{\cal P (U)}} \otimes  V_{{\cal P (U)}}$.  A terminal generated by the Det rule has type $V_{{\cal P (U)}} \to  V_{{\cal P (U)}}$.

Theorem \ref{thm:truth-rel}  is  carried over from $\Rel$ to $\FdVect$  by defining vector representations of  sentences to be true iff they are non-zero elements of $V_{\{\star\}}$.  

\begin{definition}
\label{def:truth-fdvect}
The interpretation of a quantified sentence $s$  is true in $(\FdVect, V_{{\cal P(U)}},V_{\{\star\}}, \ovl{\semantics{\ }})$ \ iff \   $\overline{\semantics{\mbox{s}}}(\star) \neq 0 $. 
\end{definition}

\begin{corollary}
\label{cor:truth-fvect}
$\overline{\semantics{\mbox{s}}}(\star) \neq 0 $  in $(\FdVect, V_{{\cal P(U)}},V_{\{\star\}}, \ovl{\semantics{\ }})$  \ iff  \  $\star \overline{\semantics{\mbox{s}}}  \star$ in $(\Rel, \cal P (U), \{\star\}, \overline{\semantics{\ }})$. 
\end{corollary}
\begin{proof}
The proof goes through the  same cases and steps as in Theorem \ref{thm:truth-rel}. Consider a quantified sentence of the form  `Det N VP'. Its interpretation is obtained by   calculating $\overline{\semantics{\mbox{s}}}(\star)$, defined to be:
\[
\epsilon_{V_{\cal P(U)}}  \circ (\overline{\semantics{d}} \otimes \mu_{V_{\cal P(U)}}) \circ (\delta_{V_{\cal P(U)}}   \otimes \operatorname{id}_{V_{\cal P(U)}}) \circ (\overline{\semantics{n}} \otimes \overline{\semantics{vp}})(\star) \]
The four stages of this computation are as follows
\begin{eqnarray}
(\overline{\semantics{n}} \otimes \overline{\semantics{vp}})(\star) &=& \overline{\semantics{n}}(\star) \otimes \overline{\semantics{vp}}(\star) = \ket{\semantics{n}} \otimes \ket{\semantics{vp}}\\
(\delta_{V_{\cal P(U)}}   \otimes \operatorname{id}_{V_{\cal P(U)}}) (\ket{\semantics{n}} \otimes \ket{\semantics{vp}}) &=& 
\ket{\semantics{n}} \otimes \ket{\semantics{n}} \otimes \ket{\semantics{vp}}
\\
(\overline{\semantics{d}} \otimes \mu_{V_{\cal P(U)}})(\ket{\semantics{n}} \otimes \ket{\semantics{n}} \otimes \ket{\semantics{vp}}) &=& \sum_{B \in \semantics{d}(\semantics{n})} \ket{B} \otimes \ket{\semantics{n} \cap \semantics{vp}}
\\
\epsilon_{V_{\cal P(U)}}\left(\sum_{B \in \semantics{d}(\semantics{n})} \ket{B} \otimes \ket{\semantics{n} \cap \semantics{vp}}\right) &=&  \sum_{B \in \semantics{d}(\semantics{n})} \langle B \mid {\semantics{n} \cap \semantics{vp}} \rangle
\end{eqnarray}
The interpretation of a sentence with a quantified object `NP V Det N' is computed similarly, resulting in the following expression:
\[
 \sum_{B \in \semantics{d}(\semantics{n})} \langle \semantics v (\semantics{np}) \cap \semantics n \mid  B \rangle
\]
The result of the first case is non zero iff there is a subset $B \in \semantics{d}(\semantics{n})$ that is equal to ${\semantics{n} \cap \semantics{vp}}$. The result of the second case is non zero  iff there is a subset $B \in \semantics{d}(\semantics{n})$ that is equal to $\semantics v (\semantics{np}) \cap \semantics n$. These are respectively equivalent to their corresponding cases in $\star \overline{\semantics{\mbox{s}}}  \star$, as computed in the proof of theorem \ref{thm:truth-rel}. 
\end{proof}

A  corpus-based distributional vector space instantiation of the  model  is obtained via a  construction similar to the  above, but this time with  real number weights (rather than boolean ones). These weights are retrievable from corpora of text using distributional methods. The non-quantified part of this instantiation closely follows that of previous work  \cite{Coeckeetal}:  nouns and noun phrases live in  distributional spaces similar to the one described in subsection \ref{subsec:DistVect}; verb phrases and transitive verbs  live in tensor spaces, built using the  methods described described in the concrete instantiations of the theoretical model of previous work, e.g. see \cite{GrefenSadr,kartsaklis2012}. 

\begin{definition}
\label{def:concretevectorspace}
The distributional  instantiation of the abstract model of definition \ref{ccc-model} to $\FdVect$  is the tuple $(\FdVect, V_{{\cal P}(\Sigma)},Z, \ovl{\semantics{\text{\ }}})$, for $V_{{\cal P}(\Sigma)}$   the vector space freely generated over the set $\Sigma$ and  $Z$ a vector space wherein interpretations of sentences  live. The interpretations of terminals are defined as follows:

\begin{itemize}
\item
A terminal $x$ generated by  N or NP rules is given by  $\ovl{\semantics{x}} (1) := \sum_i c_i^x \ket{A_i}$ for $A_i \subseteq \Sigma$. 
\item A terminal $x$ generated by the VP rule is given by $\ovl{\semantics{x}} (1) := \sum_{jk} c_{jk}^x \ket{A_j \otimes A_k}$, for $A_j \subseteq \Sigma$ and $\ket{A_k}$ a basis vector of $Z$. 
\item A terminal $x$ generated by the V rule  is given by $\ovl{\semantics{x}} (1) := \sum_{lmn} c_{lmn}^x \ket{A_l \otimes A_m \otimes A_n}$, for $A_l,A_n \subseteq \Sigma$ and $\ket{A_m}$ a basis vector of $Z$. 
\item A terminal  $d$ generated by the  Det rule is concretely given on subsets $A$ of $\Sigma$ by  $\ovl{\semantics{d}}(\ket{A})  = \sum_{B \in \semantics{d}(A)} c_B^d \ket{B}$.
\end{itemize}
\end{definition}

As for the types, a terminal generated by the either of the N and NP rules has type $\mathbb{R} \to V_{{\cal P}(\Sigma)}$, a VP terminal has type $\mathbb{R} \to V_{{\cal P}(\Sigma)} \otimes Z$; the type of a V terminal is $\mathbb{R} \to V_{{\cal P}(\Sigma)} \otimes Z \otimes  V_{{\cal P}(\Sigma)}$. A terminal  $d$ generated by the  Det rule has type $V_{{\cal P}(\Sigma)} \to V_{{\cal P}(\Sigma)}$.

Examples of this model are obtained by setting three sets of parameters:  (1) instantiating $Z$ to different  sentence spaces,  (2)  different ways of embedding the distributional vectors of $V_{\Sigma}$ in the space $V_{{\cal P}(\Sigma)}$, and (3) different ways in which word vectors and tensors are built. The concrete constructions for the weighted interpretations of quantifiers  depend on these choices, but  can be implemented  according to the same general guidelines.  The weight  $c_B^d$ of a  quantifier $d$ over the basis $A$   can  stand for  a \emph{degree of set membership}. In this case  $\sum_{B \in \semantics{d}(A)} c_B^d \ket{B}$  can be implemented as   `$c_B^d$ is the degree to which $d$ elements of $A$ are in $B$'. This weight  can also stand for  a \emph{degree of co-occurrence} and be  retrieved from a corpus. In this case, $\sum_{B \in \semantics{d}(A)} c_B^d \ket{B}$ is read as `$c_B^d$ is the degree to which $d$ elements of $A$ \emph{have co-occurred with} $B$'.  We  provide three example  instantiations below. 

\par{\bf  Scalar Sentence Dimensions.} Suppose  $Z = \mathbb{R}$. The interpretation of a sentence with a quantified subject becomes as follows:
\[
\sum_{ij}  \sum_{B \in \semantics{d}(\semantics{n})} c_i^n c_{j}^{vp} c_B^d \langle B \mid A_i \cap A_j \rangle
\]
Similarly, the interpretation of a sentence with a quantified object becomes as follows:
\[
\sum_{ijlm}  \sum_{B \in \semantics{d}(\semantics{n})} c_i^{np} c_{jl}^v c_m^n c_B^d \langle A_i \mid A_j \rangle  \langle A_l \cap A_m \mid B \rangle
\]
Here, take $\Sigma = {\cal U}$ and one can use the $\Rel$-to-$\FdVect$ embedding and  obtain a weighted  version of the boolean model of definition \ref{def:concrete-fdect}. 

\par{\bf Distributional  Sentence Dimensions.} Suppose  $\cal S$ contains the  sentence  dimensions of a  compositional distributional model of meaning and take $Z = V_{\cal S}$. The sentence  dimensions can  be constructed in different ways. In \cite{GrefenSadr}, they were taken to be $\mathbb{R}$, whereas in 
\cite{kartsaklis2012}, we took them to be  the same as the dimensions of $V_\Sigma$.   In either case, there are different options on how to interpret the  dimensions of  $V_{{\cal P}(\Sigma)}$  in a distributional model.  We present three different constructions below. 

\begin{enumerate}
\item {\bf The singleton construction.} Take the interpretation of  a terminal $x$ generated by  either of the N or NP rules to be  $\sum_i c_i^x \ket{\{v_i\}}$  whenever $\sum_i c_i^x \ket{v_i}$ is the vector interpretation of $x$ in the distributional space $V_{\Sigma}$. Similarly,  a terminal $x$ generated by the VP rule is embedded as  $\sum_{ij} c_{ij}^x \ket{\{v_i\} \otimes s_j}$ whenever  $\sum_{ij} c_{ij}^x \ket{v_i \otimes s_j}$ is the matrix interpretation of $x$ in $V_{\Sigma} \otimes V_{\cal S}$. In the same fashion, a terminal $x$ generated by the V rule embeds as  $\sum_{ijk} c_{ijk}^x \ket{\{v_i\}  \otimes {s}_j  \otimes \{{v}_k\}}$, for $\sum_{ijk} c_{ijk}^x \ket{{v}_i \otimes {s}_j \otimes {v}_k}$ the cube interpretation of $x$ in $V_{\Sigma} \otimes V_{\cal S} \otimes V_{\Sigma}$. 

The interpretation of a sentence with a quantified subject becomes as follows:
\[
\sum_{ijk}  \sum_{B \in \semantics{d}(\semantics{n})} c_i^n c_{jk}^{vp} c_B^d \langle B \mid \{v_i\} \cap \{v_j\}   \rangle \ket{s_k}
\]
Similarly, for the interpretation of a sentence with a quantified object we obtain:
\[
\sum_{ijklm} \sum_{B \in \semantics{d}(\semantics{n})} c_i^{np} c_{jkl}^v c_m^n c_B^d \langle \{v_i\} \mid \{v_j\} \rangle  \ket{s_k}  \langle \{v_l\} \cap \{v_m\} \mid B \rangle
\]
The weights in the above formulae  come from the underlying compositional distributional model.  The  vector constructions for nouns and noun phrases are obtained by following  a distributional model; the matrix and cube constructions for verbs are constructed as detailed in \cite{GrefenSadr} or in \cite{kartsaklis2012}, depending on the choice of $\cal S$.

\item {\bf Sets of dimensions as lemmas.} A lemma is a set of different forms of a word. In this instantiation, each dimension of $V_{{\cal P}(\Sigma)}$ stands for a lemma.  

The interpretation of a sentence with a quantified sentence becomes:
 \[
\sum_{ijk}  \sum_{B \in \semantics{d}(\semantics{n})} c_i^n c_{jk}^{vp} c_B^d \langle B \mid A_i \cap A_j \rangle \ket{s_k}
\]
Similarly, the interpretation of a sentence with a quantified object  becomes:
\[
\sum_{ijklm} \sum_{B \in \semantics{d}(\semantics{n})} c_i^{np} c_{jkl}^v c_m^n c_B^d \langle A_i \mid A_j \rangle  \ket{s_k}  \langle A_l \cap A_m \mid B \rangle
\]
The weights  are  retrieved from a corpus by e.g. adding, normalizing, and clustering (e.g. average or $k$-means) of the co-occurrence weights of the elements of the lemma set.

\item {\bf Sets of dimensions as features.} A feature is the set of  words that together represent a pertinent property.   In this instantiating, each such dimension  of $V_{{\cal P}(\Sigma)}$ represents a set of such words. For instance, $\{$miaow, purr$\}$ is the sound feature for the  `animals',  $\{$run, sleep$\}$ is its action feature, and $\{$cat, kitten$\}$ is its species feature.  Each dimension of $V_{{\cal P}(\Sigma)}$ stands for a feature. The interpretations of quantified sentences are obtained by computing the same formulae as in the lemma instantiation, but the concrete values of the weights are obtained differently. 
\end{enumerate}

Lemmas and features are sets of words.  Whereas lemmas are syntactic objects (different syntactic forms of a word), features represent semantic properties of words.  Any set of words can in principle be a feature, namely, a common property  shared by the words of the corresponding set. For instance the words in the set $\{$sky, sea, blueberry, winter$\}$   represent a feature  corresponding to them being  "blue";  the words in the set $\{$milk, beer, Lucozade, soya sauce$\}$ represent a feature since they are all drinkable. Thus every subset of $\Sigma$ is feature and  every dimension in  $V_{{\cal P}(\Sigma)}$  becomes interpretable in the third embedding above. This is not the case in the second embedding.   Not every subset of  $\Sigma$   corresponds to a lemma, thus not every   dimension of $V_{{\cal P}(\Sigma)}$ is interpretable.  In effect, for the theory and also the practical parts of this project, we do not need the whole of the ${\cal P}(\Sigma)$, as we have only used the intersection property of $\Rel$ with bialgebras. Working in a subset thereof, such as down or up sets is work under progress.

As an example of the third embedding, consider the feature set instantiation and suppose  the following are among the features of $V_{{\cal P}(\Sigma)}$:
\[
\{\mbox{cats, kittens}\}, \{\mbox{miaow, purr}\}, \{\mbox{sleep, snore}\} \quad \in \quad {\cal P}(\Sigma)   
\]
Take the instantiation of the universal quantifier over these  to be:
\begin{center}
\begin{tabular}{c|c|c|c|}
$\ovl{\semantics{\text{all}}}$ &  $\ket{\{\mbox{cats, kittens}\}}$ &$ \ket{\{\mbox{miaow, purr}\}}$ &$\ket{\{\mbox{sleep, snore}\}}$\\
\hline
$\ket{\{\mbox{cats, kittens}\}}$ & {\tt small} & 0.7& 0.5\\
 $\ket{\{\mbox{miaow, purr}\}}$ &  0.9 &  {\tt small} & 0.3\\
 $\ket{\{\mbox{sleep, snore}\}} $ &  0.2 &  0.3&  {\tt small}\\
 \hline
\end{tabular}
\end{center}
In the first row, 0.7 is the degree to which \emph{all} elements of $\{\mbox{cats, kittens}\}$ have feature 
$\{\mbox{miaow, purr}\}$, witnessed by the fact that, for instance, all occurrences of cats and kittens in the corpus have  occurred in sentences which have a verb such as  miaow or purr. Similarly, 
 0.5 is the degree to which \emph{all} elements of $\{\mbox{cats, kittens}\}$ have feature $
 \{\mbox{sleep, snore}\}$. The intersection of a term with itself has no information content and is  thus taken to be a very small fraction, so as not to play a  role in  deductions. 
 
For the existential quantifier, a similar instantiation results in higher degrees as the quantifier is more relaxed, witnessed by the fact that, for instance, `kittens'  have more of the miaow feature than `cats' since they miaow more. Suppose this provides us with the  following:
\begin{center}
\begin{tabular}{c|c|c|c|}
$\ovl{\semantics{\text{some}}}$ &  $\ket{\{\mbox{cats, kittens}\}}$ &$ \ket{\{\mbox{miaow, purr}\}}$ &$\ket{\{\mbox{sleep, snore}\}}$\\
\hline
$\ket{\{\mbox{cats, kittens}\}}$ &  {\tt small} & 0.9  & 0.6\\
$\ket{\{\mbox{miaow, purr}\}}$ & 0.9&  {\tt small}& 0.5 \\
$ \ket{\{\mbox{sleep, snore}\}}$& 0.5& 0.5&  {\tt small}\\
\hline
\end{tabular}
\end{center}
Suppose the vectors  of `animal' and `run' in this space are as follows:
\begin{center}
\begin{tabular}{c|c|c|c|}
& $\ket{\{\mbox{cats, kittens}\}}$ &$ \ket{\{\mbox{miaow, purr}\}}$ &$\ket{\{\mbox{sleep, snore}\}}$\\
\hline
$\ovl{\semantics{\text{animal}}}$ &  0.5& 0.4 & 0.3\\
$\ovl{\semantics{\text{run}}} $ & 0.6 $\ket{s_1}$ & 0.4$\ket{s_2}$& 0.2$\ket{s_3}$\\
\hline
\end{tabular}
\end{center}

In the first row, 0.5 is the degree to which the word `animal' has had the feature $\{\mbox{cats, kittens}\}$ in the corpus, e.g.  due to the fact that it has occurred in sentences such as `a cat is an animal' and `kittens are small and cute animals'. Similarly, 0.4 is the degree to which `animal' has feature $\ket{\{\mbox{miaow, purr}\}}$ and 0.3  the degree to which it has feature $\ket{\{\mbox{sleep, snore}\}}$. The values of the $\ket{s_i}$ in $\ovl{\semantics{\text{run}}}$ depend on the concrete instantiation of the sentence dimensions, to keep things simple we do not instantiate them. 

One  computes  the vector interpretations of $\ovl{\semantics{\text{all}}}(\ovl{\semantics{\text{animals}}})$ and $\ovl{\semantics{\text{some}}}(\ovl{\semantics{\text{animals}}})$ by linearly expanding $\ovl{\semantics{\text{all}}}$ and $\ovl{\semantics{\text{some}}}$ over the vector of `animal':

{\small  
\begin{eqnarray*}
\ovl{\semantics{\text{all}}}(\ovl{\semantics{\text{animals}}}) &=& 0.5 \ovl{\semantics{\text{all}}}(\ket{\{\mbox{cats, kittens}\}}) + 0.4 \ovl{\semantics{\text{all}}}(\ket{\{\mbox{miaow, purr}\}})  + 0.3 \ovl{\semantics{\text{all}}}(\ket{\{\mbox{sleep, snore}\}})\\
\ovl{\semantics{\text{some}}}(\ovl{\semantics{\text{animals}}}) &=& 0.5 \ovl{\semantics{\text{some}}}(\ket{\{\mbox{cats, kittens}\}})+ 0.4 \ovl{\semantics{\text{some}}}(\ket{\{\mbox{miaow, purr}\}})  + 0.3 \ovl{\semantics{\text{some}}}(\ket{\{\mbox{sleep, snore}\}})
\end{eqnarray*}}
The interpretations of the quantified sentences `all animals run' and `some animals run'  are  be computed by substituting these numbers in the formula $\sum_{ijk}  \sum_{B \in \semantics{d}(\semantics{n})} c_i^n c_{jk}^{vp} c_B^d \langle B \mid A_i \cap A_j \rangle \ket{s_k}$.  It results in the following summands of their corresponding  linear expansions:

\begin{center}
{\small
\begin{tabular}{c|c|c|c|}
& $\ket{s_1}$ & $\ket{s_2}$ & $\ket{s_3}$\\
\hline
$\ovl{\semantics{\mbox{all animals run}}}$ &  $0.5 \times (0.4 \times 0.9 + 0.3 \times 0.2)$ & 
${ 0.4 \times (0.5\times 0.7 + 0.3 \times 0.3})$ &
$ 0.3 (0.5 \times 0.5 + 0.4 \times 0.3)$\\
$ \ovl{\semantics{\mbox{some animals run}}}$&
$0.5 \times (0.4 \times 0.3 + 0.5 \times 0.2)$ &
 $0.4 \times (0.5\times 0.9 + 0.3 \times 0.5) $ &
 $0.3 (0.5 \times 0.6 + 0.4 \times 0.5)$\\
 \hline
\end{tabular}
}
\end{center}
In the literature on  \emph{distributional inclusion hypothesis}  \cite{Geffet,Weeds} different types  of  orderings on feature vectors are used  to model and experiment with word-level entailment. Wherein,  a word  `$v$' entails a word `$w$', written as `$v \vdash w$', if features of  `$v$' are also features of  `$w$'. The simplest such ordering is the point wise ordering on vector dimensions. In our model, the point wise ordering  on the feature sets  provide us with the following entailments:
\[
\ovl{\semantics{\mbox{all animals}}}  \vdash
\ovl{\semantics{\mbox{some animals}}} \hspace{2cm} 
\ovl{\semantics{\mbox{all animals run}}}  \vdash
\ovl{\semantics{\mbox{some animals run}}}
\]
This opens the way  to reason about entailment on quantified phrases and sentences  compositionally and using  statistical data from corpora of text.  Implementing some of the above instantiations and experimenting with their applications  to  entailments  on datasets  constitutes work in progress.

\section{Conclusion and Future Work}
After a review of the setting of distributional semantics and a context free  and pregroup grammatical  formalisation of the fragment of language concerning quantified phrases and sentences (and the necessary preliminaries on compact closed categories and bialgebras), we developed an abstract compact closed categorical semantics for  quantifiers with the help of bialgebras.  We instantiated the abstract setting to the category of sets and relations and proved its  equivalence to the thruth-theoretic semantics of generalised quantifier theory of Barwise and Cooper. We extended the existing instantiation of the categorical compositional distributional semantics to finite dimensional vector spaces and linear maps  to develop a corpus-based instantiation for our model. Implementing this setting on real data and experimenting with it constitutes work in progress. Extending the theory to other sets of subsets, e.g. down  or up sets, rather than the full powerset is a future direction, as is developing a logic based on any of  these collections (powerset vs down or up sets). Extending the grammar to more expressive fragments of English and addressing advanced language phenomena such as co-reference resolution, following the work of  \cite{Preller14},  is another future direction.

\section{Acknowledgements}
We thank the anonymous reviewers for their  comments. Hedges thanks ESPRC for postdoctoral fellowship EP/N021282/1. Sadrzadeh thanks EPSRC for Career Acceleration Fellowship EP/J002607/1 and AFOSR for International Scientific Collaboration Grant FA9550-14-1-0079.

\bibliography{quant.bib}
\bibliographystyle{splncs03}

\end{document}